%% file: arxiv.tex
\documentclass[10pt,twocolumn,letterpaper,abstract=true]{scrartcl}

\usepackage[utf8]{inputenc}             
\usepackage[english]{babel}                   

\usepackage{amsmath}
\usepackage{amssymb}
\usepackage{amsthm}
\usepackage{amsbsy} 
\usepackage{amsfonts}

\usepackage{lmodern}
\usepackage{fourier}

\usepackage{mathtools}
\usepackage[dvipsnames]{xcolor}			
\usepackage{tabularx}
\usepackage{enumitem}					
\usepackage{tikz}
\usepackage{tikz-cd}
\usetikzlibrary{calc}

\usepackage[style = numeric-comp, sorting = none]{biblatex}

\definecolor{iccvblue}{rgb}{0.21,0.49,0.74}
\usepackage[breaklinks,colorlinks,allcolors=iccvblue]{hyperref}

\input{preamble}
\usepackage{geometry}
\input{Code/lst-Macaulay2.tex}
\newcommand{\tt}{\normalfont \ttfamily}

\bibliography{arxiv.bib}

\setkomafont{author}{\large}

\title{\LARGE A Framework for Reducing the Complexity of Geometric Vision Problems and its Application to Two-View Triangulation with Approximation Bounds}

\author{Felix Rydell
\and
Georg Bökman
\and 
Fredrik Kahl
\and
Kathlén Kohn
}

\begin{document}
\maketitle

\begin{abstract}
In this paper, we present a new framework for reducing the computational complexity of geometric vision problems through targeted reweighting of the cost functions used to minimize reprojection errors. Triangulation - the task of estimating a 3D point from noisy 2D projections across multiple images - is a fundamental problem in multiview geometry and Structure-from-Motion (SfM) pipelines. We apply our framework to the two-view case and demonstrate that optimal triangulation, which requires solving a univariate polynomial of degree six, can be simplified through cost function reweighting reducing the polynomial degree to two. This reweighting yields a closed-form solution while preserving strong geometric accuracy. We derive optimal weighting strategies, establish theoretical bounds on the approximation error, and provide experimental results on real data demonstrating the effectiveness of the proposed approach compared to standard methods. Although this work focuses on two-view triangulation, the framework generalizes to other geometric vision problems.

\end{abstract}

\section{Introduction}
\label{sec:intro}
Multiple view geometry has long been a well-established field within computer vision, with several decades of extensive research, as noted in works such as \cite{hartley2003multiple}. The inherent complexity of various subproblems, including relative pose estimation and triangulation, has been rigorously analyzed, often quantified by the number of critical points required to achieve optimal solutions. Typically, these problems are addressed either through slow but guaranteed methods or through faster, local iterative methods that lack assurance of reaching an optimal solution. In contrast, in this paper, a fundamentally different direction is pursued. We introduce a framework that makes targeted modifications to the cost function in order to reduce the underlying complexity. When applied to the triangulation problem, we demonstrate that it can be simplified at its core, reducing the number of critical points and thereby enhancing computational efficiency. An example is given in Figure~\ref{fig:cost-function}.

Triangulation stands as a cornerstone problem with extensive applications in 3D reconstruction, robotics, and augmented reality. Formally, the task involves recovering the 3D position of a point $X$ from its observed projections
$\tilde \xx_i$ in two or more camera images, where each projection is expressed by $\pi_i(X)$. 
Ideally, if the 3D point $X$ and its corresponding 2D projections $\tilde \xx_i$ are perfectly aligned, that is, $\tilde\xx_i = \pi_i(X)$, then this reconstruction becomes a straightforward calculation. However, in real-world scenarios, various sources of error – such as inaccuracies in the camera's internal parameters, small discrepancies in relative camera positioning, or limitations in point-matching precision – lead to imperfect data. These imperfections result in skewed rays that fail to intersect precisely in 3D space. Consequently, the triangulation problem in practical settings shifts to finding the 3D point that most closely aligns with the observed 2D projections.

\begin{figure}
    \centering
    \includegraphics[width=0.9\linewidth]{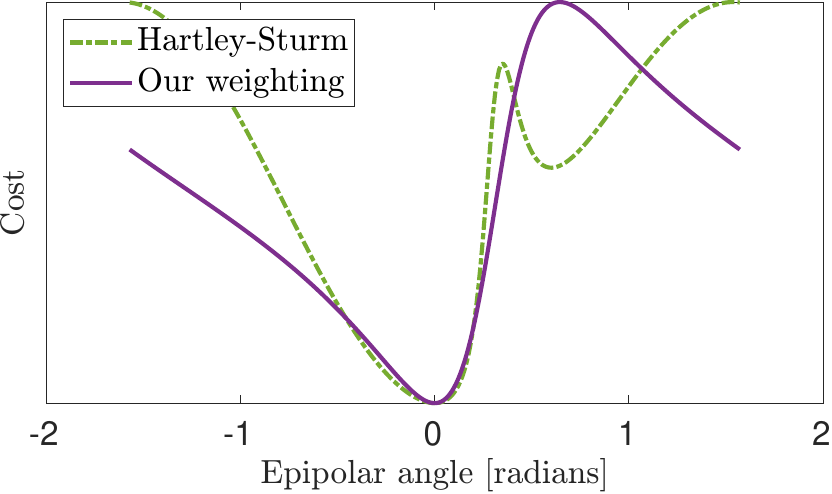}
    \caption{The optimal triangulation problem can have up to three local minima and requires finding the roots of a degree-6 univariate polynomial, following Hartley-Sturm~\cite{hartley1997triangulation}.
    We propose a weighting of the cost function that lowers the degree to 2, yielding a unique minimum.
    The plot shows a specific example where Hartley-Sturm's method has to choose between two local minima.
    Here, we plot the lowest cost value per epipolar plane, parameterised by an angle from the principal axis of one camera (see \cite[Section~4.2]{hartley1997triangulation} for details). The costs are rescaled to have the same minimum and maximum value for ease of viewing.
    }
    \label{fig:cost-function}
\end{figure}

Assuming independent Gaussian noise on the image measurements, the maximum likelihood estimate is obtained by minimizing the $L_2$-error 
between observed and ideal projections. This geometric error, addressed by Hartley and Sturm already in 1997 \cite{hartley1997triangulation}, involves computing the six critical points intrinsic to the problem \cite{galois}, meaning that any simpler (non-direct) solution inevitably involves trade-offs. Iterative approaches, such as the fast method proposed by Lindstrom \cite{lindstrom2010triangulation}, have also been suggested, although these can converge to local optima. An alternative is to change the optimization criterion, for example, using the $L_\infty$-norm \cite{kahl-hartley-pami-2008}, which enables an optimal solution through convex optimization. In contrast, we explore a novel approach in multiple view geometry by weighting the $L_2$-cost function to reduce the set of critical points, allowing direct computation of the solution in closed form.

Our framework can, in principle, be applied to any algebraic optimization problem and is thus widely applicable within multiple view geometry, e.g., to $n$-view triangulation \cite{josephson-kahl-jmiv-2010}, camera pose estimation \cite{enqvist-kahl-eccv-2008} and registration \cite{olsson-etal-icpr-2006}. Here, we showcase our strategy in detail for two-view triangulation.


In Sec.~\ref{sec:framework}, we introduce our overall approach formally.
In Sec.~\ref{sec:2viewTr}, we recap two-view triangulation and provide an equivalent reformulation via a diagonalizing change of coordinates.
In Sec.~\ref{sec:weightedOpt}, we study our proposed weighted version of two-view triangulation. We determine the best weights that reduce the number of critical points from 6 to either 4 or 2 (see Theorems \ref{thm:EDD2} and \ref{thm: unique nu}) and find theoretical bounds on the quality of our proposed approximation (see Proposition \ref{prop: bounds 2}). Experiment results and comparisons to baselines on real data are given in Sec.~\ref{sec:exp}.

\subsection{Related Work}\label{sec:related_work} 

\paragraph{Triangulation.} The most well-known approach to triangulation is to solve for the six critical points as described in \cite{hartley1997triangulation}, but the algorithm tends to be slow. Computing the midpoint between 3D rays does not work well for near parallel rays, hence it should in general be avoided. Using algebraic cost functions, such as the Direct Linear Transform, can be fast but may be inaccurate \cite{hartley2003multiple}. Iterative methods for triangulation were pioneered by Kanatani et.\ al.\ \cite{kanatani2005statistical, kanatani2008triangulation} and the method by Lindstrom \cite{lindstrom2010triangulation} leverages this approach. 
In theory, the method does not guarantee an optimal solution with respect to the cost function, whereas our approach does with respect to the weighted cost function. Other works on triangulation modify the optimize criterion, for instance, the $L_1$-cost and the $L_\infty$-cost functions are optimized in \cite{kahl-ijcv-2008,kahl-hartley-pami-2008,lee-iccv-2019}. In \cite{lee-bmvc-2019}, a variant of the midpoint method is proposed.
We experimentally compare to both Lindstrom~\cite{lindstrom2010triangulation} and Hartley-Sturm~\cite{hartley1997triangulation}, two natural baselines that solve the optimal triangulation problem — the former being a fast, iterative method, and the latter a slower, exact method.

\paragraph{The Geometric Error.} In applied algebraic geometry, fitting noisy data points to a mathematical model defined by polynomials has seen a lot of interest \cite{breiding2024metric}. The smallest distance between the data and the model is the \textit{geometric error}. The corresponding geometric error for homographies was first introduced by Sturm~\cite{sturm1997vision}. The \emph{Euclidean distance degree} is the number of smooth complex critical points for the optimization problem, given random data. It expresses the algebraic complexity of fitting data to a model; the higher the Euclidean distance degree, the more computationally expensive the optimization. For 3D reconstruction, several works have studied or computed these degrees, e.g.,  \cite{hartley1997triangulation,galois,maxim2020euclidean,rydell2023theoretical,duff2024metric}. The Euclidean distance degree is used to implement efficient solvers in homotopy continuation \cite{HomotopyContinuation} or to solve the associated polynomial systems via specialized symbolic solvers \cite{larsson2017efficient}.

\paragraph{The Sampson Error.} Sampson approximation was first proposed in~\cite{sampson1982fitting} and independently by Taubin~\cite{taubin1991estimation} to approximate the point-conic distance. Luong and Faugeras~\cite{luong1996fundamental} introduced it to approximate the reprojection error in epipolar geometry. The Sampson error has been considered in other vision settings \cite{chojnacki2000fitting,chum2005geometric,leonardos2015metric,terekhov2023tangent}. This extensive use of  Sampson approximation for geometric problems shows its versatility. Recently, the Sampson error was revisited and studied from a mathematical perspective \cite{rydell2024revisiting}.

\paragraph{Weighted Euclidean Problems.} The authors of \cite{kozhasov2023minimal} studied weighted Euclidean distance problems for rank-one approximations of tensors\new{, variations thereof, and quadric hypersurfaces}. Similarly to this article, they analyze the weights that lead to the smallest number of critical points. 


\section{Framework} \label{sec:framework}

We consider geometric vision problems involving 3D points, their image projections, and the cameras. A full-rank $3 \times 4$ matrix $C$ defines a camera that projects a 3D point $X \in \RR^3$ as
\begin{align}\begin{aligned}
    \pi:\RR^3&\dashrightarrow  \RR^2 \\
    X&\mapsto \begin{bmatrix}
       C(X;1)_1/  C(X;1)_3\\ C(X;1)_2/  C(X;1)_3
    \end{bmatrix}.
\end{aligned}
\end{align}
Here, $\dashrightarrow$ denotes a \textit{rational map}, meaning a map well-defined almost everywhere. We study both uncalibrated cameras, where $C$ is unconstrained, and calibrated cameras of the form $C = \left[ R ; t \right]$, where $R \in \mathrm{SO}(3)$ and $t \in \RR^3$.

Now, consider the residual between a projected 3D point $X$ and its measured image point $\tilde{x}$, given by $\epsilon = \tilde{x} - \pi(X)$. If the measured image points are corrupted by independent, normally distributed noise, the maximum likelihood estimate is obtained by minimizing $||\vec{\epsilon}||^2$ over the unknowns, where $\vec{\epsilon}$ is the vector of all image residuals. The unknowns, depending on the task, may consist of 3D points and/or cameras. In compact form, we seek to solve problems of the form
\begin{align}\begin{aligned} \label{eq:geometric_errX0}
  \min_{\vec{\epsilon},\vec{z}}\quad & ||\vec{\epsilon}||^2  \\
  \text{s.t. }\quad & p(\vec{\epsilon},\vec{z})=0,
\end{aligned}
\end{align}
 where $\vec{z}$ encodes the unknown parameters of interest. The constraint vector $p(\vec{\epsilon}, \vec{z}) = 0$ can be written as polynomial constraints by clearing denominators. Solving this optimization problem exactly quickly becomes intractable for large problems. One measure of complexity is the number of smooth complex critical points of the optimization given generic measured image points $\tilde{\vec{x}}$, known as the \textit{Euclidean distance degree} (ED-degree) of the problem.

\begin{exmp}[Triangulation]
Given cameras $C_i$ for $i=1,\ldots,n$ and corresponding image points $\tilde{x}_i$ in $n$ views, computing the 3D point $X$ is known as \textit{triangulation}. The ED-degree for $n=2$ is well known to be 6 (for generic cameras), with an algorithm for computing the six stationary points first presented in \cite{hartley1997triangulation}. For three-view triangulation ($n=3$), the ED-degree is 47 \cite{stewe-iccv-2005,maxim2020euclidean}.
\end{exmp}

The ED-degree of a geometric vision problem is intrinsic, meaning that reducing complexity requires altering the problem itself. We address this by reweighting the objective function and analyzing how different choices of weights affect the ED-degree. 
Concretely, we investigate 
\begin{align}\begin{aligned}\label{eq: weight opt_general} 
  \min_{\ee,\vec{z}}\quad & \sum_i \lambda_i \varepsilon_i^2  \\
  \text{s.t. }\quad & p(\ee,\vec{z})=0,
\end{aligned}
\end{align}
where $\lambda_i$ are positive weights applied to the residual terms.
The optimal choice of weights depends heavily on the constraints $p$ and can be challenging to determine.
One strategy to mitigate this is to first make the constraints simpler by applying a coordinate transformation of the form
\begin{align}\label{eq:coordinate_change}
\ee=\vec{R}\vec{\epsilon}
\end{align}
with an orthogonal matrix $\vec{R}$.
This leaves the problem \eqref{eq:geometric_errX0} unchanged, as $||\ee|| = ||\vec{\epsilon}||$, ensuring that the ED-degree remains the same. We carry out this strategy in detail for two-view triangulation.

\section{Two-View Triangulation} \label{sec:2viewTr}

A common way of expressing two-view triangulation is via the fundamental matrix $\vec{F}$ of the camera pair $C_1,C_2$. 
More precisely, given a fundamental matrix $\vec{F}$, i.e., a $3\times 3$ rank-2 matrix, it is the following squared-error minimization:
\begin{align}\begin{aligned} \label{eq:geometric_err}
\mathcal{E}_{\vec{F}}^2(\tilde\xx):=\min_{\vec{\epsilon}}\quad & \norm{\vec{\epsilon}}^2 \\
  \text{s.t. }\quad & (\tilde\xx_1+\vec{\epsilon}_1; 1)^\tp \vec{F} (\tilde\xx_2+\vec{\epsilon}_2; 1)=0. 
\end{aligned}
\end{align}

Our goal is to find an approximate solution to \eqref{eq:geometric_err} that is simpler and faster to compute via~\eqref{eq: weight opt_general} and~\eqref{eq:coordinate_change}. In this direction, we first simplify the epipolar constraint. Note that \begin{align}\label{eq: epip}
    (\xx_1;1)^\top\vec{F}(\xx_2;1)
\end{align}
equals
\begin{align} \label{eq:Qconstraint}
    (\xx;1)^\top \underbrace{\frac{1}{2}\begin{bmatrix}
        0 & F_{2\times 2} & F_h\\
        F_{2\times 2}^\top & 0 & F_v^\top \\
        F_h^\top & F_v & 2F_{3,3}
    \end{bmatrix}}_{\vec{Q}(\vec{F}):=}(\xx;1),
\end{align}
where
\begin{align}
    \vec{F}=\begin{bmatrix}
    F_{2\times 2} & F_h\\ F_v & F_{3,3}
    \end{bmatrix}.
\end{align}

\begin{lemma} For a fundamental matrix $\vec{F}$, the matrix $\vec{Q}(\vec{F})$ is rank-deficient. Moreover, if $F_{2\times 2}$ is invertible, then $\vec{Q}(\vec{F})$ has rank~4 and its kernel is
\begin{align}
    \begin{bmatrix}
         k(\vec{F})\\
         1
    \end{bmatrix}, \quad \textnormal{where} \quad 
    k(\vec{F}):= \begin{bmatrix}
          -F_{2\times 2}^{-\top}F_v^\top \\
          -F_{2\times 2}^{-1}F_h
    \end{bmatrix}.
\end{align}
\end{lemma}


\begin{proof} The determinant of $\vec{Q}(\vec{F})$ is $\det(F_{2\times 2})\det(\vec{F})/16$. Therefore, it is always rank-deficient. Moreover, since $F_{2\times 2}$ is invertible, the top left $4\times 4$ matrix of $\vec{Q}(\vec{F})$ has rank $4$, implying that $\vec{Q}(\vec{F})$ has rank $4$. 
\end{proof}

Now we can rewrite \eqref{eq:Qconstraint} further.
Denote by $\vec{P}(\vec{F})$ the  upper left $4\times 4$ matrix of $\vec{Q}(\vec{F})$. By construction,
\begin{align}\label{eq:translationByKernel}
\begin{aligned}
    (\xx;1)^\top \vec{Q}(\vec{F})(\xx;1) &=\\
    (\xx-k(\vec{F}))^\top &\vec{P}(\vec{F})(\xx-k(\vec{F})).
\end{aligned}
\end{align}
Next, we express this constraint in terms of the eigenvalues of $\vec{P}(\vec{F})$. 
Since that matrix is symmetric, its eigenvalues are real. 
In fact, they are the signed singular values of $F_{2 \times 2}$.
Hence, the eigenvalues of $\vec{P}(\vec{F})$ are $a_1,-a_1,a_2,-a_2$ for some $a_1,a_2 \geq 0$.
Up to translation and orthogonal action, we now see that \eqref{eq:translationByKernel} is  equal to 
\begin{align}
    \sum_{i}q_i\yy_i^2, \quad \text{ where } \vec{q}=(a_1,-a_1,a_2,-a_2).
\end{align}
Here, the new variables $\yy$ are obtained from $\xx$ via $\yy= \vec{R}(\vec{F})^{\top}(\xx-k(\vec{F}))$ for some orthogonal $4 \times 4$ matrix $\vec{R}(\vec{F})$.
Our updated optimization problem is then 
\begin{align}\begin{aligned} \label{eq: err simplified}
\mathcal{E}_{\vec{q}}^2(\tilde\yy):=\min_{\vec{\varepsilon}}\quad & \norm{\vec{\varepsilon}}^2 \\
  \text{s.t. }\quad & \sum_i q_i(\tilde y_i+\varepsilon_i)^2=0. 
\end{aligned}
\end{align}

\begin{proposition} \label{prop:critPtsCorrespondence}Let $\vec{F}$ be a fundamental matrix such that $F_{2\times 2}$ is invertible, and let $ \mathrm{diag}(\vec{q}) = \vec{R}^\top \vec{P}(\vec{F}) \vec{R} $ be a diagonalization of $\vec{P}(\vec{F})$. Then the critical points of \eqref{eq:geometric_err} with data $\tilde \xx$ are in bijection with the critical points of \eqref{eq: err simplified} with data $\tilde \yy= \vec{R}^{\top}(\tilde \xx-k(\vec{F})) $ via
\begin{align}\label{eq: bij map}
    \vec{\epsilon} \mapsto \ee=\vec{R}^{\top}\vec{\epsilon}.
\end{align}
In particular, $\mathcal{E}_{\vec{F}}(\tilde \xx)= \mathcal{E}_{\vec{q}}(\tilde \yy)$.
\end{proposition}

\begin{proof} It is a well-known fact in metric algebraic geometry that translation and orthogonal transformation preserve the ED-degree and that the critical points are in bijection via \eqref{eq: bij map}. To see this, one can directly study the critical equations. The details are worked out in \cite{rydell2023theoretical}. 
\end{proof}

\begin{remark}
    All parameters $a_1,a_2>0$ are possible, also when we restrict ourselves to calibrated cameras. This is because all non-zero matrices $F_{2\times 2}$ can be obtained  from $C_1 = [I \; 0]$ and $C_2 = [R \; t]$ with $R \in \mathrm{SO}(3)$. Indeed, one can choose $t_1 = t_2 = 0$ and add a column to $\left[ \begin{smallmatrix}
        0 & -1 \\ 1 & 0 
    \end{smallmatrix}\right] F_{2 \times 2}^\top$ such that the resulting $2 \times 3$ matrix $S$ has orthogonal rows of the same norm. Then choose $t_3$ such that $S/t_3$ has rows of norm $1$ and extend that matrix to a $3 \times 3$ rotation matrix $R$. That way, the top left block of $R^\top[t]_\times$ is $F_{2 \times 2}.$
    
\end{remark}


\section{The weighted optimization problem}
\label{sec:weightedOpt}
We will now show that one can change the standard squared-error minimization to a \emph{weighted} squared-error loss such that the number of critical points drops and the optimization problem becomes simpler. More concretely, we replace \eqref{eq: err simplified} by
\begin{align}\begin{aligned}\label{eq: weight opt} 
  \mathcal{E}_{\vec{q},\vec{\lambda}}^2(\tilde \yy):=\min_{\ee}\quad & \sum_i \lambda_i \varepsilon_i^2  \\
  \text{s.t. }\quad & \sum_i q_i (\tilde y_i+\varepsilon_i)^2=0,
\end{aligned}
\end{align}
where  $\vec{q}=(a_1,-a_1,a_2,-a_2)$ and $\vec{\lambda}=(\lambda_1,\lambda_2,\lambda_3,\lambda_3)\in \RR_{>0}^4$. The restriction that the $\lambda_i$ are positive ensures that the optimization problem corresponds to minimizing a distance \new{(that may differ from the standard Euclidean distance).} The number of  complex critical points of \eqref{eq: weight opt} for general $a_1,a_2$ and $\tilde \yy$ is called the \textit{$\vec{\lambda}$-Weighted ED-degree} ($\vec{\lambda}$-degree). 

\begin{theorem}
    \label{thm:EDD2}
    Let $a_i,\lambda_i>0$. The $\vec{\lambda}$-degree of \eqref{eq: weight opt} is 
    \begin{itemize}[leftmargin=0.75cm]
    \item[I.] 2 if $\lambda = (\mu a_1, \nu a_1, \mu a_2, \nu a_2)$ for some $\mu,\nu \in \mathbb{R}_{>0}$,
        \item[II.] 4 \new{otherwise} if $(\lambda_1,\lambda_3) = \mu (a_1,a_2)$ for some $\mu \in \mathbb{R}_{>0}$
    or $(\lambda_2,\lambda_4) = \nu (a_1,a_2)$ for some $\nu \in \mathbb{R}_{>0}$,
        \item[III.] 6 otherwise.
    \end{itemize}
\end{theorem}

As a consequence, if $a_1=a_2$, then the $\vec{\lambda}$-degree is $2$ for $\vec{\lambda}=(1,1,1,1)$, meaning that the ED-degree is also 2.
This is in particular the case for calibrated cameras $C_1=\begin{bmatrix}
    I & 0
\end{bmatrix}$ and $C_2=\begin{bmatrix}
    R & t
\end{bmatrix}$ where the last row of $R$ is $(0,0,1)$, as we will see in Prop.~\ref{prop:F22orthogonal}. Note that the last row of $R$ encodes the optical axis of the camera, and hence for a stereo rig with parallel image planes, we obtain the optimal (unweighted) $L_2$-solution without having to solve a degree-6 polynomial!
\begin{proof}
    Given noisy measurements $\tilde \yy$, the critical points of the optimization problem \eqref{eq: weight opt} are those $\ee \neq - \tilde \yy$ that satisfy the problem's constraint and such that
    the Jacobian  matrix 
    \begin{align*} \small
        \begin{bmatrix}
            \lambda_1 \varepsilon_1 & \lambda_2 \varepsilon_2& \lambda_3\varepsilon_3 & \lambda_4\varepsilon_4 \\
            a_1(\tilde y_1 + \varepsilon_1) & 
            -a_1(\tilde y_2 + \varepsilon_2)
            & a_2(\tilde y_3 + \varepsilon_3)
            & -a_2(\tilde y_4 + \varepsilon_4)
        \end{bmatrix}
    \end{align*}
    has rank one.
    Writing $q = (a_1,-a_1,a_2,-a_2)$, the rank constraint means that
    $\lambda_i \varepsilon_i = s q_i (\tilde y_i + \varepsilon_i) $ for some scalar $s$ and all $i = 1,2,3,4$. \new{This allows us to express 
    \begin{align}
      \varepsilon_i = s q_i \tilde y_i / (\lambda_i  - s q_i).  
    \end{align}
    Plugging the latter into the constraint of \eqref{eq: weight opt} yields a rational function 
    \begin{align}
        \mathcal{R} = \frac{r_6(s)}{\prod_i (\lambda_i - sq_i)^2},
    \end{align}
    whose numerator $r_6(s)$ depends on $a_i,\tilde y_i$ and is of degree 6 in $s$. It has too many terms to be displayed here, but the \texttt{Macaulay2} \cite{Macaulay2} code in the SM computes it explicitly.} The roots of the numerator correspond to the critical points of the optimization  \eqref{eq: weight opt}, showing that the $\vec{\lambda}$-degree is at most 6 for generic weights $\vec{\lambda}$. This is in fact an equality since 
     the standard Euclidean distance problem \eqref{eq: epip} has ED-degree $6$.

    A priori, there are two ways how the $\vec{\lambda}$-degree (i.e., the degree of the numerator above) can drop. For special choices of weights $\vec{\lambda}$, either the leading coefficient of \new{$r_6(s)$} can vanish, or \new{$r_6(s)$} can share a common factor with the denominator of $\mathcal{R}$.
    The first case cannot happen, as \new{the leading coefficient is} 
    \begin{align}
        a_1^3 a_2^3 (a_2 \lambda_1^2 \tilde y_1^2 - a_2 \lambda_2^3 \tilde y_2^2 + a_1 \lambda_3^2 \tilde y_3^2 - a_1 \lambda_4^2 \tilde y_4^2),
    \end{align}
    which is non-zero for generic $\tilde y$. 

    Next we analyze under which conditions one of the factors $(\lambda_i - s q_i)$ of the denominator divides the numerator. This is equivalent to that $s = \lambda_i/q_i$ is a root of the numerator.
    Plugging $s = \lambda_1/q_1$ into the numerator yields
    \begin{align}
        a_1^{-3} \, \lambda_1^2 \,(\lambda_1+\lambda_2)^2
        \, \tilde y_1^2 
        \,(a_1 \lambda_4 + a_2 \lambda_1)^2 
        \, (a_1 \lambda_3 - a_2 \lambda_1)^2.
    \end{align}
    Due to $a_i,\lambda_j > 0$, only the last factor in this expression can be zero (for generic $\tilde y$).
    That term being zero means that $(\lambda_1,\lambda_3) = \mu (a_1,a_2)$ for some $\mu \in \mathbb{R}_{>0}$.
    Hence, we have shown that the latter condition is equivalent to $(\lambda_1 - s q_1)$ dividing the numerator.

    Analogously, we obtain that $(\lambda_3 - s q_3)$ divides the numerator if and only if $(\lambda_1,\lambda_3) = \mu (a_1,a_2)$ for some $\mu \in \mathbb{R}_{>0}$;
    and that $(\lambda_2 - s q_2)$ or $(\lambda_4 - s q_4)$ divide the numerator if and only if $(\lambda_2,\lambda_4) = \nu (a_1,a_2)$ for some $\nu \in \mathbb{R}_{>0}$.

    Without loss of generality, we now assume that $(\lambda_1,\lambda_3) = \mu (a_1,a_2)$ for some $\mu \in \mathbb{R}_{>0}$.
    \new{Then 
    \begin{align}
        \mathcal R = \frac{r_4(s)}{(\lambda_2-sq_2)^2(\lambda_4-sq_4)^2(\mu-s)^2},
    \end{align}
    whose numerator $r_4(s)$ depends on $a_i,\tilde y_i$ and is of degree 4 in $s$. The code in the SM produces also this numerator.} Therefore, for generic $\lambda_2$ and $\lambda_4$, the $\vec{\lambda}$-degree is now 4.
    As before, the leading coefficient of the numerator of $\mathcal{R}$ does not vanish for generic $\tilde y$.
    Thus, the degree can only drop further if one of the factors in the denominator of $\mathcal{R}$ divides the numerator. 
    The factor $(\mu-s)$ cannot divide the numerator for generic $\tilde y$, since plugging $s=\mu$ into the numerator yields
    \begin{align}
        \mu^2 \, (a_1 \mu + \lambda_2)^2 \, (a_2 \mu+\lambda_4)^2 \, (a_1 \tilde y_1^2 + a_2 \tilde y_3^2),
    \end{align}
    which is positive for generic $\tilde y$.
    Hence, the degree can only drop further if $(\lambda_2 - s q_2)$ or $(\lambda_4 - s q_4)$ divide the numerator.
    We have already shown above that this is equivalent to $(\lambda_2,\lambda_4) = \nu (a_1,a_2)$ for some $\nu \in \mathbb{R}_{>0}$.
    In this case,  \new{
    \begin{align}
       \mathcal{R} = \frac{r_2(s)}{(\mu-s)^2(\nu+s)^2}, 
    \end{align}
    whose numerator $r_2(s)$ depends on $a_i,\tilde y_i$ and is of degree $2$ in $s$. (Its coefficients  are explicit stated in \eqref{eq:coefficientsQuadricR}; see  proof of Lemma \ref{le: critical points}).} This shows that the $\vec{\lambda}$-degree is now 2. 
\end{proof}

The weighted squared-error minimizations of $\vec{\lambda}$-degree~2 and~4 from Theorem~\ref{thm:EDD2} can be explicitly solved. Therefore,  they are significantly faster to solve than the original problem of ED-degree 6. This raises two natural questions:
\begin{itemize}[leftmargin=0.75cm]
    \item Which choice of scalars $\mu,\nu$ is `best' in the sense that the solution to \eqref{eq: weight opt} best approximates the original problem?
    \item How good is this `best' approximation?
\end{itemize}
More concretely, we say that the best $\mu,\nu$ are those such that the minimizer $\ee(\mu,\nu)$ of \eqref{eq: weight opt}  minimizes the standard squared error $\sum_i \varepsilon_i^2$. Since the minimizer $\varepsilon(\mu,\nu)$ is not affected by multipying $\lambda$ with a global scalar, we assume from now on without loss of generality that $\mu=1$. (We could of course also assume that $a_1=1$, but choose not do so.)

We  solve the two questions above for weights of the form $\lambda = (a_1, \nu a_1,  a_2, \nu a_2)$. We find the optimal $\nu$ by theoretical means in Theorem \ref{thm: unique nu} and use this result to  bound 
 the error $\mathcal E_{\vec{F}}$ in Proposition \ref{prop: bounds 2}. 
We then perform experiments on the quality of our proposed approximation in Sec. \ref{sec:exp}.

We begin by describing the critical points of \eqref{eq: weight opt} in terms of 
\begin{align}\begin{aligned}
        A &:=  a_1 \tilde y_1^2 -\nu^2 a_1 \tilde y_2^2 + a_2 \tilde y_3^2 -\nu^2 a_2 \tilde y_4^2,\\
        B &:= 2\nu(a_1 \tilde y_1^2 + \nu a_1 \tilde y_2^2 + a_2 \tilde y_3^2 + \nu a_2 \tilde y_4^2), \\
        C &:= \nu^2 (a_1 \tilde y_1^2 - a_1 \tilde y_2^2 + a_2 \tilde y_3^2 - a_2 \tilde y_4^2),
        \label{eq:coefficientsQuadricR}\end{aligned}
    \end{align}
    and
    \begin{align}\begin{aligned}
        \Delta:=& B^2-4AC\\ = &4 \nu^2 (\nu+1)^2 (a_1 \tilde y_1^2+ a_2 \tilde y_3^2) (a_1 \tilde y_2^2+ a_2 \tilde y_4^2).
    \end{aligned}
    \end{align}

\begin{lemma}\label{le: critical points} For $\lambda=(a_1,\nu a_1,a_2,\nu a_2)$, the two critical points of \eqref{eq: weight opt} are
\begin{align}
\label{eq:criticalPoints}
\new{\varepsilon^{\pm}(\nu) = (s^{\pm} q_i \tilde y_i / (\lambda_i  - s^{\pm} q_i))_{i=1}^4},
\end{align}
where $s^{\pm} = (-B \pm \sqrt{\Delta} )/2A$.
\end{lemma}

\begin{proof}
     We consider the rational function $\mathcal{R}$ from the proof of Theorem~\ref{thm:EDD2} whose roots correspond to the critical points of \eqref{eq: weight opt}.
    The numerator of $\mathcal{R}$ is $A s^2 + Bs + C$. The discriminant  $\Delta$ of the numerator is positive for generic $\tilde y$ and $\nu >0$. Thus, there are 2 real roots $s^{\pm} = (-B \pm \sqrt{\Delta} )/2A$.
    These yield the critical points 
    $\varepsilon_i^{\pm}(\nu) = s^{\pm} q_i \tilde y_i / (\lambda_i  - s^{\pm} q_i) $. 
\end{proof}

\new{Now we can describe the best $\nu$ in terms of
\begin{align}
    S &= (\tilde y_1^2 +\tilde y_3^2)(a_1\tilde y_2^2 +a_2\tilde y_4^2),\\
    T &= (\tilde y_2^2 +\tilde y_4^2)(a_1\tilde y_1^2 +a_2\tilde y_3^2),
\end{align}
which we do in the next result. For the proof, we define $\delta:=(a_1 \tilde y_1^2+ a_2 \tilde y_3^2) (a_1 \tilde y_2^2+ a_2 \tilde y_4^2)$, which satisfies $\Delta = 4\nu^2(\nu+1)^2\delta$.}

\begin{theorem}\label{thm: unique nu}
    Let  $\lambda = (a_1, \nu a_1, a_2, \nu a_2)$. For every  $\nu >0$,  $\ee^+(\nu)$ minimizes both the weighted and the nonweighted squared error, i.e., 
    $\sum_i\lambda_i (\ee_i^+(\nu))^2< \sum_i \lambda_i(\ee_i^-(\nu))^2$
    and
    $\sum_i(\ee_i^+(\nu))^2< \sum_i(\ee_i^-(\nu))^2$. Further, there is a unique $\nu \in \mathbb{R}_{>0}$ for which $\sum_i(\ee_i^+(\nu))^2$ is minimized. It is
    \begin{align}\label{eq:optimal_nu}
        \new{\nu = \frac{T}{S}.}
    \end{align}
\end{theorem}

\begin{proof} 
Evaluating the weighted squared-error  at the critical points in \eqref{eq:criticalPoints} yields
    \begin{align}\label{eq:valueWeightedError}
        \sum_i \lambda_i (\varepsilon_i^{\pm}(\nu))^2 &= \frac{\nu}{\nu+1} \Big(\new{\big(\underbrace{\sum |q_i| \tilde y_i^2\big) - \pm 2  \sqrt{\delta}}_{\alpha^{\pm}:=}}\Big).
    \end{align}
\new{The code that produces this identity is provided in the SM.} The minimizer is therefore $ \varepsilon^{+}(\nu)$. \new{Note that  
\begin{align}
    \alpha^{\pm} = \Big(\sqrt{a_1\tilde y_1^2 + a_2 \tilde y_3^2} -\pm  \sqrt{a_1\tilde y_2^2 + a_2 \tilde y_4^2}\Big)^2,
\end{align}
so that for generic data $\tilde y_i$, $\alpha^{\pm}>0$.}
Evaluating the nonweighted squared-error at the critical points gives\new{
    \begin{align}\label{eq: sum}
        \sum_i (\varepsilon_i^\pm(\nu))^2 &= \frac{\alpha^{\pm}(S\nu^2 + T)}{\delta(\nu+1)^2}.
    \end{align}
    The code that produces this identity is provided in the SM. Since each $\alpha^{\pm},S,T,\delta$ are $>0$ for generic data, it follows that $\varepsilon^+(\nu)$ is the minimizer of the standard squared error.} 
    
    Next we compute the best $\nu$. The derivative of \eqref{eq: sum} \new{for $\pm = +$} with respect to $\nu$ is \new{
    \begin{align}
     \frac{2\alpha^{+}(S\nu -T)}{\delta (\nu +1 )^3}.   
    \end{align} Setting this expression to $0$, the unique solution $T/S$ is nonnegative.} One can check that the second derivative at \new{$T/S$} is positive, implying that this choice of $\nu$ yields the global minimum. 
\end{proof}


\begin{figure*}
    \centering
    \includegraphics[width=0.49\linewidth]{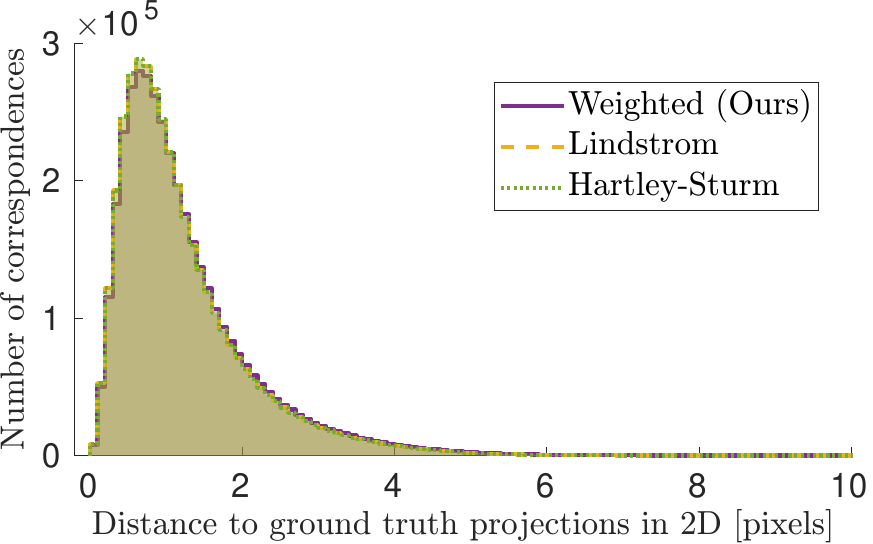}
    \includegraphics[width=0.49\linewidth]{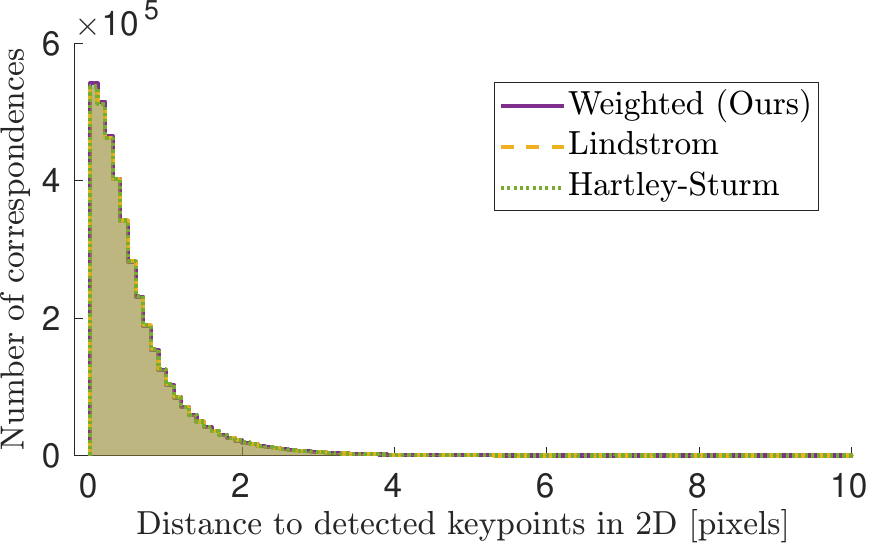}
    \caption{2D errors over correspondences from randomly sampled image pairs from the Pantheon dataset, when solving the triangulation problem using different methods. The methods perform very similarly, but our weighted method is slightly worse than the others. Left: Distance to ground truth projections; Right: Distance to measured 2D points.}
    \label{fig:errors}
\end{figure*}

\begin{figure}
    \centering
    \includegraphics[width=0.99\linewidth]{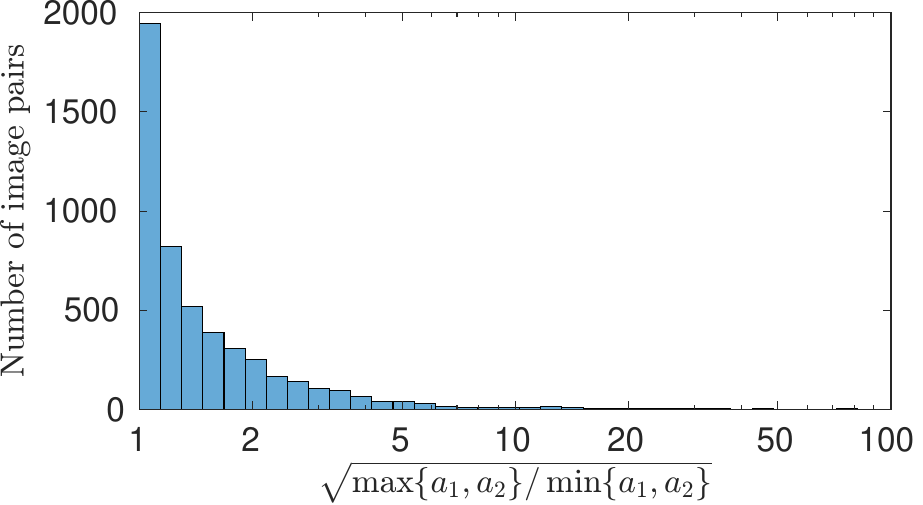}
    \caption{The eigenvalue ratios in our randomly sampled set of 5000 image pairs from the Pantheon dataset.}
    \label{fig:eigRatios}
\end{figure}

Now we provide bounds on the error $\mathcal E_{\vec{F}}$. Recall from Proposition \ref{prop:critPtsCorrespondence} that $\mathcal{E}_{\vec{F}}(\tilde \xx)= \mathcal{E}_{\vec{q}}(\tilde \yy)$\new{, and note that from the proof of Theorem~\ref{thm: unique nu},
\begin{align}
    \sqrt{\alpha^+}= \big|\sqrt{a_1\tilde y_1^2 + a_2\tilde y_3^2}-\sqrt{a_1\tilde y_2^2 + a_2\tilde y_4^2} \big|.
\end{align}
\new{Observe that $\alpha^+=0$ if and only if $\tilde \yy$ lies on the model, and is strictly greater than $0$ otherwise.}

\begin{proposition}\label{prop: bounds 2} \new{The inequality}
\begin{align}\label{eq: upper bound}
    \mathcal E_{\vec{q}}(\tilde \yy)\le\new{\sqrt{\frac{\alpha^+}{\delta}\frac{ST}{S+T}}}\end{align}
    \new{holds, along with the bounds}
    \begin{align}\begin{aligned}\label{eq: bounds}
    \new{\frac{\sqrt{\alpha^+}}{\sqrt{2\max\{a_1,a_2\}}}\le \mathcal E_{\vec{q}}(\tilde \yy)  \le \frac{\sqrt{\alpha^+}}{\sqrt{2\min\{a_1,a_2\}}}.}
    \end{aligned}
\end{align}
\end{proposition}
\new{The right-hand side of \eqref{eq: upper bound} can be simplified somewhat by noting that 
\begin{align}
    ST=(\tilde y_1^2+\tilde y_3^2)(\tilde y_2^2+\tilde y_4^2)\delta.
\end{align}

The ratio between the upper and lower bounds of \eqref{eq: bounds}} is
\begin{align}
    \sqrt{\frac{\max\{a_1,a_2\}}{\min\{a_1,a_2\}}}.
\end{align}
Therefore, the closer the ratio $a_2/a_1$ is to $1$, the better the these bounds are. In comparison, the upper bound \eqref{eq: upper bound} is a better approximation, which follows from the proof.} However, this comes at the cost of a more complicated expression.

\begin{proof} Let $\varepsilon^*$ be the minimizer for the nonweighted problem. Then the minimizer from Theorem \ref{thm: unique nu} gives an upper bound for $\mathcal E_{\vec{q}}^2(\tilde \yy)=\sum (\varepsilon^*_i)^2$. As in the proof of Theorem \ref{thm: unique nu}, plugging \new{$\nu =T/S$ into \eqref{eq: sum} for $\pm = +$}, we get that the minimum value of \eqref{eq: sum} is
    \begin{align}
        \new{\frac{\alpha^+}{\delta}\frac{ST}{S+T}},
    \end{align}
which proves the upper bound. 

For the other part, we prove the lower bound and note that the upper bound is proven the same way. We observe that  
\begin{align}\label{eq:eigRatio}
   \max\{\lambda_i\} \sum (\varepsilon_i^*)^2 \ge  \sum \lambda_i(\varepsilon_i^*)^2.
\end{align}
Thus, Theorem \ref{thm: unique nu} implies $\mathcal E_{\vec{q}}^2(\tilde \yy) \ge \frac{1}{\max\{\lambda_i\}} \sum \lambda_i (\ee_i^+(\nu))^2$ for $\lambda = (a_1, \nu a_1, a_2, \nu a_2)$ and arbitrary $\nu > 0$. By \eqref{eq:valueWeightedError}, the latter expression equals
\begin{align}\label{eq:lowerBoundNu}
      \frac{\nu}{\max\{\lambda_i\}(\nu+1)} \alpha^+.
\end{align}
Note that $\max\{\lambda_i\}=\max\{1,\nu\}\max\{a_1,a_2\}$. Therefore, \eqref{eq:lowerBoundNu} reaches its maximum value when $\nu=1$. 
\end{proof} 

\begin{figure*}
    \centering
    \includegraphics[width=0.49\linewidth]{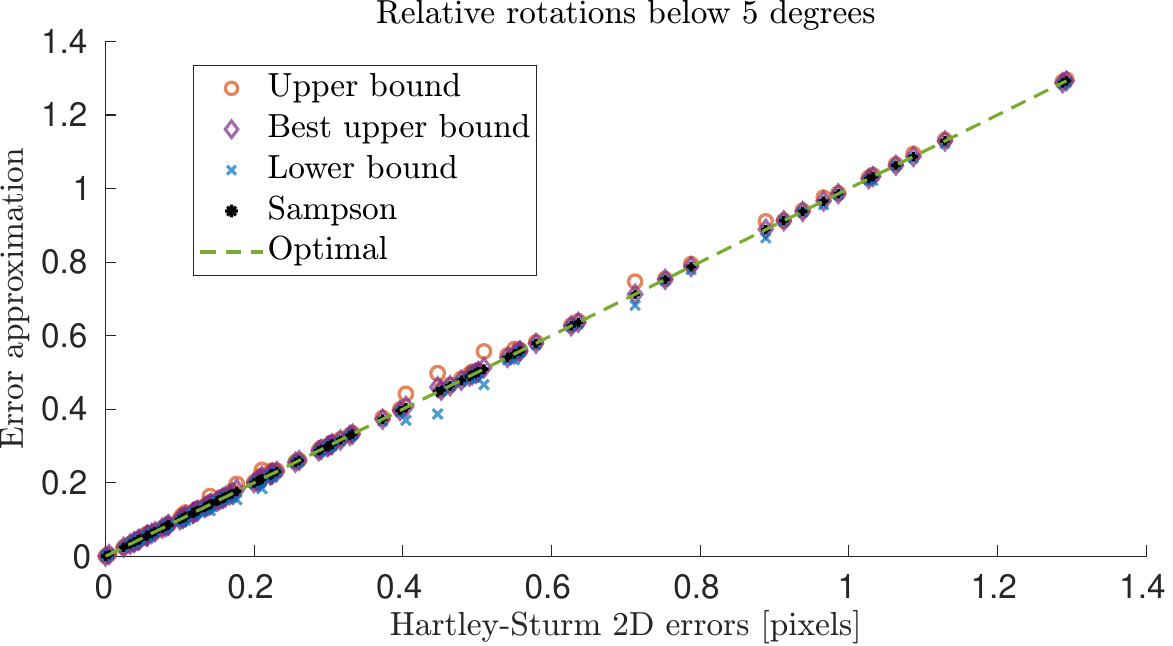}
    \includegraphics[width=0.49\linewidth]{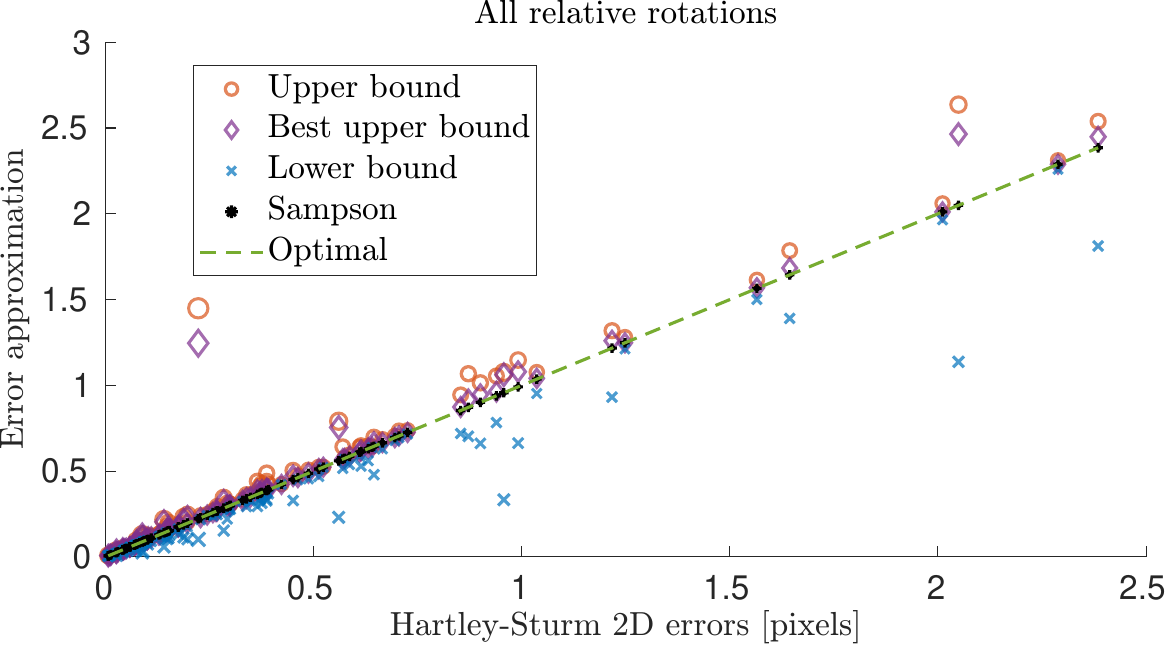}
    \caption{Evaluation of the error bounds~\eqref{eq: upper bound} and \eqref{eq: bounds} under different amounts of relative rotation between cameras.
    For cameras with close to the same viewing direction, the bounds are sharp as predicted by the theory. 
    The markers are scaled (logarithmically) by the eigenvalue ratio~\eqref{eq:eigRatio}. We subsample 100 random correspondences to plot, for ease of viewing.
    }
    \label{fig:bounds}
\end{figure*}

Finally, we investigate when the case $a_1=a_2$ happens.

\begin{proposition}
    \label{prop:F22orthogonal}
    We have $a_1=a_2$ if and only if $F_{2 \times 2} $ is a scalar times an orthogonal matrix. \\ \indent
    This condition is satisfied for calibrated cameras $C_1=\begin{bmatrix}
    I & 0
\end{bmatrix}$ and $C_2= R \begin{bmatrix}
    I & -c
\end{bmatrix}$ if and only if the last row of $R$ is $(0,0,\pm 1)$ or
$c$ is a scalar times $(r_{31},r_{32}, r_{33} \pm 1)^\top$.
\end{proposition}

Recall that the last row of $R$ encodes the optical axis of the camera $C_2$. So that row being $(0,0,\pm 1)$ is equivalent to the optical axes of both cameras being parallel, while the latter condition in Proposition \ref{prop:F22orthogonal} means that the center of $C_2$ is proportional to the sum / difference of the optical axes.

\begin{proof}
    The first statement is clear since $a_1,a_2$ are the singular values of $F_{2 \times 2}$.
    For calibrated cameras $C_1=\begin{bmatrix}
    I & 0
\end{bmatrix}$ and $C_2=R\begin{bmatrix}
    I & -c
\end{bmatrix}$, we have $\vec{F}= R^\top[t]_\times $, where $t = -Rc$ and
\begin{align}
    [t]_\times:= \begin{bmatrix}
       0 & -t_3 & t_2 \\
       t_3 & 0 & -t_1\\
       -t_2 & t_1 & 0
    \end{bmatrix}.
\end{align}
We note that $\vec F =R^\top[t]_\times = R^\top[-Rc]_\times=-[c]_\times R^\top$ due to rotation equivariance of the cross-product.
If  $R= \left[\begin{smallmatrix}
    R' & 0 \\ 0 & \pm 1
\end{smallmatrix} \right]$, 
then 
$F_{2 \times 2}=t_3(R')^\top \left[\begin{smallmatrix}0 & -1 \\ 1 & 0
\end{smallmatrix}\right] \in t_3 \mathrm{O}(2)$. \\ \indent
Next, we analyze the case $r_{33} \neq \pm 1$ and $(c_1,c_2) = (r_{31},r_{32}) \frac{c_3}{r_{33} \pm 1}$.
For $\pm=+$, a direct computation reveals that $\vec{F}$ becomes
\begin{align}
    \frac{c_3}{r_{33}+1}  \begin{bmatrix}
        r_{12}-r_{21} & r_{11}+r_{22} & r_{32} \\
        -r_{11}-r_{22} & r_{12}-r_{21} & -r_{31} \\
        -r_{23} & r_{13} & 0
    \end{bmatrix}.
\end{align}
This can also be verified via \texttt{Macaulay2} code provided in the SM.
In particular, we see that the top left $2 \times 2$ block is a scaled $\mathrm{SO}(2)$ matrix.
Analogously, for $\pm=-$, we obtain 
\begin{align}
\vec{F}=
    \frac{c_3}{r_{33}-1}  \begin{bmatrix}
        -r_{12}-r_{21} & r_{11}-r_{22} & -r_{32} \\
        r_{11}-r_{22} & r_{12}+r_{21} & r_{31} \\
        -r_{23} & r_{13} & 0
    \end{bmatrix}
\end{align}
and so $F_{2 \times 2}$ is a scalar times an orthogonal matrix (of determinant $-1$).

For the converse direction, we consider $\vec{F}= R^\top[t]_\times =-[c]_\times R^\top$ for some rotation matrix $R$ and $t = -Rc$ such that  $F_{2\times 2}$ is a scaling of an orthogonal matrix.
We provide \texttt{Macaulay2} code in the SM for solving the resulting equations.
Here we show a straightforward calculation of $c_1$ in the case of $F_{2\times 2}$ having positive determinant. (Other cases can be proven similarly.)
We have that
\begin{equation}
F_{2 \times 2} =  -\begin{bmatrix}
    r_{13}c_2-r_{12}c_3 &  r_{23}c_2-r_{22}c_3 \\
    -r_{13}c_1+r_{11}c_3 & -r_{23}c_1+r_{21}c_3
\end{bmatrix},
\end{equation}
which is a scaled rotation if and only if
\begin{align}
\begin{split}
    r_{13}c_2-r_{12}c_3 &=  -r_{23}c_1+r_{21}c_3 \\
    r_{23}c_2-r_{22}c_3 &= r_{13}c_1-r_{11}c_3.
\end{split}
\end{align}
By multiplying the first equation with $r_{23}$, the second with $r_{13}$, subtracting the first from the second equation and collecting terms we obtain
\begin{equation}\label{eq:c2relation}
(r_{13}^2+r_{23}^2)c_1 = (r_{23}(r_{21}+r_{12})+r_{13}(r_{11}-r_{22}))c_3.
\end{equation}
We can rewrite the left-hand-side as $(1 - r_{33}^2)c_1$ by using that the last column of $R$ has unit norm.
Further, since the cross product of the first two rows of $R$ equals the third row we have
$r_{23}r_{12} - r_{13}r_{22}=r_{31}$
and since the columns of $R$ are orthogonal we have
$r_{23}r_{21}+r_{13}r_{11}=-r_{33}r_{31}$.
Thus, we can rewrite \eqref{eq:c2relation} as
\begin{equation}
    (1+r_{33})(1-r_{33})c_1 = r_{31}(1-r_{33})c_3.
\end{equation}
This means that when $r_{33}\neq\pm1$, we obtain $c_1=r_{31}c_3/(1+r_{33})$, as we wanted to show.
\end{proof}

\section{Experiments} 
\label{sec:exp}


We evaluate the proposed triangulation method on 5000 randomly sampled image pairs from the Pantheon collection of the Image Matching Competition training set \cite{image-matching-challenge-2020, image-matching-challenge-2024}, which has an associated 3D reconstruction from COLMAP~\cite{schoenberger2016mvs, schoenberger2016sfm}.
We select pairs which have at least 100 covisible 3D points in the COLMAP reconstruction.
Figure~\ref{fig:eigRatios} shows the distribution of the eigenvalue ratio~\eqref{eq:eigRatio} over the selected image pairs.
Clearly, for most image pairs, the ratio is close to $1$, meaning that our reweighted cost function is close to the original unweighted cost function.

In Figure~\ref{fig:errors}, we show results for our method compared to Lindstrom's~\cite{lindstrom2010triangulation} and Hartley-Sturm's~\cite{hartley1997triangulation} methods.\footnote{
When comparing with Lindstrom, we refer to his \texttt{niter2}-method which is the fastest variant.
}
We take 2D correspondences from the COLMAP reconstruction, compute new 2D points using the respective methods and measure the 2D distance both from the projected associated 3D point (which has been refined by bundle adjustment in COLMAP) and to the original 2D keypoints.
The three methods all perform very similarly, with our method falling slightly behind the others.
Lindstrom's method is furthermore very efficient, $1.3$ -- $1.4$ times faster than our method in our optimized implementations.
We have not implemented an optimized version of Hartley-Sturm's method, but Lindstrom reports that his method is around $50$ times faster than a fast implementation of Hartley-Sturm's.
Hence, practically, we recommend using Lindstrom's method,
\new{
except if the eigenvalue ratio is known to be $1$ (see Proposition~\ref{prop:F22orthogonal} for when this happens) in which case our method is optimal.
}

In Figure~\ref{fig:bounds}, we show the approximating quality of the bounds in \new{equation~\eqref{eq: upper bound} (denoted ``Best upper bound'') and} equation~\eqref{eq: bounds}.
The markers are scaled according to the eigenvalue ratio~\eqref{eq:eigRatio}, illustrating the fact that the bounds become worse with increased ratio as predicted by the theory.
Our bounds can be used for outlier rejection, with computation time roughly identical to the Sampson error~\cite{sampson1982fitting, rydell2024revisiting}.\footnote{
$\vec{P}(\vec{F})$ in Proposition~\ref{prop:critPtsCorrespondence} is diagonalized using a fast SVD of ${F_{2\times2}}$.
The diagonalization is computed once per image pair.
Further,
we can avoid computationally expensive square roots when checking if the upper bound in~\eqref{eq: bounds} is smaller than an inlier threshold $r$,
by using that \mbox{$|\sqrt{\alpha} - \sqrt{\beta}| < r$} is equivalent to \mbox{$(\alpha + \beta - r^2)^2 < 4\alpha\beta$}.
} \new{Under data-dependent assumptions, the Sampson method provides bounds for the true error \cite[Section 3]{rydell2024revisiting}, whereas our bounds do not require any such assumptions. However, we find that in practice, it should be recommended to use the Sampson error as it more closely aligns with the optimal value. In particular, we find that} the noise levels need to be extremely high for Sampson to fail in approximating the true error well, so this is not a practically relevant issue, see Figure~\ref{fig:boundsWithNoise}.

In summary, the experiments show that our triangulation method does not outperform the state-of-the-art but performs competitively. In settings where the optical axes are parallel, our method is preferred as it only involves solving a quadratic equation and it is guaranteed to be optimal.
This is promising for future development of reweighted cost functions for other problems in multiple view geometry.


\begin{figure}
    \centering
    \includegraphics[width=0.99\linewidth]{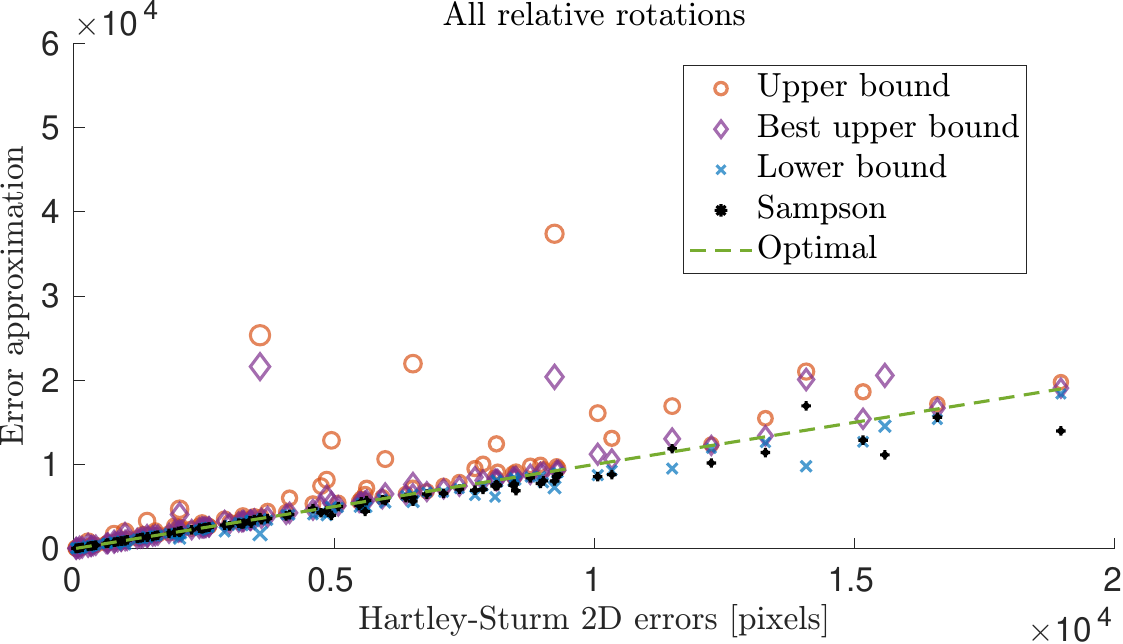}
    \caption{Evaluation of the error bounds~\eqref{eq: upper bound} and \eqref{eq: bounds} under a large amount of noise on the 2D correspondences. The markers are scaled (logarithmically) by the eigenvalue ratio~\eqref{eq:eigRatio} as in Figure~\ref{fig:bounds}. We subsample 100 random correspondences to plot, for ease of viewing.
    At this point Sampson fails to approximate the optimal errors well, but all of the correspondences would anyway be classified as outliers, so practically this failure is not relevant.}
    \label{fig:boundsWithNoise}
\end{figure}

\section{Conclusions}
We showed that diagonalizing the constraint in optimal 2-view triangulation makes it possible to devise a weighted optimization objective such that the problem reduces from finding the roots of a degree 6 polynomial to finding the roots of a degree 2 polynomial.
Further, we showed how to choose the weights to perturb the minimum as little as possible from the unweighted objective.
We also derived several bounds on the unweighted objective as direct consequences.

While our experiments showed that prior methods may be preferable in practice for some settings, we also found that the methods developed in this paper are close to the state-of-the-art and they provide approximation guarantees.
This illustrates the potential of reweighting the objective to reduce the ED-degree as a method in multiple view geometry, and algebraic optimization more generally, which opens up new avenues for future research. As an example, can the 3-view triangulation problem --- which is known to have ED-degree $47$ \cite{stewe-iccv-2005,maxim2020euclidean} --- be simplifed by reweighting?

\section*{Acknowledgements}
All authors have been supported by the Wallenberg AI, Autonomous Systems and Software Program (WASP)  funded by the Knut and Alice Wallenberg Foundation.

\printbibliography

\newpage
\newgeometry{margin=1cm, bmargin=2.5cm}
\onecolumn
\appendix
{\footnotesize
\input{Code/supp_mac_text}
}
\end{document}

%% file: preamble.tex
\usepackage{tikz}
\usepackage{pgfplots}
\pgfplotsset{compat=1.18}
\usepackage{enumerate}

\usepackage{graphicx}
\usepackage{amsmath}
\usepackage{amsthm}
\usepackage{amssymb}
\usepackage{booktabs}
\usepackage{mathtools}
\usepackage{xcolor}
\usepackage{comment}
\usepackage{tikz}
\usepackage{pgfplots}
\usepackage[normalem]{ulem}
\usetikzlibrary{decorations.pathreplacing}
\usepackage[accsupp]{axessibility}
\usepackage{nicematrix}

\usepackage{listings}
\definecolor{codedarkgreen}{RGB}{51, 133, 4}
\definecolor{codemaroon}{RGB}{133, 5, 63}
\definecolor{codeteal}{RGB}{0, 128, 96}
\lstdefinelanguage{Macaulay2}{
basicstyle=\small\ttfamily,
alsoletter=",
classoffset=1,
keywords={sub,join,matrix,minors,gb,transpose,det,ideal,apply,subsets,ker,gens,fold,flatten,entries},
keywordstyle={\color{blue}},
classoffset=2,
morekeywords={from, to, list},
keywordstyle={\color{codemaroon}},
classoffset=3,
morekeywords={QQ},
keywordstyle={\color{codedarkgreen}},
classoffset=4,
morekeywords={MonomialOrder},
keywordstyle={\color{codeteal}},
xleftmargin=1.5cm,
xrightmargin=1em,
columns=fullflexible,
keepspaces=true,
stepnumber=1,
numbers=none,
captionpos=b,
showspaces=false,
frame=none
}

\renewcommand{\vec}[1]{\boldsymbol{#1}}
\newcommand{\todo}[1]{{\color{red} #1}}
\newcommand{\tp}{\intercal}

\newcommand{\PP}{\mathbb P}
\newcommand{\RR}{\mathbb R}
\newcommand{\ee}{\boldsymbol{\varepsilon}}
\newcommand{\xx}{\boldsymbol{x}}

\newcommand{\yy}{\boldsymbol{y}}

\newcommand{\norm}[1]{\left\|#1\right\|}

\newtheorem{theorem}{Theorem}
\numberwithin{theorem}{section}
\newtheorem{proposition}[theorem]{Proposition}
\newtheorem{lemma}[theorem]{Lemma}

\newtheorem{remark}[theorem]{Remark}

\theoremstyle{definition}
\newtheorem{exmp}{Example}[section]

\newcommand{\new}[1]{{\color{black}{#1}}}

%% file: Code/lst-Macaulay2.tex
\usepackage{listings}
\usepackage{xcolor}
\definecolor{maccolor}{rgb}{0.3,0.3,0.8}
\lstdefinelanguage{Macaulay2}
{
basicstyle={\ttfamily},
keywordstyle={\color{maccolor!80!black}},
commentstyle={\color{gray}},
stringstyle={\color{red!40!black}},
rulecolor=\color{maccolor},
basewidth={1.2ex}, 
sensitive=false,
morecomment=[l]{--},
morecomment=[s]{-*}{*-},
morestring=[b]",
escapechar={`},
escapebegin={\rmfamily},
morekeywords={about,abs,AbstractToricVarieties,accumulate,Acknowledgement,acos,acosh,acot,addCancelTask,addDependencyTask,addEndFunction,addHook,AdditionalPaths,addStartFunction,addStartTask,Adjacent,adjoint,AdjointIdeal,AffineVariety,AfterEval,AfterNoPrint,AfterPrint,agm,AInfinity,alarm,AlgebraicSplines,Algorithm,Alignment,all,AllCodimensions,allowableThreads,ambient,analyticSpread,Analyzer,AnalyzeSheafOnP1,ancestor,ancestors,ANCHOR,and,andP,AngleBarList,ann,annihilator,antipode,any,append,applicationDirectory,applicationDirectorySuffix,apply,applyKeys,applyPairs,applyTable,applyValues,apropos,argument,Array,arXiv,Ascending,ascii,asin,asinh,ass,assert,associatedGradedRing,associatedPrimes,AssociativeAlgebras,AssociativeExpression,atan,atan2,atEndOfFile,Authors,autoload,AuxiliaryFiles,backtrace,Bag,Bareiss,baseFilename,BaseFunction,baseName,baseRing,baseRings,BaseRow,BasicList,basis,BasisElementLimit,Bayer,BeforePrint,beginDocumentation,BeginningMacaulay2,Benchmark,benchmark,Bertini,BesselJ,BesselY,betti,BettiCharacters,BettiTally,between,BGG,BIBasis,Binary,BinaryOperation,Binomial,binomial,BinomialEdgeIdeals,Binomials,BKZ,BlockMatrix,BLOCKQUOTE,BODY,Body,BoijSoederberg,BOLD,Book3264Examples,Boolean,BooleanGB,borel,Boxes,BR,break,Browse,Bruns,cache,CacheExampleOutput,CacheFunction,CacheTable,cacheValue,CallLimit,cancelTask,capture,catch,Caveat,CC,CDATA,ceiling,Center,centerString,Certification,ChainComplex,chainComplex,ChainComplexExtras,ChainComplexMap,ChainComplexOperations,ChangeMatrix,char,CharacteristicClasses,characters,charAnalyzer,check,CheckDocumentation,chi,Chordal,class,Classic,clean,clearAll,clearEcho,clearOutput,close,closeIn,closeOut,ClosestFit,CODE,code,codim,CodimensionLimit,coefficient,CoefficientRing,coefficientRing,coefficients,Cofactor,CohenEngine,CohenTopLevel,CoherentSheaf,CohomCalg,cohomology,coimage,CoincidentRootLoci,coker,cokernel,collectGarbage,columnAdd,columnate,columnMult,columnPermute,columnRankProfile,columnSwap,combine,Command,commandInterpreter,commandLine,COMMENT,commonest,commonRing,comodule,CompactMatrix,compactMatrixForm,CompiledFunction,CompiledFunctionBody,CompiledFunctionClosure,Complement,complement,complete,CompleteIntersection,CompleteIntersectionResolutions,Complexes,ComplexField,components,compose,compositions,compress,concatenate,conductor,ConductorElement,cone,Configuration,ConformalBlocks,conjugate,connectionCount,Consequences,Constant,Constants,constParser,content,continue,contract,Contributors,ConvexInterface,conwayPolynomial,ConwayPolynomials,copy,copyDirectory,copyFile,copyright,Core,CorrespondenceScrolls,cos,cosh,cot,CotangentSchubert,cotangentSheaf,coth,cover,coverMap,cpuTime,createTask,Cremona,csc,csch,current,currentColumnNumber,currentDirectory,currentFileDirectory,currentFileName,currentLayout,currentLineNumber,currentPackage,currentString,currentTime,Cyclotomic,Database,Date,DD,dd,deadParser,debug,debugError,DebuggingMode,debuggingMode,debugLevel,DecomposableSparseSystems,Decompose,decompose,deepSplice,Default,default,defaultPrecision,Degree,degree,degreeLength,DegreeLift,DegreeLimit,DegreeMap,DegreeOrder,DegreeRank,Degrees,degrees,degreesMonoid,degreesRing,delete,demark,denominator,Dense,Density,Depth,depth,Descending,Descent,Describe,describe,Description,det,determinant,DeterminantalRepresentations,DGAlgebras,diagonalMatrix,diameter,Dictionary,dictionary,dictionaryPath,diff,DiffAlg,difference,dim,directSum,disassemble,discriminant,dismiss,Dispatch,distinguished,DIV,Divide,divideByVariable,DivideConquer,DividedPowers,Divisor,DL,Dmodules,do,doc,docExample,docTemplate,document,DocumentTag,Down,drop,DT,dual,eagonNorthcott,EagonResolution,echoOff,echoOn,EdgeIdeals,edit,EigenSolver,eigenvalues,eigenvectors,eint,EisenbudHunekeVasconcelos,elapsedTime,elapsedTiming,elements,Eliminate,eliminate,Elimination,EliminationMatrices,EllipticCurves,EllipticIntegrals,else,EM,Email,End,end,endl,endPackage,Engine,engineDebugLevel,EngineRing,EngineTests,entries,EnumerationCurves,environment,Equation,EquivariantGB,erase,erf,erfc,error,errorDepth,euler,EulerConstant,eulers,even,EXAMPLE,ExampleFiles,ExampleItem,examples,ExampleSystems,Exclude,exec,exit,exp,expectedReesIdeal,expm1,exponents,export,exportFrom,exportMutable,Expression,expression,Ext,extend,ExteriorIdeals,ExteriorModules,exteriorPower,Factor,factor,false,Fano,FastMinors,FastNonminimal,FGLM,File,fileDictionaries,fileExecutable,fileExists,fileExitHooks,fileLength,fileMode,FileName,FilePosition,fileReadable,fileTime,fileWritable,fillMatrix,findFiles,findHeft,FindOne,findProgram,findSynonyms,FiniteFittingIdeals,First,first,firstkey,FirstPackage,fittingIdeal,flagLookup,FlatMonoid,flatten,flattenRing,Flexible,flip,floor,flush,fold,FollowLinks,for,forceGB,fork,FormalGroupLaws,Format,format,formation,FourierMotzkin,FourTiTwo,fpLLL,frac,fraction,FractionField,frames,FrobeniusThresholds,from,fromDividedPowers,fromDual,Function,FunctionApplication,FunctionBody,functionBody,FunctionClosure,FunctionFieldDesingularization,fusePairs,futureParser,GaloisField,Gamma,gb,GBDegrees,gbRemove,gbSnapshot,gbTrace,gcd,gcdCoefficients,gcdLLL,GCstats,genera,GeneralOrderedMonoid,GenerateAssertions,generateAssertions,generator,generators,Generic,GenericInitialIdeal,genericMatrix,genericSkewMatrix,genericSymmetricMatrix,gens,genus,get,getc,getChangeMatrix,getenv,getGlobalSymbol,getNetFile,getNonUnit,getPrimeWithRootOfUnity,getSymbol,getWWW,GF,gfanInterface,Givens,GKMVarieties,GLex,Global,global,globalAssign,globalAssignFunction,GlobalAssignHook,globalAssignment,globalAssignmentHooks,GlobalDictionary,GlobalHookStore,globalReleaseFunction,GlobalReleaseHook,Gorenstein,GradedLieAlgebras,GradedModule,gradedModule,GradedModuleMap,gradedModuleMap,gramm,GraphicalModels,GraphicalModelsMLE,Graphics,graphIdeal,graphRing,Graphs,Grassmannian,GRevLex,GroebnerBasis,groebnerBasis,GroebnerBasisOptions,GroebnerStrata,GroebnerWalk,groupID,GroupLex,GroupRevLex,GTZ,Hadamard,handleInterrupts,HardDegreeLimit,hash,HashTable,hashTable,HEAD,HEADER1,HEADER2,HEADER3,HEADER4,HEADER5,HEADER6,HeaderType,Heading,Headline,Heft,heft,Height,height,help,Hermite,hermite,Hermitian,HH,hh,HigherCIOperators,HighestWeights,Hilbert,hilbertFunction,hilbertPolynomial,hilbertSeries,HodgeIntegrals,hold,Holder,Hom,homeDirectory,HomePage,Homogeneous,Homogeneous2,homogenize,homology,homomorphism,HomotopyLieAlgebra,hooks,horizontalJoin,HorizontalSpace,HR,HREF,HTML,html,httpHeaders,Hybrid,HyperplaneArrangements,Hypertext,hypertext,HypertextContainer,HypertextParagraph,icFracP,icFractions,icMap,icPIdeal,id,Ideal,ideal,idealizer,identity,if,IgnoreExampleErrors,ii,image,imaginaryPart,IMG,ImmutableType,importFrom,in,incomparable,Increment,independentSets,indeterminate,IndeterminateNumber,Index,index,indexComponents,IndexedVariable,IndexedVariableTable,indices,inducedMap,inducesWellDefinedMap,InexactField,InexactFieldFamily,InexactNumber,InfiniteNumber,infinity,info,InfoDirSection,infoHelp,Inhomogeneous,input,Inputs,insert,installAssignmentMethod,installedPackages,installHilbertFunction,installMethod,installMinprimes,installPackage,InstallPrefix,instance,instances,IntegralClosure,integralClosure,integrate,IntermediateMarkUpType,interpreterDepth,intersect,intersectInP,Intersection,intersection,interval,InvariantRing,inverse,InverseMethod,inversePermutation,Inverses,inverseSystem,InverseSystems,Invertible,InvolutiveBases,irreducibleCharacteristicSeries,irreducibleDecomposition,isAffineRing,isANumber,isBorel,isCanceled,isCommutative,isConstant,isDirectory,isDirectSum,isEmpty,isField,isFinite,isFinitePrimeField,isFreeModule,isGlobalSymbol,isHomogeneous,isIdeal,isInfinite,isInjective,isInputFile,isIsomorphism,isLinearType,isListener,isLLL,isMember,isModule,isMonomialIdeal,isNormal,isOpen,isOutputFile,isPolynomialRing,isPrimary,isPrime,isPrimitive,isPseudoprime,isQuotientModule,isQuotientOf,isQuotientRing,isReady,isReal,isReduction,isRegularFile,isRing,isSkewCommutative,isSorted,isSquareFree,isStandardGradedPolynomialRing,isSubmodule,isSubquotient,isSubset,isSupportedInZeroLocus,isSurjective,isTable,isUnit,isWellDefined,isWeylAlgebra,ITALIC,Iterate,Jacobian,jacobian,jacobianDual,Jets,Join,join,Jupyter,K3Carpets,K3Surfaces,Keep,KeepFiles,KeepZeroes,ker,kernel,kernelLLL,kernelOfLocalization,Key,keys,Keyword,Keywords,kill,koszul,Kronecker,KustinMiller,LABEL,last,lastMatch,LATER,LatticePolytopes,Layout,lcm,leadCoefficient,leadComponent,leadMonomial,leadTerm,Left,left,length,LengthLimit,letterParser,Lex,LexIdeals,LI,Licenses,LieTypes,lift,liftable,Limit,limitFiles,limitProcesses,Linear,LinearAlgebra,LinearTruncations,lineNumber,lines,LINK,linkFile,List,list,listForm,listLocalSymbols,listSymbols,listUserSymbols,LITERAL,LLL,LLLBases,lngamma,load,loadDepth,LoadDocumentation,loadedFiles,loadedPackages,loadPackage,Local,local,localDictionaries,LocalDictionary,localize,LocalRings,locate,log,log1p,LongPolynomial,lookup,lookupCount,LowerBound,LUdecomposition,M0nbar,M2CODE,Macaulay2Doc,makeDirectory,MakeDocumentation,makeDocumentTag,MakeHTML,MakeInfo,MakeLinks,makePackageIndex,MakePDF,makeS2,Manipulator,map,MapExpression,MapleInterface,markedGB,Markov,MarkUpType,match,mathML,Matrix,matrix,MatrixExpression,Matroids,max,maxAllowableThreads,maxExponent,MaximalRank,maxPosition,MaxReductionCount,MCMApproximations,member,memoize,memoizeClear,memoizeValues,MENU,merge,mergePairs,META,method,MethodFunction,MethodFunctionBinary,MethodFunctionSingle,MethodFunctionWithOptions,methodOptions,methods,midpoint,min,minExponent,mingens,mingle,minimalBetti,MinimalGenerators,MinimalMatrix,minimalPresentation,minimalPresentationMap,minimalPresentationMapInv,MinimalPrimes,minimalPrimes,minimalReduction,Minimize,minimizeFilename,MinimumVersion,minors,minPosition,minPres,minprimes,Minus,minus,Miura,MixedMultiplicity,mkdir,mod,Module,module,ModuleDeformations,modulo,MonodromySolver,Monoid,monoid,MonoidElement,Monomial,MonomialAlgebras,monomialCurveIdeal,MonomialIdeal,monomialIdeal,MonomialIntegerPrograms,MonomialOrbits,MonomialOrder,Monomials,monomials,MonomialSize,monomialSubideal,moveFile,multidegree,multidoc,multigraded,MultigradedBettiTally,MultiGradedRationalMap,multiplicity,MultiplicitySequence,MultiplierIdeals,MultiplierIdealsDim2,MultiprojectiveVarieties,mutable,MutableHashTable,mutableIdentity,MutableList,MutableMatrix,mutableMatrix,NAGtypes,Name,nanosleep,Nauty,NautyGraphs,NCAlgebra,NCLex,needs,needsPackage,Net,net,NetFile,netList,new,newClass,newCoordinateSystem,NewFromMethod,newline,NewMethod,newNetFile,NewOfFromMethod,NewOfMethod,newPackage,newRing,nextkey,nextPrime,nil,NNParser,NoetherianOperators,NoetherNormalization,NonminimalComplexes,nonspaceAnalyzer,NoPrint,norm,normalCone,Normaliz,NormalToricVarieties,not,Nothing,notify,notImplemented,NTL,null,nullaryMethods,nullhomotopy,nullParser,nullSpace,Number,number,NumberedVerticalList,numcols,numColumns,numerator,numeric,NumericalAlgebraicGeometry,NumericalCertification,NumericalImplicitization,NumericalLinearAlgebra,NumericalSchubertCalculus,numericInterval,NumericSolutions,numgens,numRows,numrows,odd,oeis,of,ofClass,OL,OldPolyhedra,OldToricVectorBundles,on,OneExpression,OnlineLookup,OO,oo,ooo,oooo,openDatabase,openDatabaseOut,openFiles,openIn,openInOut,openListener,OpenMath,openOut,openOutAppend,operatorAttributes,Option,OptionalComponentsPresent,optionalSignParser,Options,options,OptionTable,optP,or,Order,order,OrderedMonoid,orP,OutputDictionary,Outputs,override,pack,Package,package,PackageCitations,PackageDictionary,PackageExports,PackageImports,PackageTemplate,packageTemplate,pad,pager,PairLimit,pairs,PairsRemaining,PARA,Parametrization,parent,Parenthesize,Parser,Parsing,part,Partition,partition,partitions,parts,path,pdim,peek,PencilsOfQuadrics,Permanents,permanents,permutations,pfaffians,PHCpack,PhylogeneticTrees,pi,PieriMaps,pivots,PlaneCurveSingularities,plus,poincare,poincareN,Points,polarize,poly,Polyhedra,Polymake,PolynomialRing,Posets,Position,position,positions,PositivityToricBundles,POSIX,Postfix,Power,power,powermod,PRE,Precision,precision,Prefix,prefixDirectory,prefixPath,preimage,prepend,presentation,pretty,primaryComponent,PrimaryDecomposition,primaryDecomposition,PrimaryTag,PrimitiveElement,Print,print,printerr,printingAccuracy,printingLeadLimit,printingPrecision,printingSeparator,printingTimeLimit,printingTrailLimit,printString,printWidth,processID,Product,product,ProductOrder,profile,profileSummary,Program,programPaths,ProgramRun,Proj,Projective,ProjectiveHilbertPolynomial,projectiveHilbertPolynomial,ProjectiveVariety,promote,protect,Prune,prune,PruneComplex,pruningMap,Pseudocode,pseudocode,pseudoRemainder,Pullback,PushForward,pushForward,Python,QQ,QQParser,QRDecomposition,QthPower,Quasidegrees,QuaternaryQuartics,QuillenSuslin,quit,Quotient,quotient,quotientRemainder,QuotientRing,Radical,radical,RadicalCodim1,radicalContainment,RaiseError,random,RandomCanonicalCurves,RandomComplexes,RandomCurves,RandomCurvesOverVerySmallFiniteFields,RandomGenus14Curves,RandomIdeals,randomKRationalPoint,RandomMonomialIdeals,randomMutableMatrix,RandomObjects,RandomPlaneCurves,RandomPoints,RandomSpaceCurves,Range,rank,RationalMaps,RationalPoints,RationalPoints2,ReactionNetworks,read,readDirectory,readlink,readPackage,RealField,RealFP,realPart,realpath,RealQP,RealQP1,RealRoots,RealRR,RealXD,recursionDepth,recursionLimit,Reduce,reducedRowEchelonForm,reduceHilbert,reductionNumber,ReesAlgebra,reesAlgebra,reesAlgebraIdeal,reesIdeal,References,ReflexivePolytopesDB,regex,regexQuote,registerFinalizer,regSeqInIdeal,Regularity,regularity,relations,RelativeCanonicalResolution,relativizeFilename,Reload,remainder,RemakeAllDocumentation,remove,removeDirectory,removeFile,removeLowestDimension,reorganize,replace,RerunExamples,res,reshape,ResidualIntersections,ResLengthThree,Resolution,resolution,ResolutionsOfStanleyReisnerRings,restart,Result,resultant,Resultants,return,returnCode,Reverse,reverse,RevLex,Right,right,Ring,ring,RingElement,RingFamily,ringFromFractions,RingMap,rootPath,roots,rootURI,rotate,round,rowAdd,RowExpression,rowMult,rowPermute,rowRankProfile,rowSwap,RR,RRi,rsort,run,RunDirectory,RunExamples,RunExternalM2,runHooks,runLengthEncode,runProgram,same,saturate,Saturation,scan,scanKeys,scanLines,scanPairs,scanValues,schedule,schreyerOrder,Schubert,Schubert2,SchurComplexes,SchurFunctors,SchurRings,SCRIPT,scriptCommandLine,ScriptedFunctor,SCSCP,searchPath,sec,sech,SectionRing,SeeAlso,seeParsing,SegreClasses,select,selectInSubring,selectVariables,SelfInitializingType,SemidefiniteProgramming,Seminormalization,separate,SeparateExec,separateRegexp,Sequence,sequence,Serialization,serialNumber,Set,set,setEcho,setGroupID,setIOExclusive,setIOSynchronized,setIOUnSynchronized,setRandomSeed,setup,setupEmacs,sheaf,SheafExpression,sheafExt,sheafHom,SheafOfRings,shield,ShimoyamaYokoyama,short,show,showClassStructure,showHtml,showStructure,showTex,showUserStructure,SimpleDoc,simpleDocFrob,SimplicialComplexes,SimplicialDecomposability,SimplicialPosets,SimplifyFractions,sin,singularLocus,sinh,size,size2,SizeLimit,SkewCommutative,SlackIdeals,sleep,SLnEquivariantMatrices,SLPexpressions,SMALL,smithNormalForm,solve,someTerms,Sort,sort,sortColumns,SortStrategy,source,SourceCode,SourceRing,SPACE,SpaceCurves,SPAN,span,SparseMonomialVectorExpression,SparseResultants,SparseVectorExpression,Spec,SpechtModule,SpecialFanoFourfolds,specialFiber,specialFiberIdeal,SpectralSequences,splice,splitWWW,sqrt,SRdeformations,stack,stacksProject,Standard,standardForm,standardPairs,StartWithOneMinor,stashValue,StatePolytope,StatGraphs,status,stderr,stdio,step,StopBeforeComputation,stopIfError,StopWithMinimalGenerators,Strategy,String,STRONG,StronglyStableIdeals,STYLE,Style,style,SUB,sub,SubalgebraBases,sublists,submatrix,submatrixByDegrees,Subnodes,subquotient,SubringLimit,Subscript,subscript,SUBSECTION,subsets,substitute,substring,subtable,Sugarless,Sum,sum,SumOfTwists,SumsOfSquares,SUP,super,SuperLinearAlgebra,Superscript,superscript,support,SVD,SVDComplexes,switch,SwitchingFields,sylvesterMatrix,Symbol,symbol,SymbolBody,symbolBody,SymbolicPowers,symlinkDirectory,symlinkFile,symmetricAlgebra,symmetricAlgebraIdeal,symmetricKernel,SymmetricPolynomials,symmetricPower,synonym,SYNOPSIS,syz,Syzygies,SyzygyLimit,SyzygyMatrix,SyzygyRows,syzygyScheme,TABLE,Table,table,take,Tally,tally,tan,TangentCone,tangentCone,tangentSheaf,tanh,target,Task,taskResult,TateOnProducts,TD,temporaryFileName,tensor,tensorAssociativity,TensorComplexes,terminalParser,terms,TEST,Test,testExample,testHunekeQuestion,TestIdeals,TestInput,tests,TEX,tex,TeXmacs,texMath,Text,TH,then,Thing,ThinSincereQuivers,ThreadedGB,threadVariable,Threshold,throw,Time,time,times,timing,TITLE,TO,to,TO2,toAbsolutePath,toCC,toDividedPowers,toDual,toExternalString,toField,TOH,toList,toLower,top,top,topCoefficients,Topcom,topComponents,topLevelMode,Tor,TorAlgebra,Toric,ToricInvariants,ToricTopology,ToricVectorBundles,toRR,toRRi,toSequence,toString,TotalPairs,toUpper,TR,trace,transpose,TriangularSets,Tries,Trim,trim,Triplets,Tropical,true,Truncate,truncate,truncateOutput,Truncations,try,TSpreadIdeals,TT,tutorial,Type,TypicalValue,typicalValues,UL,ultimate,unbag,uncurry,Undo,undocumented,uniform,uninstallAllPackages,uninstallPackage,Unique,unique,Units,Unmixed,unsequence,unstack,Up,UpdateOnly,UpperTriangular,URL,urlEncode,Usage,use,UseCachedExampleOutput,UseHilbertFunction,UserMode,userSymbols,UseSyzygies,utf8,utf8check,validate,value,values,Variable,VariableBaseName,Variables,Variety,variety,vars,Vasconcelos,Vector,vector,VectorExpression,VectorFields,VectorGraphics,Verbose,Verbosity,Verify,VersalDeformations,versalEmbedding,Version,version,VerticalList,VerticalSpace,viewHelp,VirtualResolutions,VirtualTally,VisibleList,Visualize,wait,WebApp,wedgeProduct,weightRange,Weights,WeylAlgebra,WeylGroups,when,whichGm,while,width,wikipedia,Wrap,wrap,WrapperType,XML,xor,youngest,zero,ZeroExpression,zeta,ZZ,ZZParser}
}
\lstalias{Macaulay2output}{Macaulay2}

%% file: Code/supp_mac_text.tex
\section{Macaulay2 code}
This appendix contains the Macaulay2 code used for computations in the paper, as well as its output.
The Macaulay2 code is also separately provided as \texttt{.m2}-files, with further calculations.
We change notation from $\lambda,\mu,\nu$ to $l,m,n$ in the computations, for compatibility of character encodings between Macaulay2 and \LaTeX.

\subsection{Proposition \texorpdfstring{\ref{prop:F22orthogonal}}{4.5}}
{ 
\begin{lstlisting}[language=Macaulay2output]
`\underline{\tt i1}` : --the case when c is given as a formula in terms of R: +
     clearAll
`\underline{\tt i2}` : A = QQ[r_(1,1)..r_(3,3),c_3]
`\underline{\tt o2}` = `$A$`
`\underline{\tt o2}` : `$\texttt{PolynomialRing}$`
`\underline{\tt i3}` : R = matrix apply(3, i -> apply(3, j -> r_(i+1,j+1)))
`\underline{\tt o3}` = `$\left(\!\begin{array}{ccc}
r_{1,1}&r_{1,2}&r_{1,3}\\
r_{2,1}&r_{2,2}&r_{2,3}\\
r_{3,1}&r_{3,2}&r_{3,3}
\end{array}\!\right)$`
`\underline{\tt o3}` : `$\texttt{Matrix}$ $A^{3}\,\longleftarrow \,A^{3}$`
`\underline{\tt i4}` : c_1 = r_(3,1)*c_3/(r_(3,3)+1)
`\underline{\tt o4}` = `$\frac{r_{3,1}c_{3}}{r_{3,3}+1}$`
`\underline{\tt o4}` : `$\texttt{frac}\ A$`
`\underline{\tt i5}` : c_2 = r_(3,2)*c_3/(r_(3,3)+1)
`\underline{\tt o5}` = `$\frac{r_{3,2}c_{3}}{r_{3,3}+1}$`
`\underline{\tt o5}` : `$\texttt{frac}\ A$`
`\underline{\tt i6}` : C = matrix{{c_1},{c_2},{c_3}}
`\underline{\tt o6}` = `$(\!\begin{array}{c}
\frac{r_{3,1}c_{3}}{r_{3,3}+1}\\
\frac{r_{3,2}c_{3}}{r_{3,3}+1}\\
c_{3}
\end{array}\!)$`
`\underline{\tt o6}` : `$\texttt{Matrix}$ $(\texttt{frac}\ A)^{3}\,\longleftarrow \,(\texttt{frac}\ A)^{1}$`
`\underline{\tt i7}` : t = -R*C
`\underline{\tt o7}` = `$(\!\begin{array}{c}
\frac{-r_{1,1}r_{3,1}c_{3}-r_{1,2}r_{3,2}c_{3}-r_{1,3}r_{3,3}c_{3}-r_{1,3}c_{3}}{r_{3,3}+1}\\
\frac{-r_{2,1}r_{3,1}c_{3}-r_{2,2}r_{3,2}c_{3}-r_{2,3}r_{3,3}c_{3}-r_{2,3}c_{3}}{r_{3,3}+1}\\
\frac{-r_{3,1}^{2}c_{3}-r_{3,2}^{2}c_{3}-r_{3,3}^{2}c_{3}-r_{3,3}c_{3}}{r_{3,3}+1}
\end{array}\!)$`
`\underline{\tt o7}` : `$\texttt{Matrix}$ $(\texttt{frac}\ A)^{3}\,\longleftarrow \,(\texttt{frac}\ A)^{1}$`
`\underline{\tt i8}` : T = matrix{{0,-t_(2,0),t_(1,0)},{t_(2,0),0,-t_(0,0)},{-t_(1,0),t_(0,0),0}}
`\underline{\tt o8}` = `$\left(\!\begin{array}{ccc}
0&\frac{r_{3,1}^{2}c_{3}+r_{3,2}^{2}c_{3}+r_{3,3}^{2}c_{3}+r_{3,3}c_{3}}{r_{3,3}+1}&\frac{-r_{2,1}r_{3,1}c_{3}-r_{2,2}r_{3,2}c_{3}-r_{2,3}r_{3,3}c_{3}-r_{2,3}c_{3}}{r_{3,3}+1}\\
\frac{-r_{3,1}^{2}c_{3}-r_{3,2}^{2}c_{3}-r_{3,3}^{2}c_{3}-r_{3,3}c_{3}}{r_{3,3}+1}&0&\frac{r_{1,1}r_{3,1}c_{3}+r_{1,2}r_{3,2}c_{3}+r_{1,3}r_{3,3}c_{3}+r_{1,3}c_{3}}{r_{3,3}+1}\\
\frac{r_{2,1}r_{3,1}c_{3}+r_{2,2}r_{3,2}c_{3}+r_{2,3}r_{3,3}c_{3}+r_{2,3}c_{3}}{r_{3,3}+1}&\frac{-r_{1,1}r_{3,1}c_{3}-r_{1,2}r_{3,2}c_{3}-r_{1,3}r_{3,3}c_{3}-r_{1,3}c_{3}}{r_{3,3}+1}&0
\end{array}\!\right)$`
`\underline{\tt o8}` : `$\texttt{Matrix}$ $(\texttt{frac}\ A)^{3}\,\longleftarrow \,(\texttt{frac}\ A)^{3}$`
`\underline{\tt i9}` : F = transpose(R)*T
`\underline{\tt o9}` = `[very long output]`
`\underline{\tt o9}` : `$\texttt{Matrix}$ $(\texttt{frac}\ A)^{3}\,\longleftarrow \,(\texttt{frac}\ A)^{3}$`
`\underline{\tt i10}` : 
      --scale F so that its entries become polynomials
      Fscaled = sub((r_(3,3)+1)*F, A)
`\underline{\tt o10}` = `[very long output]`
`\underline{\tt o10}` : `$\texttt{Matrix}$ $A^{3}\,\longleftarrow \,A^{3}$`
`\underline{\tt i11}` : 
      --now use that R is a rotation matrix
      Id = matrix apply(3, i -> apply(3, j -> if i==j then 1 else 0))
`\underline{\tt o11}` = `$(\!\begin{array}{ccc}
1&0&0\\
0&1&0\\
0&0&1
\end{array}\!)$`
`\underline{\tt o11}` : `$\texttt{Matrix}$ ${\mathbb Z}^{3}\,\longleftarrow \,{\mathbb Z}^{3}$`
`\underline{\tt i12}` : SO3 = ideal flatten entries(transpose(R)*R-Id) + ideal((det R)-1)
`\underline{\tt o12}` = `$\texttt{ideal}{}(r_{1,1}^{2}+r_{2,1}^{2}+r_{3,1}^{2}-1,\,r_{1,1}r_{1,2}+r_{2,1}r_{2,2}+r_{3,1}r_{3,2},\,r_{1,1}r_{1,3}+r_{2,1}r_{2,3}+r_{3,1}r_{3,3},\,r_{1,1}r_{1,2}+r_{2,1}r_{2,2}+r_{3,1}r_{3,2},\,r_{1,2}^{2}+r_{2,2}^{2}+r_{3,2}^{2}-1,\,r_{1,2}r_{1,3}+r_{2,2}r_{2,3}+r_{3,2}r_{3,3},\,r_{1,1}r_{1,3}+r_{2,1}r_{2,3}+r_{3,1}r_{3,3},\,r_{1,2}r_{1,3}+r_{2,2}r_{2,3}+r_{3,2}r_{3,3},\,r_{1,3}^{2}+r_{2,3}^{2}+r_{3,3}^{2}-1,\,-r_{1,3}r_{2,2}r_{3,1}+r_{1,2}r_{2,3}r_{3,1}+r_{1,3}r_{2,1}r_{3,2}-r_{1,1}r_{2,3}r_{3,2}-r_{1,2}r_{2,1}r_{3,3}+r_{1,1}r_{2,2}r_{3,3}-1)$`
`\underline{\tt o12}` : `$\texttt{Ideal}$ of $A$`
`\underline{\tt i13}` : sub(Fscaled, A/SO3)
`\underline{\tt o13}` = `$\begin{array}{l}\{-1\}\vphantom{r_{1,2}c_{3}-r_{2,1}c_{3}r_{1,1}c_{3}+r_{2,2}c_{3}r_{3,2}c_{3}}\\\{-1\}\vphantom{-r_{1,1}c_{3}-r_{2,2}c_{3}r_{1,2}c_{3}-r_{2,1}c_{3}-r_{3,1}c_{3}}\\\{-1\}\vphantom{-r_{2,3}c_{3}r_{1,3}c_{3}0}\end{array}(\!\begin{array}{ccc}
\vphantom{\{-1\}}r_{1,2}c_{3}-r_{2,1}c_{3}&r_{1,1}c_{3}+r_{2,2}c_{3}&r_{3,2}c_{3}\\
\vphantom{\{-1\}}-r_{1,1}c_{3}-r_{2,2}c_{3}&r_{1,2}c_{3}-r_{2,1}c_{3}&-r_{3,1}c_{3}\\
\vphantom{\{-1\}}-r_{2,3}c_{3}&r_{1,3}c_{3}&0
\end{array}\!)$`
`\underline{\tt o13}` : `[very long output]`
`\underline{\tt i14}` : 
      
      --the case when c is given as a formula in terms of R: -
      clearAll
`\underline{\tt i15}` : A = QQ[r_(1,1)..r_(3,3),c_3]
`\underline{\tt o15}` = `$A$`
`\underline{\tt o15}` : `$\texttt{PolynomialRing}$`
`\underline{\tt i16}` : R = matrix apply(3, i -> apply(3, j -> r_(i+1,j+1)))
`\underline{\tt o16}` = `$(\!\begin{array}{ccc}
r_{1,1}&r_{1,2}&r_{1,3}\\
r_{2,1}&r_{2,2}&r_{2,3}\\
r_{3,1}&r_{3,2}&r_{3,3}
\end{array}\!)$`
`\underline{\tt o16}` : `$\texttt{Matrix}$ $A^{3}\,\longleftarrow \,A^{3}$`
`\underline{\tt i17}` : c_1 = r_(3,1)*c_3/(r_(3,3)-1)
`\underline{\tt o17}` = `$\frac{r_{3,1}c_{3}}{r_{3,3}-1}$`
`\underline{\tt o17}` : `$\texttt{frac}\ A$`
`\underline{\tt i18}` : c_2 = r_(3,2)*c_3/(r_(3,3)-1)
`\underline{\tt o18}` = `$\frac{r_{3,2}c_{3}}{r_{3,3}-1}$`
`\underline{\tt o18}` : `$\texttt{frac}\ A$`
`\underline{\tt i19}` : C = matrix{{c_1},{c_2},{c_3}}
`\underline{\tt o19}` = `$(\!\begin{array}{c}
\frac{r_{3,1}c_{3}}{r_{3,3}-1}\\
\frac{r_{3,2}c_{3}}{r_{3,3}-1}\\
c_{3}
\end{array}\!)$`
`\underline{\tt o19}` : `$\texttt{Matrix}$ $(\texttt{frac}\ A)^{3}\,\longleftarrow \,(\texttt{frac}\ A)^{1}$`
`\underline{\tt i20}` : t = -R*C
`\underline{\tt o20}` = `$\left(\!\begin{array}{c}
\frac{-r_{1,1}r_{3,1}c_{3}-r_{1,2}r_{3,2}c_{3}-r_{1,3}r_{3,3}c_{3}+r_{1,3}c_{3}}{r_{3,3}-1}\\
\frac{-r_{2,1}r_{3,1}c_{3}-r_{2,2}r_{3,2}c_{3}-r_{2,3}r_{3,3}c_{3}+r_{2,3}c_{3}}{r_{3,3}-1}\\
\frac{-r_{3,1}^{2}c_{3}-r_{3,2}^{2}c_{3}-r_{3,3}^{2}c_{3}+r_{3,3}c_{3}}{r_{3,3}-1}
\end{array}\!\right)$`
`\underline{\tt o20}` : `$\texttt{Matrix}$ $(\texttt{frac}\ A)^{3}\,\longleftarrow \,(\texttt{frac}\ A)^{1}$`
`\underline{\tt i21}` : T = matrix{{0,-t_(2,0),t_(1,0)},{t_(2,0),0,-t_(0,0)},{-t_(1,0),t_(0,0),0}}
`\underline{\tt o21}` = `$\left(\!\begin{array}{ccc}
0&\frac{r_{3,1}^{2}c_{3}+r_{3,2}^{2}c_{3}+r_{3,3}^{2}c_{3}-r_{3,3}c_{3}}{r_{3,3}-1}&\frac{-r_{2,1}r_{3,1}c_{3}-r_{2,2}r_{3,2}c_{3}-r_{2,3}r_{3,3}c_{3}+r_{2,3}c_{3}}{r_{3,3}-1}\\
\frac{-r_{3,1}^{2}c_{3}-r_{3,2}^{2}c_{3}-r_{3,3}^{2}c_{3}+r_{3,3}c_{3}}{r_{3,3}-1}&0&\frac{r_{1,1}r_{3,1}c_{3}+r_{1,2}r_{3,2}c_{3}+r_{1,3}r_{3,3}c_{3}-r_{1,3}c_{3}}{r_{3,3}-1}\\
\frac{r_{2,1}r_{3,1}c_{3}+r_{2,2}r_{3,2}c_{3}+r_{2,3}r_{3,3}c_{3}-r_{2,3}c_{3}}{r_{3,3}-1}&\frac{-r_{1,1}r_{3,1}c_{3}-r_{1,2}r_{3,2}c_{3}-r_{1,3}r_{3,3}c_{3}+r_{1,3}c_{3}}{r_{3,3}-1}&0
\end{array}\!\right)$`
`\underline{\tt o21}` : `$\texttt{Matrix}$ $(\texttt{frac}\ A)^{3}\,\longleftarrow \,(\texttt{frac}\ A)^{3}$`
`\underline{\tt i22}` : F = transpose(R)*T
`\underline{\tt o22}` = `[very long output]`
`\underline{\tt o22}` : `$\texttt{Matrix}$ $(\texttt{frac}\ A)^{3}\,\longleftarrow \,(\texttt{frac}\ A)^{3}$`
`\underline{\tt i23}` : 
      --scale F so that its entries become polynomials
      Fscaled = sub((r_(3,3)-1)*F, A)
`\underline{\tt o23}` = `[very long output]`
`\underline{\tt o23}` : `$\texttt{Matrix}$ $A^{3}\,\longleftarrow \,A^{3}$`
`\underline{\tt i24}` : 
      --now use that R is a rotation matrix
      Id = matrix apply(3, i -> apply(3, j -> if i==j then 1 else 0))
`\underline{\tt o24}` = `$\left(\!\begin{array}{ccc}
1&0&0\\
0&1&0\\
0&0&1
\end{array}\!\right)$`
`\underline{\tt o24}` : `$\texttt{Matrix}$ ${\mathbb Z}^{3}\,\longleftarrow \,{\mathbb Z}^{3}$`
`\underline{\tt i25}` : SO3 = ideal flatten entries(transpose(R)*R-Id) + ideal((det R)-1)
`\underline{\tt o25}` = `$\texttt{ideal}{}(r_{1,1}^{2}+r_{2,1}^{2}+r_{3,1}^{2}-1,\,r_{1,1}r_{1,2}+r_{2,1}r_{2,2}+r_{3,1}r_{3,2},\,r_{1,1}r_{1,3}+r_{2,1}r_{2,3}+r_{3,1}r_{3,3},\,r_{1,1}r_{1,2}+r_{2,1}r_{2,2}+r_{3,1}r_{3,2},\,r_{1,2}^{2}+r_{2,2}^{2}+r_{3,2}^{2}-1,\,r_{1,2}r_{1,3}+r_{2,2}r_{2,3}+r_{3,2}r_{3,3},\,r_{1,1}r_{1,3}+r_{2,1}r_{2,3}+r_{3,1}r_{3,3},\,r_{1,2}r_{1,3}+r_{2,2}r_{2,3}+r_{3,2}r_{3,3},\,r_{1,3}^{2}+r_{2,3}^{2}+r_{3,3}^{2}-1,\,-r_{1,3}r_{2,2}r_{3,1}+r_{1,2}r_{2,3}r_{3,1}+r_{1,3}r_{2,1}r_{3,2}-r_{1,1}r_{2,3}r_{3,2}-r_{1,2}r_{2,1}r_{3,3}+r_{1,1}r_{2,2}r_{3,3}-1)$`
`\underline{\tt o25}` : `$\texttt{Ideal}$ of $A$`
`\underline{\tt i26}` : sub(Fscaled, A/SO3)
`\underline{\tt o26}` = `$\begin{array}{l}\{-1\}\vphantom{-r_{1,2}c_{3}-r_{2,1}c_{3}r_{1,1}c_{3}-r_{2,2}c_{3}-r_{3,2}c_{3}}\\\{-1\}\vphantom{r_{1,1}c_{3}-r_{2,2}c_{3}r_{1,2}c_{3}+r_{2,1}c_{3}r_{3,1}c_{3}}\\\{-1\}\vphantom{-r_{2,3}c_{3}r_{1,3}c_{3}0}\end{array}\left(\!\begin{array}{ccc}
\vphantom{\{-1\}}-r_{1,2}c_{3}-r_{2,1}c_{3}&r_{1,1}c_{3}-r_{2,2}c_{3}&-r_{3,2}c_{3}\\
\vphantom{\{-1\}}r_{1,1}c_{3}-r_{2,2}c_{3}&r_{1,2}c_{3}+r_{2,1}c_{3}&r_{3,1}c_{3}\\
\vphantom{\{-1\}}-r_{2,3}c_{3}&r_{1,3}c_{3}&0
\end{array}\!\right)$`
`\underline{\tt o26}` : `[very long output]`
`\underline{\tt i27}` : 
      
      --now we assume that the top left 2x2 block of F is a scaling of an orthogonal matrix
      clearAll
`\underline{\tt i28}` : A = QQ[r_(1,1)..r_(3,3),c_1..c_3]
`\underline{\tt o28}` = `$A$`
`\underline{\tt o28}` : `$\texttt{PolynomialRing}$`
`\underline{\tt i29}` : R = matrix apply(3, i -> apply(3, j -> r_(i+1,j+1)))
`\underline{\tt o29}` = `$\left(\!\begin{array}{ccc}
r_{1,1}&r_{1,2}&r_{1,3}\\
r_{2,1}&r_{2,2}&r_{2,3}\\
r_{3,1}&r_{3,2}&r_{3,3}
\end{array}\!\right)$`
`\underline{\tt o29}` : `$\texttt{Matrix}$ $A^{3}\,\longleftarrow \,A^{3}$`
`\underline{\tt i30}` : C = matrix{{c_1},{c_2},{c_3}}
`\underline{\tt o30}` = `$\left(\!\begin{array}{c}
c_{1}\\
c_{2}\\
c_{3}
\end{array}\!\right)$`
`\underline{\tt o30}` : `$\texttt{Matrix}$ $A^{3}\,\longleftarrow \,A^{1}$`
`\underline{\tt i31}` : t = -R*C
`\underline{\tt o31}` = `$\left(\!\begin{array}{c}
-r_{1,1}c_{1}-r_{1,2}c_{2}-r_{1,3}c_{3}\\
-r_{2,1}c_{1}-r_{2,2}c_{2}-r_{2,3}c_{3}\\
-r_{3,1}c_{1}-r_{3,2}c_{2}-r_{3,3}c_{3}
\end{array}\!\right)$`
`\underline{\tt o31}` : `$\texttt{Matrix}$ $A^{3}\,\longleftarrow \,A^{1}$`
`\underline{\tt i32}` : T = matrix{{0,-t_(2,0),t_(1,0)},{t_(2,0),0,-t_(0,0)},{-t_(1,0),t_(0,0),0}}
`\underline{\tt o32}` = `$\left(\!\begin{array}{ccc}
0&r_{3,1}c_{1}+r_{3,2}c_{2}+r_{3,3}c_{3}&-r_{2,1}c_{1}-r_{2,2}c_{2}-r_{2,3}c_{3}\\
-r_{3,1}c_{1}-r_{3,2}c_{2}-r_{3,3}c_{3}&0&r_{1,1}c_{1}+r_{1,2}c_{2}+r_{1,3}c_{3}\\
r_{2,1}c_{1}+r_{2,2}c_{2}+r_{2,3}c_{3}&-r_{1,1}c_{1}-r_{1,2}c_{2}-r_{1,3}c_{3}&0
\end{array}\!\right)$`
`\underline{\tt o32}` : `$\texttt{Matrix}$ $A^{3}\,\longleftarrow \,A^{3}$`
`\underline{\tt i33}` : F = transpose(R)*T
`\underline{\tt o33}` = `[very long output]`
`\underline{\tt o33}` : `$\texttt{Matrix}$ $A^{3}\,\longleftarrow \,A^{3}$`
`\underline{\tt i34}` : F22 = F_{0,1}^{0,1} --the top left 2x2 block of F
`\underline{\tt o34}` = `$\begin{array}{l}\{-1\}\vphantom{r_{2,2}r_{3,1}c_{2}-r_{2,1}r_{3,2}c_{2}+r_{2,3}r_{3,1}c_{3}-r_{2,1}r_{3,3}c_{3}-r_{1,2}r_{3,1}c_{2}+r_{1,1}r_{3,2}c_{2}-r_{1,3}r_{3,1}c_{3}+r_{1,1}r_{3,3}c_{3}}\\\{-1\}\vphantom{-r_{2,2}r_{3,1}c_{1}+r_{2,1}r_{3,2}c_{1}+r_{2,3}r_{3,2}c_{3}-r_{2,2}r_{3,3}c_{3}r_{1,2}r_{3,1}c_{1}-r_{1,1}r_{3,2}c_{1}-r_{1,3}r_{3,2}c_{3}+r_{1,2}r_{3,3}c_{3}}\end{array}\left(\!\begin{array}{cc}
\vphantom{\{-1\}}r_{2,2}r_{3,1}c_{2}-r_{2,1}r_{3,2}c_{2}+r_{2,3}r_{3,1}c_{3}-r_{2,1}r_{3,3}c_{3}&-r_{1,2}r_{3,1}c_{2}+r_{1,1}r_{3,2}c_{2}-r_{1,3}r_{3,1}c_{3}+r_{1,1}r_{3,3}c_{3}\\
\vphantom{\{-1\}}-r_{2,2}r_{3,1}c_{1}+r_{2,1}r_{3,2}c_{1}+r_{2,3}r_{3,2}c_{3}-r_{2,2}r_{3,3}c_{3}&r_{1,2}r_{3,1}c_{1}-r_{1,1}r_{3,2}c_{1}-r_{1,3}r_{3,2}c_{3}+r_{1,2}r_{3,3}c_{3}
\end{array}\!\right)$`
`\underline{\tt o34}` : `$\texttt{Matrix}$ $A^{2}\,\longleftarrow \,A^{2}$`
`\underline{\tt i35}` : 
      --the following conditions encode that F22 is scaling of an orthogonal matrix of determina +1 resp. -1
      O2plus = ideal(F22_(0,0)-F22_(1,1), F22_(0,1)+F22_(1,0)) 
`\underline{\tt o35}` = `$\texttt{ideal}{}(-r_{1,2}r_{3,1}c_{1}+r_{1,1}r_{3,2}c_{1}+r_{2,2}r_{3,1}c_{2}-r_{2,1}r_{3,2}c_{2}+r_{2,3}r_{3,1}c_{3}+r_{1,3}r_{3,2}c_{3}-r_{1,2}r_{3,3}c_{3}-r_{2,1}r_{3,3}c_{3},\,-r_{2,2}r_{3,1}c_{1}+r_{2,1}r_{3,2}c_{1}-r_{1,2}r_{3,1}c_{2}+r_{1,1}r_{3,2}c_{2}-r_{1,3}r_{3,1}c_{3}+r_{2,3}r_{3,2}c_{3}+r_{1,1}r_{3,3}c_{3}-r_{2,2}r_{3,3}c_{3})$`
`\underline{\tt o35}` : `$\texttt{Ideal}$ of $A$`
`\underline{\tt i36}` : O2minus = ideal(F22_(0,0)+F22_(1,1), F22_(0,1)-F22_(1,0)) 
`\underline{\tt o36}` = `$\texttt{ideal}{}(r_{1,2}r_{3,1}c_{1}-r_{1,1}r_{3,2}c_{1}+r_{2,2}r_{3,1}c_{2}-r_{2,1}r_{3,2}c_{2}+r_{2,3}r_{3,1}c_{3}-r_{1,3}r_{3,2}c_{3}+r_{1,2}r_{3,3}c_{3}-r_{2,1}r_{3,3}c_{3},\,r_{2,2}r_{3,1}c_{1}-r_{2,1}r_{3,2}c_{1}-r_{1,2}r_{3,1}c_{2}+r_{1,1}r_{3,2}c_{2}-r_{1,3}r_{3,1}c_{3}-r_{2,3}r_{3,2}c_{3}+r_{1,1}r_{3,3}c_{3}+r_{2,2}r_{3,3}c_{3})$`
`\underline{\tt o36}` : `$\texttt{Ideal}$ of $A$`
`\underline{\tt i37}` : 
      --now we constrain R to be a rotation matrix
      Id = matrix apply(3, i -> apply(3, j -> if i==j then 1 else 0))
`\underline{\tt o37}` = `$(\!\begin{array}{ccc}
1&0&0\\
0&1&0\\
0&0&1
\end{array}\!)$`
`\underline{\tt o37}` : `$\texttt{Matrix}$ ${\mathbb Z}^{3}\,\longleftarrow \,{\mathbb Z}^{3}$`
`\underline{\tt i38}` : SO3 = ideal flatten entries(transpose(R)*R-Id) + ideal((det R)-1)
`\underline{\tt o38}` = `$\texttt{ideal}{}(r_{1,1}^{2}+r_{2,1}^{2}+r_{3,1}^{2}-1,\,r_{1,1}r_{1,2}+r_{2,1}r_{2,2}+r_{3,1}r_{3,2},\,r_{1,1}r_{1,3}+r_{2,1}r_{2,3}+r_{3,1}r_{3,3},\,r_{1,1}r_{1,2}+r_{2,1}r_{2,2}+r_{3,1}r_{3,2},\,r_{1,2}^{2}+r_{2,2}^{2}+r_{3,2}^{2}-1,\,r_{1,2}r_{1,3}+r_{2,2}r_{2,3}+r_{3,2}r_{3,3},\,r_{1,1}r_{1,3}+r_{2,1}r_{2,3}+r_{3,1}r_{3,3},\,r_{1,2}r_{1,3}+r_{2,2}r_{2,3}+r_{3,2}r_{3,3},\,r_{1,3}^{2}+r_{2,3}^{2}+r_{3,3}^{2}-1,\,-r_{1,3}r_{2,2}r_{3,1}+r_{1,2}r_{2,3}r_{3,1}+r_{1,3}r_{2,1}r_{3,2}-r_{1,1}r_{2,3}r_{3,2}-r_{1,2}r_{2,1}r_{3,3}+r_{1,1}r_{2,2}r_{3,3}-1)$`
`\underline{\tt o38}` : `$\texttt{Ideal}$ of $A$`
`\underline{\tt i39}` : 
      --this function eats a polynomial or a rational function and returns the numerator 
      num = eq -> if (class class eq === FractionField) then numerator eq else eq
`\underline{\tt o39}` = `$\texttt{num}$`
`\underline{\tt o39}` : `$\texttt{FunctionClosure}$`
`\underline{\tt i40}` : 
      --we start by analyzing SO(2)
      decPlus = decompose(O2plus+SO3)
`\underline{\tt o40}` = `$\{\texttt{ideal}{}(r_{3,3}-1,\,r_{3,2},\,r_{3,1},\,r_{2,3},\,r_{1,3},\,r_{1,2}+r_{2,1},\,r_{1,1}-r_{2,2},\,r_{2,1}^{2}+r_{2,2}^{2}-1),\:\texttt{ideal}{}(r_{3,3}-1,\,r_{2,3}c_{1}+r_{1,3}c_{2}-r_{1,2}c_{3}-r_{2,1}c_{3},\,r_{1,3}c_{1}-r_{2,3}c_{2}-r_{1,1}c_{3}+r_{2,2}c_{3},\,r_{2,3}r_{3,2}+r_{1,1}-r_{2,2},\,r_{1,3}r_{3,2}-r_{1,2}-r_{2,1},\,r_{3,1}^{2}+r_{3,2}^{2},\,r_{2,3}r_{3,1}-r_{1,2}-r_{2,1},\,r_{2,2}r_{3,1}-r_{2,1}r_{3,2}+r_{1,3},\,r_{2,1}r_{3,1}+r_{2,2}r_{3,2}+r_{2,3},\,r_{1,3}r_{3,1}-r_{1,1}+r_{2,2},\,r_{1,2}r_{3,1}-r_{1,1}r_{3,2}-r_{2,3},\,r_{1,1}r_{3,1}+r_{1,2}r_{3,2}+r_{1,3},\,r_{1,3}r_{2,2}-r_{1,2}r_{2,3}+r_{3,1},\,r_{2,1}^{2}+r_{2,2}^{2}+r_{2,3}^{2}-1,\,r_{1,3}r_{2,1}-r_{1,1}r_{2,3}-r_{3,2},\,r_{1,2}r_{2,1}-r_{1,1}r_{2,2}+1,\,r_{1,1}r_{2,1}+r_{1,2}r_{2,2}+r_{1,3}r_{2,3},\,r_{1,3}^{2}+r_{2,3}^{2},\,r_{1,2}r_{1,3}+r_{2,2}r_{2,3}+r_{3,2},\,r_{1,1}r_{1,3}+r_{2,1}r_{2,3}+r_{3,1},\,r_{1,2}^{2}+r_{2,2}^{2}+r_{3,2}^{2}-1,\,r_{1,1}r_{1,2}+r_{2,1}r_{2,2}+r_{3,1}r_{3,2},\,r_{1,1}^{2}-r_{2,2}^{2}-r_{2,3}^{2}-r_{3,2}^{2},\,r_{3,1}c_{1}c_{3}+r_{3,2}c_{2}c_{3}-c_{1}^{2}-c_{2}^{2}),\:\texttt{ideal}{}(r_{3,3}c_{2}-r_{3,2}c_{3}+c_{2},\,r_{3,3}c_{1}-r_{3,1}c_{3}+c_{1},\,r_{3,2}c_{1}-r_{3,1}c_{2},\,r_{3,1}c_{1}+r_{3,2}c_{2}+r_{3,3}c_{3}-c_{3},\,r_{2,3}c_{1}+r_{1,3}c_{2}-r_{1,2}c_{3}-r_{2,1}c_{3},\,r_{1,3}c_{1}-r_{2,3}c_{2}-r_{1,1}c_{3}+r_{2,2}c_{3},\,r_{1,2}c_{1}-r_{2,1}c_{1}-r_{1,1}c_{2}-r_{2,2}c_{2}-r_{2,3}c_{3},\,r_{1,1}c_{1}+r_{2,2}c_{1}+r_{1,2}c_{2}-r_{2,1}c_{2}+r_{1,3}c_{3},\,r_{2,3}r_{3,2}-r_{2,2}r_{3,3}+r_{1,1},\,r_{1,3}r_{3,2}-r_{1,2}r_{3,3}-r_{2,1},\,r_{3,1}^{2}+r_{3,2}^{2}+r_{3,3}^{2}-1,\,r_{2,3}r_{3,1}-r_{2,1}r_{3,3}-r_{1,2},\,r_{2,2}r_{3,1}-r_{2,1}r_{3,2}+r_{1,3},\,r_{2,1}r_{3,1}+r_{2,2}r_{3,2}+r_{2,3}r_{3,3},\,r_{1,3}r_{3,1}-r_{1,1}r_{3,3}+r_{2,2},\,r_{1,2}r_{3,1}-r_{1,1}r_{3,2}-r_{2,3},\,r_{1,1}r_{3,1}+r_{1,2}r_{3,2}+r_{1,3}r_{3,3},\,r_{1,3}r_{2,2}-r_{1,2}r_{2,3}+r_{3,1},\,r_{2,1}^{2}+r_{2,2}^{2}+r_{2,3}^{2}-1,\,r_{1,3}r_{2,1}-r_{1,1}r_{2,3}-r_{3,2},\,r_{1,2}r_{2,1}-r_{1,1}r_{2,2}+r_{3,3},\,r_{1,1}r_{2,1}+r_{1,2}r_{2,2}+r_{1,3}r_{2,3},\,r_{1,3}^{2}+r_{2,3}^{2}+r_{3,3}^{2}-1,\,r_{1,2}r_{1,3}+r_{2,2}r_{2,3}+r_{3,2}r_{3,3},\,r_{1,1}r_{1,3}+r_{2,1}r_{2,3}+r_{3,1}r_{3,3},\,r_{1,2}^{2}+r_{2,2}^{2}+r_{3,2}^{2}-1,\,r_{1,1}r_{1,2}+r_{2,1}r_{2,2}+r_{3,1}r_{3,2},\,r_{1,1}^{2}-r_{2,2}^{2}-r_{2,3}^{2}-r_{3,2}^{2}-r_{3,3}^{2}+1)\}$`
`\underline{\tt o40}` : `$\texttt{List}$`
`\underline{\tt i41}` : --this ideal has 3 components:
      #decPlus
`\underline{\tt o41}` = `$3$`
`\underline{\tt i42}` : --the first 2 have the last entry of R equal to 1:
      decPlus#0
`\underline{\tt o42}` = `$\texttt{ideal}{}(r_{3,3}-1,\,r_{3,2},\,r_{3,1},\,r_{2,3},\,r_{1,3},\,r_{1,2}+r_{2,1},\,r_{1,1}-r_{2,2},\,r_{2,1}^{2}+r_{2,2}^{2}-1)$`
`\underline{\tt o42}` : `$\texttt{Ideal}$ of $A$`
`\underline{\tt i43}` : decPlus#1
`\underline{\tt o43}` = `$\texttt{ideal}{}(r_{3,3}-1,\,r_{2,3}c_{1}+r_{1,3}c_{2}-r_{1,2}c_{3}-r_{2,1}c_{3},\,r_{1,3}c_{1}-r_{2,3}c_{2}-r_{1,1}c_{3}+r_{2,2}c_{3},\,r_{2,3}r_{3,2}+r_{1,1}-r_{2,2},\,r_{1,3}r_{3,2}-r_{1,2}-r_{2,1},\,r_{3,1}^{2}+r_{3,2}^{2},\,r_{2,3}r_{3,1}-r_{1,2}-r_{2,1},\,r_{2,2}r_{3,1}-r_{2,1}r_{3,2}+r_{1,3},\,r_{2,1}r_{3,1}+r_{2,2}r_{3,2}+r_{2,3},\,r_{1,3}r_{3,1}-r_{1,1}+r_{2,2},\,r_{1,2}r_{3,1}-r_{1,1}r_{3,2}-r_{2,3},\,r_{1,1}r_{3,1}+r_{1,2}r_{3,2}+r_{1,3},\,r_{1,3}r_{2,2}-r_{1,2}r_{2,3}+r_{3,1},\,r_{2,1}^{2}+r_{2,2}^{2}+r_{2,3}^{2}-1,\,r_{1,3}r_{2,1}-r_{1,1}r_{2,3}-r_{3,2},\,r_{1,2}r_{2,1}-r_{1,1}r_{2,2}+1,\,r_{1,1}r_{2,1}+r_{1,2}r_{2,2}+r_{1,3}r_{2,3},\,r_{1,3}^{2}+r_{2,3}^{2},\,r_{1,2}r_{1,3}+r_{2,2}r_{2,3}+r_{3,2},\,r_{1,1}r_{1,3}+r_{2,1}r_{2,3}+r_{3,1},\,r_{1,2}^{2}+r_{2,2}^{2}+r_{3,2}^{2}-1,\,r_{1,1}r_{1,2}+r_{2,1}r_{2,2}+r_{3,1}r_{3,2},\,r_{1,1}^{2}-r_{2,2}^{2}-r_{2,3}^{2}-r_{3,2}^{2},\,r_{3,1}c_{1}c_{3}+r_{3,2}c_{2}c_{3}-c_{1}^{2}-c_{2}^{2})$`
`\underline{\tt o43}` : `$\texttt{Ideal}$ of $A$`
`\underline{\tt i44}` : --we investigate the last component:
      I = decPlus#2
`\underline{\tt o44}` = `$\texttt{ideal}{}(r_{3,3}c_{2}-r_{3,2}c_{3}+c_{2},\,r_{3,3}c_{1}-r_{3,1}c_{3}+c_{1},\,r_{3,2}c_{1}-r_{3,1}c_{2},\,r_{3,1}c_{1}+r_{3,2}c_{2}+r_{3,3}c_{3}-c_{3},\,r_{2,3}c_{1}+r_{1,3}c_{2}-r_{1,2}c_{3}-r_{2,1}c_{3},\,r_{1,3}c_{1}-r_{2,3}c_{2}-r_{1,1}c_{3}+r_{2,2}c_{3},\,r_{1,2}c_{1}-r_{2,1}c_{1}-r_{1,1}c_{2}-r_{2,2}c_{2}-r_{2,3}c_{3},\,r_{1,1}c_{1}+r_{2,2}c_{1}+r_{1,2}c_{2}-r_{2,1}c_{2}+r_{1,3}c_{3},\,r_{2,3}r_{3,2}-r_{2,2}r_{3,3}+r_{1,1},\,r_{1,3}r_{3,2}-r_{1,2}r_{3,3}-r_{2,1},\,r_{3,1}^{2}+r_{3,2}^{2}+r_{3,3}^{2}-1,\,r_{2,3}r_{3,1}-r_{2,1}r_{3,3}-r_{1,2},\,r_{2,2}r_{3,1}-r_{2,1}r_{3,2}+r_{1,3},\,r_{2,1}r_{3,1}+r_{2,2}r_{3,2}+r_{2,3}r_{3,3},\,r_{1,3}r_{3,1}-r_{1,1}r_{3,3}+r_{2,2},\,r_{1,2}r_{3,1}-r_{1,1}r_{3,2}-r_{2,3},\,r_{1,1}r_{3,1}+r_{1,2}r_{3,2}+r_{1,3}r_{3,3},\,r_{1,3}r_{2,2}-r_{1,2}r_{2,3}+r_{3,1},\,r_{2,1}^{2}+r_{2,2}^{2}+r_{2,3}^{2}-1,\,r_{1,3}r_{2,1}-r_{1,1}r_{2,3}-r_{3,2},\,r_{1,2}r_{2,1}-r_{1,1}r_{2,2}+r_{3,3},\,r_{1,1}r_{2,1}+r_{1,2}r_{2,2}+r_{1,3}r_{2,3},\,r_{1,3}^{2}+r_{2,3}^{2}+r_{3,3}^{2}-1,\,r_{1,2}r_{1,3}+r_{2,2}r_{2,3}+r_{3,2}r_{3,3},\,r_{1,1}r_{1,3}+r_{2,1}r_{2,3}+r_{3,1}r_{3,3},\,r_{1,2}^{2}+r_{2,2}^{2}+r_{3,2}^{2}-1,\,r_{1,1}r_{1,2}+r_{2,1}r_{2,2}+r_{3,1}r_{3,2},\,r_{1,1}^{2}-r_{2,2}^{2}-r_{2,3}^{2}-r_{3,2}^{2}-r_{3,3}^{2}+1)$`
`\underline{\tt o44}` : `$\texttt{Ideal}$ of $A$`
`\underline{\tt i45}` : L = flatten entries gens I
`\underline{\tt o45}` = `$\{r_{3,3}c_{2}-r_{3,2}c_{3}+c_{2},\:r_{3,3}c_{1}-r_{3,1}c_{3}+c_{1},\:r_{3,2}c_{1}-r_{3,1}c_{2},\:r_{3,1}c_{1}+r_{3,2}c_{2}+r_{3,3}c_{3}-c_{3},\:r_{2,3}c_{1}+r_{1,3}c_{2}-r_{1,2}c_{3}-r_{2,1}c_{3},\:r_{1,3}c_{1}-r_{2,3}c_{2}-r_{1,1}c_{3}+r_{2,2}c_{3},\:r_{1,2}c_{1}-r_{2,1}c_{1}-r_{1,1}c_{2}-r_{2,2}c_{2}-r_{2,3}c_{3},\:r_{1,1}c_{1}+r_{2,2}c_{1}+r_{1,2}c_{2}-r_{2,1}c_{2}+r_{1,3}c_{3},\:r_{2,3}r_{3,2}-r_{2,2}r_{3,3}+r_{1,1},\:r_{1,3}r_{3,2}-r_{1,2}r_{3,3}-r_{2,1},\:r_{3,1}^{2}+r_{3,2}^{2}+r_{3,3}^{2}-1,\:r_{2,3}r_{3,1}-r_{2,1}r_{3,3}-r_{1,2},\:r_{2,2}r_{3,1}-r_{2,1}r_{3,2}+r_{1,3},\:r_{2,1}r_{3,1}+r_{2,2}r_{3,2}+r_{2,3}r_{3,3},\:r_{1,3}r_{3,1}-r_{1,1}r_{3,3}+r_{2,2},\:r_{1,2}r_{3,1}-r_{1,1}r_{3,2}-r_{2,3},\:r_{1,1}r_{3,1}+r_{1,2}r_{3,2}+r_{1,3}r_{3,3},\:r_{1,3}r_{2,2}-r_{1,2}r_{2,3}+r_{3,1},\:r_{2,1}^{2}+r_{2,2}^{2}+r_{2,3}^{2}-1,\:r_{1,3}r_{2,1}-r_{1,1}r_{2,3}-r_{3,2},\:r_{1,2}r_{2,1}-r_{1,1}r_{2,2}+r_{3,3},\:r_{1,1}r_{2,1}+r_{1,2}r_{2,2}+r_{1,3}r_{2,3},\:r_{1,3}^{2}+r_{2,3}^{2}+r_{3,3}^{2}-1,\:r_{1,2}r_{1,3}+r_{2,2}r_{2,3}+r_{3,2}r_{3,3},\:r_{1,1}r_{1,3}+r_{2,1}r_{2,3}+r_{3,1}r_{3,3},\:r_{1,2}^{2}+r_{2,2}^{2}+r_{3,2}^{2}-1,\:r_{1,1}r_{1,2}+r_{2,1}r_{2,2}+r_{3,1}r_{3,2},\:r_{1,1}^{2}-r_{2,2}^{2}-r_{2,3}^{2}-r_{3,2}^{2}-r_{3,3}^{2}+1\}$`
`\underline{\tt o45}` : `$\texttt{List}$`
`\underline{\tt i46}` : --by investigating the first 2 equations, we see that t_1 and t_2 can be expressed in terms of R and t_3
      --(if the last entry of R is not equal to -1)
      --we substitute these expressions into all equations:
      Lsub = apply(L, eq -> sub(eq, {c_2 => (r_(3,2)*c_3)/(r_(3,3)+1), c_1 => (r_(3,1)*c_3)/(r_(3,3)+1)}))
`\underline{\tt o46}` = `$\{0,\\0,\\0,\\\frac{r_{3,1}^{2}c_{3}+r_{3,2}^{2}c_{3}+r_{3,3}^{2}c_{3}-c_{3}}{r_{3,3}+1},\\\frac{r_{2,3}r_{3,1}c_{3}+r_{1,3}r_{3,2}c_{3}-r_{1,2}r_{3,3}c_{3}-r_{2,1}r_{3,3}c_{3}-r_{1,2}c_{3}-r_{2,1}c_{3}}{r_{3,3}+1},\\\frac{r_{1,3}r_{3,1}c_{3}-r_{2,3}r_{3,2}c_{3}-r_{1,1}r_{3,3}c_{3}+r_{2,2}r_{3,3}c_{3}-r_{1,1}c_{3}+r_{2,2}c_{3}}{r_{3,3}+1},\\\frac{r_{1,2}r_{3,1}c_{3}-r_{2,1}r_{3,1}c_{3}-r_{1,1}r_{3,2}c_{3}-r_{2,2}r_{3,2}c_{3}-r_{2,3}r_{3,3}c_{3}-r_{2,3}c_{3}}{r_{3,3}+1},\\\frac{r_{1,1}r_{3,1}c_{3}+r_{2,2}r_{3,1}c_{3}+r_{1,2}r_{3,2}c_{3}-r_{2,1}r_{3,2}c_{3}+r_{1,3}r_{3,3}c_{3}+r_{1,3}c_{3}}{r_{3,3}+1},\\r_{2,3}r_{3,2}-r_{2,2}r_{3,3}+r_{1,1},\\r_{1,3}r_{3,2}-r_{1,2}r_{3,3}-r_{2,1},\\r_{3,1}^{2}+r_{3,2}^{2}+r_{3,3}^{2}-1,\\r_{2,3}r_{3,1}-r_{2,1}r_{3,3}-r_{1,2},\\r_{2,2}r_{3,1}-r_{2,1}r_{3,2}+r_{1,3},\\r_{2,1}r_{3,1}+r_{2,2}r_{3,2}+r_{2,3}r_{3,3},\\r_{1,3}r_{3,1}-r_{1,1}r_{3,3}+r_{2,2},\\r_{1,2}r_{3,1}-r_{1,1}r_{3,2}-r_{2,3},\\r_{1,1}r_{3,1}+r_{1,2}r_{3,2}+r_{1,3}r_{3,3},\\r_{1,3}r_{2,2}-r_{1,2}r_{2,3}+r_{3,1},\\r_{2,1}^{2}+r_{2,2}^{2}+r_{2,3}^{2}-1,\\r_{1,3}r_{2,1}-r_{1,1}r_{2,3}-r_{3,2},\\r_{1,2}r_{2,1}-r_{1,1}r_{2,2}+r_{3,3},\\r_{1,1}r_{2,1}+r_{1,2}r_{2,2}+r_{1,3}r_{2,3},\\r_{1,3}^{2}+r_{2,3}^{2}+r_{3,3}^{2}-1,\\r_{1,2}r_{1,3}+r_{2,2}r_{2,3}+r_{3,2}r_{3,3},\\r_{1,1}r_{1,3}+r_{2,1}r_{2,3}+r_{3,1}r_{3,3},\\r_{1,2}^{2}+r_{2,2}^{2}+r_{3,2}^{2}-1,\\r_{1,1}r_{1,2}+r_{2,1}r_{2,2}+r_{3,1}r_{3,2},\\r_{1,1}^{2}-r_{2,2}^{2}-r_{2,3}^{2}-r_{3,2}^{2}-r_{3,3}^{2}+1\}$`
`\underline{\tt o46}` : `$\texttt{List}$`
`\underline{\tt i47}` : --now we see that we obtained again the ideal of SO(3), meaning that there are no further constraints on R and t besides the formulas for t_1 and t_2 above 
      SO3 == ideal apply(Lsub, eq -> num eq)
`\underline{\tt o47}` = `$\texttt{true}$`
`\underline{\tt i48}` : --hence, for every R with last entry not equal -1, we find exactly one solution in t (up to scaling)
      
      --finally, we analyze the negative component of O(2)
      decMinus = decompose(O2minus+SO3)
`\underline{\tt o48}` = `$\{\texttt{ideal}{}(r_{3,3}+1,\,r_{3,2},\,r_{3,1},\,r_{2,3},\,r_{1,3},\,r_{1,2}-r_{2,1},\,r_{1,1}+r_{2,2},\,r_{2,1}^{2}+r_{2,2}^{2}-1),\:\texttt{ideal}{}(r_{3,3}+1,\,r_{2,3}c_{1}-r_{1,3}c_{2}+r_{1,2}c_{3}-r_{2,1}c_{3},\,r_{1,3}c_{1}+r_{2,3}c_{2}-r_{1,1}c_{3}-r_{2,2}c_{3},\,r_{2,3}r_{3,2}+r_{1,1}+r_{2,2},\,r_{1,3}r_{3,2}+r_{1,2}-r_{2,1},\,r_{3,1}^{2}+r_{3,2}^{2},\,r_{2,3}r_{3,1}-r_{1,2}+r_{2,1},\,r_{2,2}r_{3,1}-r_{2,1}r_{3,2}+r_{1,3},\,r_{2,1}r_{3,1}+r_{2,2}r_{3,2}-r_{2,3},\,r_{1,3}r_{3,1}+r_{1,1}+r_{2,2},\,r_{1,2}r_{3,1}-r_{1,1}r_{3,2}-r_{2,3},\,r_{1,1}r_{3,1}+r_{1,2}r_{3,2}-r_{1,3},\,r_{1,3}r_{2,2}-r_{1,2}r_{2,3}+r_{3,1},\,r_{2,1}^{2}+r_{2,2}^{2}+r_{2,3}^{2}-1,\,r_{1,3}r_{2,1}-r_{1,1}r_{2,3}-r_{3,2},\,r_{1,2}r_{2,1}-r_{1,1}r_{2,2}-1,\,r_{1,1}r_{2,1}+r_{1,2}r_{2,2}+r_{1,3}r_{2,3},\,r_{1,3}^{2}+r_{2,3}^{2},\,r_{1,2}r_{1,3}+r_{2,2}r_{2,3}-r_{3,2},\,r_{1,1}r_{1,3}+r_{2,1}r_{2,3}-r_{3,1},\,r_{1,2}^{2}+r_{2,2}^{2}+r_{3,2}^{2}-1,\,r_{1,1}r_{1,2}+r_{2,1}r_{2,2}+r_{3,1}r_{3,2},\,r_{1,1}^{2}-r_{2,2}^{2}-r_{2,3}^{2}-r_{3,2}^{2},\,r_{3,1}c_{1}c_{3}+r_{3,2}c_{2}c_{3}+c_{1}^{2}+c_{2}^{2}),\:\texttt{ideal}{}(r_{3,3}c_{2}-r_{3,2}c_{3}-c_{2},\,r_{3,3}c_{1}-r_{3,1}c_{3}-c_{1},\,r_{3,2}c_{1}-r_{3,1}c_{2},\,r_{3,1}c_{1}+r_{3,2}c_{2}+r_{3,3}c_{3}+c_{3},\,r_{2,3}c_{1}-r_{1,3}c_{2}+r_{1,2}c_{3}-r_{2,1}c_{3},\,r_{1,3}c_{1}+r_{2,3}c_{2}-r_{1,1}c_{3}-r_{2,2}c_{3},\,r_{1,2}c_{1}+r_{2,1}c_{1}-r_{1,1}c_{2}+r_{2,2}c_{2}+r_{2,3}c_{3},\,r_{1,1}c_{1}-r_{2,2}c_{1}+r_{1,2}c_{2}+r_{2,1}c_{2}+r_{1,3}c_{3},\,r_{2,3}r_{3,2}-r_{2,2}r_{3,3}+r_{1,1},\,r_{1,3}r_{3,2}-r_{1,2}r_{3,3}-r_{2,1},\,r_{3,1}^{2}+r_{3,2}^{2}+r_{3,3}^{2}-1,\,r_{2,3}r_{3,1}-r_{2,1}r_{3,3}-r_{1,2},\,r_{2,2}r_{3,1}-r_{2,1}r_{3,2}+r_{1,3},\,r_{2,1}r_{3,1}+r_{2,2}r_{3,2}+r_{2,3}r_{3,3},\,r_{1,3}r_{3,1}-r_{1,1}r_{3,3}+r_{2,2},\,r_{1,2}r_{3,1}-r_{1,1}r_{3,2}-r_{2,3},\,r_{1,1}r_{3,1}+r_{1,2}r_{3,2}+r_{1,3}r_{3,3},\,r_{1,3}r_{2,2}-r_{1,2}r_{2,3}+r_{3,1},\,r_{2,1}^{2}+r_{2,2}^{2}+r_{2,3}^{2}-1,\,r_{1,3}r_{2,1}-r_{1,1}r_{2,3}-r_{3,2},\,r_{1,2}r_{2,1}-r_{1,1}r_{2,2}+r_{3,3},\,r_{1,1}r_{2,1}+r_{1,2}r_{2,2}+r_{1,3}r_{2,3},\,r_{1,3}^{2}+r_{2,3}^{2}+r_{3,3}^{2}-1,\,r_{1,2}r_{1,3}+r_{2,2}r_{2,3}+r_{3,2}r_{3,3},\,r_{1,1}r_{1,3}+r_{2,1}r_{2,3}+r_{3,1}r_{3,3},\,r_{1,2}^{2}+r_{2,2}^{2}+r_{3,2}^{2}-1,\,r_{1,1}r_{1,2}+r_{2,1}r_{2,2}+r_{3,1}r_{3,2},\,r_{1,1}^{2}-r_{2,2}^{2}-r_{2,3}^{2}-r_{3,2}^{2}-r_{3,3}^{2}+1)\}$`
`\underline{\tt o48}` : `$\texttt{List}$`
`\underline{\tt i49}` : --this ideal has 3 components:
      #decMinus
`\underline{\tt o49}` = `$3$`
`\underline{\tt i50}` : --the first 2 have the last entry of R equal to -1:
      decMinus#0
`\underline{\tt o50}` = `$\texttt{ideal}{}(r_{3,3}+1,\,r_{3,2},\,r_{3,1},\,r_{2,3},\,r_{1,3},\,r_{1,2}-r_{2,1},\,r_{1,1}+r_{2,2},\,r_{2,1}^{2}+r_{2,2}^{2}-1)$`
`\underline{\tt o50}` : `$\texttt{Ideal}$ of $A$`
`\underline{\tt i51}` : decMinus#1
`\underline{\tt o51}` = `$\texttt{ideal}{}(r_{3,3}+1,\,r_{2,3}c_{1}-r_{1,3}c_{2}+r_{1,2}c_{3}-r_{2,1}c_{3},\,r_{1,3}c_{1}+r_{2,3}c_{2}-r_{1,1}c_{3}-r_{2,2}c_{3},\,r_{2,3}r_{3,2}+r_{1,1}+r_{2,2},\,r_{1,3}r_{3,2}+r_{1,2}-r_{2,1},\,r_{3,1}^{2}+r_{3,2}^{2},\,r_{2,3}r_{3,1}-r_{1,2}+r_{2,1},\,r_{2,2}r_{3,1}-r_{2,1}r_{3,2}+r_{1,3},\,r_{2,1}r_{3,1}+r_{2,2}r_{3,2}-r_{2,3},\,r_{1,3}r_{3,1}+r_{1,1}+r_{2,2},\,r_{1,2}r_{3,1}-r_{1,1}r_{3,2}-r_{2,3},\,r_{1,1}r_{3,1}+r_{1,2}r_{3,2}-r_{1,3},\,r_{1,3}r_{2,2}-r_{1,2}r_{2,3}+r_{3,1},\,r_{2,1}^{2}+r_{2,2}^{2}+r_{2,3}^{2}-1,\,r_{1,3}r_{2,1}-r_{1,1}r_{2,3}-r_{3,2},\,r_{1,2}r_{2,1}-r_{1,1}r_{2,2}-1,\,r_{1,1}r_{2,1}+r_{1,2}r_{2,2}+r_{1,3}r_{2,3},\,r_{1,3}^{2}+r_{2,3}^{2},\,r_{1,2}r_{1,3}+r_{2,2}r_{2,3}-r_{3,2},\,r_{1,1}r_{1,3}+r_{2,1}r_{2,3}-r_{3,1},\,r_{1,2}^{2}+r_{2,2}^{2}+r_{3,2}^{2}-1,\,r_{1,1}r_{1,2}+r_{2,1}r_{2,2}+r_{3,1}r_{3,2},\,r_{1,1}^{2}-r_{2,2}^{2}-r_{2,3}^{2}-r_{3,2}^{2},\,r_{3,1}c_{1}c_{3}+r_{3,2}c_{2}c_{3}+c_{1}^{2}+c_{2}^{2})$`
`\underline{\tt o51}` : `$\texttt{Ideal}$ of $A$`
`\underline{\tt i52}` : --we investigate the last component:
      I = decMinus#2
`\underline{\tt o52}` = `$\texttt{ideal}{}(r_{3,3}c_{2}-r_{3,2}c_{3}-c_{2},\,r_{3,3}c_{1}-r_{3,1}c_{3}-c_{1},\,r_{3,2}c_{1}-r_{3,1}c_{2},\,r_{3,1}c_{1}+r_{3,2}c_{2}+r_{3,3}c_{3}+c_{3},\,r_{2,3}c_{1}-r_{1,3}c_{2}+r_{1,2}c_{3}-r_{2,1}c_{3},\,r_{1,3}c_{1}+r_{2,3}c_{2}-r_{1,1}c_{3}-r_{2,2}c_{3},\,r_{1,2}c_{1}+r_{2,1}c_{1}-r_{1,1}c_{2}+r_{2,2}c_{2}+r_{2,3}c_{3},\,r_{1,1}c_{1}-r_{2,2}c_{1}+r_{1,2}c_{2}+r_{2,1}c_{2}+r_{1,3}c_{3},\,r_{2,3}r_{3,2}-r_{2,2}r_{3,3}+r_{1,1},\,r_{1,3}r_{3,2}-r_{1,2}r_{3,3}-r_{2,1},\,r_{3,1}^{2}+r_{3,2}^{2}+r_{3,3}^{2}-1,\,r_{2,3}r_{3,1}-r_{2,1}r_{3,3}-r_{1,2},\,r_{2,2}r_{3,1}-r_{2,1}r_{3,2}+r_{1,3},\,r_{2,1}r_{3,1}+r_{2,2}r_{3,2}+r_{2,3}r_{3,3},\,r_{1,3}r_{3,1}-r_{1,1}r_{3,3}+r_{2,2},\,r_{1,2}r_{3,1}-r_{1,1}r_{3,2}-r_{2,3},\,r_{1,1}r_{3,1}+r_{1,2}r_{3,2}+r_{1,3}r_{3,3},\,r_{1,3}r_{2,2}-r_{1,2}r_{2,3}+r_{3,1},\,r_{2,1}^{2}+r_{2,2}^{2}+r_{2,3}^{2}-1,\,r_{1,3}r_{2,1}-r_{1,1}r_{2,3}-r_{3,2},\,r_{1,2}r_{2,1}-r_{1,1}r_{2,2}+r_{3,3},\,r_{1,1}r_{2,1}+r_{1,2}r_{2,2}+r_{1,3}r_{2,3},\,r_{1,3}^{2}+r_{2,3}^{2}+r_{3,3}^{2}-1,\,r_{1,2}r_{1,3}+r_{2,2}r_{2,3}+r_{3,2}r_{3,3},\,r_{1,1}r_{1,3}+r_{2,1}r_{2,3}+r_{3,1}r_{3,3},\,r_{1,2}^{2}+r_{2,2}^{2}+r_{3,2}^{2}-1,\,r_{1,1}r_{1,2}+r_{2,1}r_{2,2}+r_{3,1}r_{3,2},\,r_{1,1}^{2}-r_{2,2}^{2}-r_{2,3}^{2}-r_{3,2}^{2}-r_{3,3}^{2}+1)$`
`\underline{\tt o52}` : `$\texttt{Ideal}$ of $A$`
`\underline{\tt i53}` : L = flatten entries gens I
`\underline{\tt o53}` = `$\{r_{3,3}c_{2}-r_{3,2}c_{3}-c_{2},\:r_{3,3}c_{1}-r_{3,1}c_{3}-c_{1},\:r_{3,2}c_{1}-r_{3,1}c_{2},\:r_{3,1}c_{1}+r_{3,2}c_{2}+r_{3,3}c_{3}+c_{3},\:r_{2,3}c_{1}-r_{1,3}c_{2}+r_{1,2}c_{3}-r_{2,1}c_{3},\:r_{1,3}c_{1}+r_{2,3}c_{2}-r_{1,1}c_{3}-r_{2,2}c_{3},\:r_{1,2}c_{1}+r_{2,1}c_{1}-r_{1,1}c_{2}+r_{2,2}c_{2}+r_{2,3}c_{3},\:r_{1,1}c_{1}-r_{2,2}c_{1}+r_{1,2}c_{2}+r_{2,1}c_{2}+r_{1,3}c_{3},\:r_{2,3}r_{3,2}-r_{2,2}r_{3,3}+r_{1,1},\:r_{1,3}r_{3,2}-r_{1,2}r_{3,3}-r_{2,1},\:r_{3,1}^{2}+r_{3,2}^{2}+r_{3,3}^{2}-1,\:r_{2,3}r_{3,1}-r_{2,1}r_{3,3}-r_{1,2},\:r_{2,2}r_{3,1}-r_{2,1}r_{3,2}+r_{1,3},\:r_{2,1}r_{3,1}+r_{2,2}r_{3,2}+r_{2,3}r_{3,3},\:r_{1,3}r_{3,1}-r_{1,1}r_{3,3}+r_{2,2},\:r_{1,2}r_{3,1}-r_{1,1}r_{3,2}-r_{2,3},\:r_{1,1}r_{3,1}+r_{1,2}r_{3,2}+r_{1,3}r_{3,3},\:r_{1,3}r_{2,2}-r_{1,2}r_{2,3}+r_{3,1},\:r_{2,1}^{2}+r_{2,2}^{2}+r_{2,3}^{2}-1,\:r_{1,3}r_{2,1}-r_{1,1}r_{2,3}-r_{3,2},\:r_{1,2}r_{2,1}-r_{1,1}r_{2,2}+r_{3,3},\:r_{1,1}r_{2,1}+r_{1,2}r_{2,2}+r_{1,3}r_{2,3},\:r_{1,3}^{2}+r_{2,3}^{2}+r_{3,3}^{2}-1,\:r_{1,2}r_{1,3}+r_{2,2}r_{2,3}+r_{3,2}r_{3,3},\:r_{1,1}r_{1,3}+r_{2,1}r_{2,3}+r_{3,1}r_{3,3},\:r_{1,2}^{2}+r_{2,2}^{2}+r_{3,2}^{2}-1,\:r_{1,1}r_{1,2}+r_{2,1}r_{2,2}+r_{3,1}r_{3,2},\:r_{1,1}^{2}-r_{2,2}^{2}-r_{2,3}^{2}-r_{3,2}^{2}-r_{3,3}^{2}+1\}$`
`\underline{\tt o53}` : `$\texttt{List}$`
`\underline{\tt i54}` : --by investigating the first 2 equations, we see that t_1 and t_2 can be expressed in terms of R and t_3
      --(if the last entry of R is not equal to 1)
      --we substitute these expressions into all equations:
      Lsub = apply(L, eq -> sub(eq, {c_2 => (r_(3,2)*c_3)/(r_(3,3)-1), c_1 => (r_(3,1)*c_3)/(r_(3,3)-1)}))
`\underline{\tt o54}` = `$\{0,\\0,\\0,\\\frac{r_{3,1}^{2}c_{3}+r_{3,2}^{2}c_{3}+r_{3,3}^{2}c_{3}-c_{3}}{r_{3,3}-1},\\\frac{r_{2,3}r_{3,1}c_{3}-r_{1,3}r_{3,2}c_{3}+r_{1,2}r_{3,3}c_{3}-r_{2,1}r_{3,3}c_{3}-r_{1,2}c_{3}+r_{2,1}c_{3}}{r_{3,3}-1},\\\frac{r_{1,3}r_{3,1}c_{3}+r_{2,3}r_{3,2}c_{3}-r_{1,1}r_{3,3}c_{3}-r_{2,2}r_{3,3}c_{3}+r_{1,1}c_{3}+r_{2,2}c_{3}}{r_{3,3}-1},\\\frac{r_{1,2}r_{3,1}c_{3}+r_{2,1}r_{3,1}c_{3}-r_{1,1}r_{3,2}c_{3}+r_{2,2}r_{3,2}c_{3}+r_{2,3}r_{3,3}c_{3}-r_{2,3}c_{3}}{r_{3,3}-1},\\\frac{r_{1,1}r_{3,1}c_{3}-r_{2,2}r_{3,1}c_{3}+r_{1,2}r_{3,2}c_{3}+r_{2,1}r_{3,2}c_{3}+r_{1,3}r_{3,3}c_{3}-r_{1,3}c_{3}}{r_{3,3}-1},\\r_{2,3}r_{3,2}-r_{2,2}r_{3,3}+r_{1,1},\\r_{1,3}r_{3,2}-r_{1,2}r_{3,3}-r_{2,1},\\r_{3,1}^{2}+r_{3,2}^{2}+r_{3,3}^{2}-1,\\r_{2,3}r_{3,1}-r_{2,1}r_{3,3}-r_{1,2},\\r_{2,2}r_{3,1}-r_{2,1}r_{3,2}+r_{1,3},\\r_{2,1}r_{3,1}+r_{2,2}r_{3,2}+r_{2,3}r_{3,3},\\r_{1,3}r_{3,1}-r_{1,1}r_{3,3}+r_{2,2},\\r_{1,2}r_{3,1}-r_{1,1}r_{3,2}-r_{2,3},\\r_{1,1}r_{3,1}+r_{1,2}r_{3,2}+r_{1,3}r_{3,3},\\r_{1,3}r_{2,2}-r_{1,2}r_{2,3}+r_{3,1},\\r_{2,1}^{2}+r_{2,2}^{2}+r_{2,3}^{2}-1,\\r_{1,3}r_{2,1}-r_{1,1}r_{2,3}-r_{3,2},\\r_{1,2}r_{2,1}-r_{1,1}r_{2,2}+r_{3,3},\\r_{1,1}r_{2,1}+r_{1,2}r_{2,2}+r_{1,3}r_{2,3},\\r_{1,3}^{2}+r_{2,3}^{2}+r_{3,3}^{2}-1,\\r_{1,2}r_{1,3}+r_{2,2}r_{2,3}+r_{3,2}r_{3,3},\\r_{1,1}r_{1,3}+r_{2,1}r_{2,3}+r_{3,1}r_{3,3},\\r_{1,2}^{2}+r_{2,2}^{2}+r_{3,2}^{2}-1,\\r_{1,1}r_{1,2}+r_{2,1}r_{2,2}+r_{3,1}r_{3,2},\\r_{1,1}^{2}-r_{2,2}^{2}-r_{2,3}^{2}-r_{3,2}^{2}-r_{3,3}^{2}+1\}$`
`\underline{\tt o54}` : `$\texttt{List}$`
`\underline{\tt i55}` : --now we see that we obtained again the ideal of SO(3), meaning that there are no further constraints on R and t besides the formulas for t_1 and t_2 above 
      SO3 == ideal apply(Lsub, eq -> num eq)
`\underline{\tt o55}` = `$\texttt{true}$`
`\underline{\tt i56}` : --hence, for every R with last entry not equal 1, we find exactly one solution in t (up to scaling)
\end{lstlisting}
}

\newpage
\subsection{\texorpdfstring{$r_6(s)$ from Theorem \ref{thm:EDD2}}{r6(s) from Theorem 4.1}}

\begin{lstlisting}[language=Macaulay2output]
`\underline{\tt i1}` : clearAll
`\underline{\tt i2}` : R = QQ[a_1..a_2,l_1..l_4,y_1..y_4][s]
`\underline{\tt o2}` = `$R$`
`\underline{\tt o2}` : `$\texttt{PolynomialRing}$`
`\underline{\tt i3}` : q_1 = a_1
`\underline{\tt o3}` = `$a_{1}$`
`\underline{\tt o3}` : `${\mathbb Q}\mathopen{}[a_{1}\,{.}{.}\,a_{2},\,l_{1}\,{.}{.}\,l_{4},\,y_{1}\,{.}{.}\,y_{4}]$`
`\underline{\tt i4}` : q_2 = -a_1
`\underline{\tt o4}` = `$-a_{1}$`
`\underline{\tt o4}` : `${\mathbb Q}\mathopen{}[a_{1}\,{.}{.}\,a_{2},\,l_{1}\,{.}{.}\,l_{4},\,y_{1}\,{.}{.}\,y_{4}]$`
`\underline{\tt i5}` : q_3 = a_2
`\underline{\tt o5}` = `$a_{2}$`
`\underline{\tt o5}` : `${\mathbb Q}\mathopen{}[a_{1}\,{.}{.}\,a_{2},\,l_{1}\,{.}{.}\,l_{4},\,y_{1}\,{.}{.}\,y_{4}]$`
`\underline{\tt i6}` : q_4 = -a_2
`\underline{\tt o6}` = `$-a_{2}$`
`\underline{\tt o6}` : `${\mathbb Q}\mathopen{}[a_{1}\,{.}{.}\,a_{2},\,l_{1}\,{.}{.}\,l_{4},\,y_{1}\,{.}{.}\,y_{4}]$`
`\underline{\tt i7}` : scan(4, i -> e_(i+1) = s*q_(i+1)*y_(i+1)/( l_(i+1)-s*q_(i+1) )) 
`\underline{\tt i8}` : Rfunction = sum apply(4, i -> q_(i+1)*(y_(i+1) + e_(i+1))^2)
`\underline{\tt o8}` = `[very long expression]`
`\underline{\tt o8}` : `$\texttt{frac}\ R$`
`\underline{\tt i9}` : r6 = numerator Rfunction 
`\underline{\tt o9}` = `$(a_{1}^{3}a_{2}^{4}l_{1}^{2}y_{1}^{2}-a_{1}^{3}a_{2}^{4}l_{2}^{2}y_{2}^{2}+a_{1}^{4}a_{2}^{3}l_{3}^{2}y_{3}^{2}-a_{1}^{4}a_{2}^{3}l_{4}^{2}y_{4}^{2})s^{6}\\+(2\,a_{1}^{2}a_{2}^{4}l_{1}^{2}l_{2}y_{1}^{2}-2\,a_{1}^{3}a_{2}^{3}l_{1}^{2}l_{3}y_{1}^{2}+2\,a_{1}^{3}a_{2}^{3}l_{1}^{2}l_{4}y_{1}^{2}+2\,a_{1}^{2}a_{2}^{4}l_{1}l_{2}^{2}y_{2}^{2}+2\,a_{1}^{3}a_{2}^{3}l_{2}^{2}l_{3}y_{2}^{2}-2\,a_{1}^{3}a_{2}^{3}l_{2}^{2}l_{4}y_{2}^{2}-2\,a_{1}^{3}a_{2}^{3}l_{1}l_{3}^{2}y_{3}^{2}+2\,a_{1}^{3}a_{2}^{3}l_{2}l_{3}^{2}y_{3}^{2}+2\,a_{1}^{4}a_{2}^{2}l_{3}^{2}l_{4}y_{3}^{2}+2\,a_{1}^{3}a_{2}^{3}l_{1}l_{4}^{2}y_{4}^{2}-2\,a_{1}^{3}a_{2}^{3}l_{2}l_{4}^{2}y_{4}^{2}+2\,a_{1}^{4}a_{2}^{2}l_{3}l_{4}^{2}y_{4}^{2})s^{5}\\+(a_{1}a_{2}^{4}l_{1}^{2}l_{2}^{2}y_{1}^{2}-4\,a_{1}^{2}a_{2}^{3}l_{1}^{2}l_{2}l_{3}y_{1}^{2}+a_{1}^{3}a_{2}^{2}l_{1}^{2}l_{3}^{2}y_{1}^{2}+4\,a_{1}^{2}a_{2}^{3}l_{1}^{2}l_{2}l_{4}y_{1}^{2}-4\,a_{1}^{3}a_{2}^{2}l_{1}^{2}l_{3}l_{4}y_{1}^{2}+a_{1}^{3}a_{2}^{2}l_{1}^{2}l_{4}^{2}y_{1}^{2}-a_{1}a_{2}^{4}l_{1}^{2}l_{2}^{2}y_{2}^{2}-4\,a_{1}^{2}a_{2}^{3}l_{1}l_{2}^{2}l_{3}y_{2}^{2}-a_{1}^{3}a_{2}^{2}l_{2}^{2}l_{3}^{2}y_{2}^{2}+4\,a_{1}^{2}a_{2}^{3}l_{1}l_{2}^{2}l_{4}y_{2}^{2}+4\,a_{1}^{3}a_{2}^{2}l_{2}^{2}l_{3}l_{4}y_{2}^{2}-a_{1}^{3}a_{2}^{2}l_{2}^{2}l_{4}^{2}y_{2}^{2}+a_{1}^{2}a_{2}^{3}l_{1}^{2}l_{3}^{2}y_{3}^{2}-4\,a_{1}^{2}a_{2}^{3}l_{1}l_{2}l_{3}^{2}y_{3}^{2}+a_{1}^{2}a_{2}^{3}l_{2}^{2}l_{3}^{2}y_{3}^{2}-4\,a_{1}^{3}a_{2}^{2}l_{1}l_{3}^{2}l_{4}y_{3}^{2}+4\,a_{1}^{3}a_{2}^{2}l_{2}l_{3}^{2}l_{4}y_{3}^{2}+a_{1}^{4}a_{2}l_{3}^{2}l_{4}^{2}y_{3}^{2}-a_{1}^{2}a_{2}^{3}l_{1}^{2}l_{4}^{2}y_{4}^{2}+4\,a_{1}^{2}a_{2}^{3}l_{1}l_{2}l_{4}^{2}y_{4}^{2}-a_{1}^{2}a_{2}^{3}l_{2}^{2}l_{4}^{2}y_{4}^{2}-4\,a_{1}^{3}a_{2}^{2}l_{1}l_{3}l_{4}^{2}y_{4}^{2}+4\,a_{1}^{3}a_{2}^{2}l_{2}l_{3}l_{4}^{2}y_{4}^{2}-a_{1}^{4}a_{2}l_{3}^{2}l_{4}^{2}y_{4}^{2})s^{4}\\+(-2\,a_{1}a_{2}^{3}l_{1}^{2}l_{2}^{2}l_{3}y_{1}^{2}+2\,a_{1}^{2}a_{2}^{2}l_{1}^{2}l_{2}l_{3}^{2}y_{1}^{2}+2\,a_{1}a_{2}^{3}l_{1}^{2}l_{2}^{2}l_{4}y_{1}^{2}-8\,a_{1}^{2}a_{2}^{2}l_{1}^{2}l_{2}l_{3}l_{4}y_{1}^{2}+2\,a_{1}^{3}a_{2}l_{1}^{2}l_{3}^{2}l_{4}y_{1}^{2}+2\,a_{1}^{2}a_{2}^{2}l_{1}^{2}l_{2}l_{4}^{2}y_{1}^{2}-2\,a_{1}^{3}a_{2}l_{1}^{2}l_{3}l_{4}^{2}y_{1}^{2}+2\,a_{1}a_{2}^{3}l_{1}^{2}l_{2}^{2}l_{3}y_{2}^{2}+2\,a_{1}^{2}a_{2}^{2}l_{1}l_{2}^{2}l_{3}^{2}y_{2}^{2}-2\,a_{1}a_{2}^{3}l_{1}^{2}l_{2}^{2}l_{4}y_{2}^{2}-8\,a_{1}^{2}a_{2}^{2}l_{1}l_{2}^{2}l_{3}l_{4}y_{2}^{2}-2\,a_{1}^{3}a_{2}l_{2}^{2}l_{3}^{2}l_{4}y_{2}^{2}+2\,a_{1}^{2}a_{2}^{2}l_{1}l_{2}^{2}l_{4}^{2}y_{2}^{2}+2\,a_{1}^{3}a_{2}l_{2}^{2}l_{3}l_{4}^{2}y_{2}^{2}+2\,a_{1}a_{2}^{3}l_{1}^{2}l_{2}l_{3}^{2}y_{3}^{2}-2\,a_{1}a_{2}^{3}l_{1}l_{2}^{2}l_{3}^{2}y_{3}^{2}+2\,a_{1}^{2}a_{2}^{2}l_{1}^{2}l_{3}^{2}l_{4}y_{3}^{2}-8\,a_{1}^{2}a_{2}^{2}l_{1}l_{2}l_{3}^{2}l_{4}y_{3}^{2}+2\,a_{1}^{2}a_{2}^{2}l_{2}^{2}l_{3}^{2}l_{4}y_{3}^{2}-2\,a_{1}^{3}a_{2}l_{1}l_{3}^{2}l_{4}^{2}y_{3}^{2}+2\,a_{1}^{3}a_{2}l_{2}l_{3}^{2}l_{4}^{2}y_{3}^{2}-2\,a_{1}a_{2}^{3}l_{1}^{2}l_{2}l_{4}^{2}y_{4}^{2}+2\,a_{1}a_{2}^{3}l_{1}l_{2}^{2}l_{4}^{2}y_{4}^{2}+2\,a_{1}^{2}a_{2}^{2}l_{1}^{2}l_{3}l_{4}^{2}y_{4}^{2}-8\,a_{1}^{2}a_{2}^{2}l_{1}l_{2}l_{3}l_{4}^{2}y_{4}^{2}+2\,a_{1}^{2}a_{2}^{2}l_{2}^{2}l_{3}l_{4}^{2}y_{4}^{2}+2\,a_{1}^{3}a_{2}l_{1}l_{3}^{2}l_{4}^{2}y_{4}^{2}-2\,a_{1}^{3}a_{2}l_{2}l_{3}^{2}l_{4}^{2}y_{4}^{2})s^{3}\\+(a_{1}a_{2}^{2}l_{1}^{2}l_{2}^{2}l_{3}^{2}y_{1}^{2}-4\,a_{1}a_{2}^{2}l_{1}^{2}l_{2}^{2}l_{3}l_{4}y_{1}^{2}+4\,a_{1}^{2}a_{2}l_{1}^{2}l_{2}l_{3}^{2}l_{4}y_{1}^{2}+a_{1}a_{2}^{2}l_{1}^{2}l_{2}^{2}l_{4}^{2}y_{1}^{2}-4\,a_{1}^{2}a_{2}l_{1}^{2}l_{2}l_{3}l_{4}^{2}y_{1}^{2}+a_{1}^{3}l_{1}^{2}l_{3}^{2}l_{4}^{2}y_{1}^{2}-a_{1}a_{2}^{2}l_{1}^{2}l_{2}^{2}l_{3}^{2}y_{2}^{2}+4\,a_{1}a_{2}^{2}l_{1}^{2}l_{2}^{2}l_{3}l_{4}y_{2}^{2}+4\,a_{1}^{2}a_{2}l_{1}l_{2}^{2}l_{3}^{2}l_{4}y_{2}^{2}-a_{1}a_{2}^{2}l_{1}^{2}l_{2}^{2}l_{4}^{2}y_{2}^{2}-4\,a_{1}^{2}a_{2}l_{1}l_{2}^{2}l_{3}l_{4}^{2}y_{2}^{2}-a_{1}^{3}l_{2}^{2}l_{3}^{2}l_{4}^{2}y_{2}^{2}+a_{2}^{3}l_{1}^{2}l_{2}^{2}l_{3}^{2}y_{3}^{2}+4\,a_{1}a_{2}^{2}l_{1}^{2}l_{2}l_{3}^{2}l_{4}y_{3}^{2}-4\,a_{1}a_{2}^{2}l_{1}l_{2}^{2}l_{3}^{2}l_{4}y_{3}^{2}+a_{1}^{2}a_{2}l_{1}^{2}l_{3}^{2}l_{4}^{2}y_{3}^{2}-4\,a_{1}^{2}a_{2}l_{1}l_{2}l_{3}^{2}l_{4}^{2}y_{3}^{2}+a_{1}^{2}a_{2}l_{2}^{2}l_{3}^{2}l_{4}^{2}y_{3}^{2}-a_{2}^{3}l_{1}^{2}l_{2}^{2}l_{4}^{2}y_{4}^{2}+4\,a_{1}a_{2}^{2}l_{1}^{2}l_{2}l_{3}l_{4}^{2}y_{4}^{2}-4\,a_{1}a_{2}^{2}l_{1}l_{2}^{2}l_{3}l_{4}^{2}y_{4}^{2}-a_{1}^{2}a_{2}l_{1}^{2}l_{3}^{2}l_{4}^{2}y_{4}^{2}+4\,a_{1}^{2}a_{2}l_{1}l_{2}l_{3}^{2}l_{4}^{2}y_{4}^{2}-a_{1}^{2}a_{2}l_{2}^{2}l_{3}^{2}l_{4}^{2}y_{4}^{2})s^{2}\\+(2\,a_{1}a_{2}l_{1}^{2}l_{2}^{2}l_{3}^{2}l_{4}y_{1}^{2}-2\,a_{1}a_{2}l_{1}^{2}l_{2}^{2}l_{3}l_{4}^{2}y_{1}^{2}+2\,a_{1}^{2}l_{1}^{2}l_{2}l_{3}^{2}l_{4}^{2}y_{1}^{2}-2\,a_{1}a_{2}l_{1}^{2}l_{2}^{2}l_{3}^{2}l_{4}y_{2}^{2}+2\,a_{1}a_{2}l_{1}^{2}l_{2}^{2}l_{3}l_{4}^{2}y_{2}^{2}+2\,a_{1}^{2}l_{1}l_{2}^{2}l_{3}^{2}l_{4}^{2}y_{2}^{2}+2\,a_{2}^{2}l_{1}^{2}l_{2}^{2}l_{3}^{2}l_{4}y_{3}^{2}+2\,a_{1}a_{2}l_{1}^{2}l_{2}l_{3}^{2}l_{4}^{2}y_{3}^{2}-2\,a_{1}a_{2}l_{1}l_{2}^{2}l_{3}^{2}l_{4}^{2}y_{3}^{2}+2\,a_{2}^{2}l_{1}^{2}l_{2}^{2}l_{3}l_{4}^{2}y_{4}^{2}-2\,a_{1}a_{2}l_{1}^{2}l_{2}l_{3}^{2}l_{4}^{2}y_{4}^{2}+2\,a_{1}a_{2}l_{1}l_{2}^{2}l_{3}^{2}l_{4}^{2}y_{4}^{2})s\\+a_{1}l_{1}^{2}l_{2}^{2}l_{3}^{2}l_{4}^{2}y_{1}^{2}-a_{1}l_{1}^{2}l_{2}^{2}l_{3}^{2}l_{4}^{2}y_{2}^{2}+a_{2}l_{1}^{2}l_{2}^{2}l_{3}^{2}l_{4}^{2}y_{3}^{2}-a_{2}l_{1}^{2}l_{2}^{2}l_{3}^{2}l_{4}^{2}y_{4}^{2}$`
`\underline{\tt o9}` : `$R$`
\end{lstlisting}

\newpage
\subsection{\texorpdfstring{$r_4(s)$ from Theorem \ref{thm:EDD2}}{r4(s) from Theorem 4.1}}
\begin{lstlisting}[language=Macaulay2output]
`\underline{\tt i10}` : clearAll
`\underline{\tt i11}` : R = QQ[a_1..a_2,l_2,l_4,y_1..y_4,m][s]
`\underline{\tt o11}` = `$R$`
`\underline{\tt o11}` : `$\texttt{PolynomialRing}$`
`\underline{\tt i12}` : q_1 = a_1
`\underline{\tt o12}` = `$a_{1}$`
`\underline{\tt o12}` : `${\mathbb Q}\mathopen{}[a_{1}\,{.}{.}\,a_{2},\,l_{2},\,l_{4},\,y_{1}\,{.}{.}\,y_{4},\,m]$`
`\underline{\tt i13}` : q_2 = -a_1
`\underline{\tt o13}` = `$-a_{1}$`
`\underline{\tt o13}` : `${\mathbb Q}\mathopen{}[a_{1}\,{.}{.}\,a_{2},\,l_{2},\,l_{4},\,y_{1}\,{.}{.}\,y_{4},\,m]$`
`\underline{\tt i14}` : q_3 = a_2
`\underline{\tt o14}` = `$a_{2}$`
`\underline{\tt o14}` : `${\mathbb Q}\mathopen{}[a_{1}\,{.}{.}\,a_{2},\,l_{2},\,l_{4},\,y_{1}\,{.}{.}\,y_{4},\,m]$`
`\underline{\tt i15}` : q_4 = -a_2
`\underline{\tt o15}` = `$-a_{2}$`
`\underline{\tt o15}` : `${\mathbb Q}\mathopen{}[a_{1}\,{.}{.}\,a_{2},\,l_{2},\,l_{4},\,y_{1}\,{.}{.}\,y_{4},\,m]$`
`\underline{\tt i16}` : l_1 = m*a_1
`\underline{\tt o16}` = `$a_{1}m$`
`\underline{\tt o16}` : `${\mathbb Q}\mathopen{}[a_{1}\,{.}{.}\,a_{2},\,l_{2},\,l_{4},\,y_{1}\,{.}{.}\,y_{4},\,m]$`
`\underline{\tt i17}` : l_3 = m*a_2
`\underline{\tt o17}` = `$a_{2}m$`
`\underline{\tt o17}` : `${\mathbb Q}\mathopen{}[a_{1}\,{.}{.}\,a_{2},\,l_{2},\,l_{4},\,y_{1}\,{.}{.}\,y_{4},\,m]$`
`\underline{\tt i18}` : scan(4, i -> e_(i+1) = s*q_(i+1)*y_(i+1)/( l_(i+1)-s*q_(i+1) )) 
`\underline{\tt i19}` : Rfunction = sum apply(4, i -> q_(i+1)*(y_(i+1) + e_(i+1))^2)
`\underline{\tt o19}` = `[very long expression]`
`\underline{\tt o19}` : `$\texttt{frac}\ R$`
`\underline{\tt i20}` : r4 = numerator Rfunction
`\underline{\tt o20}` = `$(a_{1}^{3}a_{2}^{2}y_{1}^{2}m^{2}+a_{1}^{2}a_{2}^{3}y_{3}^{2}m^{2}-a_{1}a_{2}^{2}l_{2}^{2}y_{2}^{2}-a_{1}^{2}a_{2}l_{4}^{2}y_{4}^{2})s^{4}\\+(2\,a_{1}^{2}a_{2}^{2}l_{2}y_{1}^{2}m^{2}+2\,a_{1}^{3}a_{2}l_{4}y_{1}^{2}m^{2}+2\,a_{1}a_{2}^{3}l_{2}y_{3}^{2}m^{2}+2\,a_{1}^{2}a_{2}^{2}l_{4}y_{3}^{2}m^{2}+2\,a_{1}a_{2}^{2}l_{2}^{2}y_{2}^{2}m+2\,a_{1}^{2}a_{2}l_{4}^{2}y_{4}^{2}m-2\,a_{1}a_{2}l_{2}^{2}l_{4}y_{2}^{2}-2\,a_{1}a_{2}l_{2}l_{4}^{2}y_{4}^{2})s^{3}\\+(a_{1}a_{2}^{2}l_{2}^{2}y_{1}^{2}m^{2}+4\,a_{1}^{2}a_{2}l_{2}l_{4}y_{1}^{2}m^{2}+a_{1}^{3}l_{4}^{2}y_{1}^{2}m^{2}-a_{1}a_{2}^{2}l_{2}^{2}y_{2}^{2}m^{2}+a_{2}^{3}l_{2}^{2}y_{3}^{2}m^{2}+4\,a_{1}a_{2}^{2}l_{2}l_{4}y_{3}^{2}m^{2}+a_{1}^{2}a_{2}l_{4}^{2}y_{3}^{2}m^{2}-a_{1}^{2}a_{2}l_{4}^{2}y_{4}^{2}m^{2}+4\,a_{1}a_{2}l_{2}^{2}l_{4}y_{2}^{2}m+4\,a_{1}a_{2}l_{2}l_{4}^{2}y_{4}^{2}m-a_{1}l_{2}^{2}l_{4}^{2}y_{2}^{2}-a_{2}l_{2}^{2}l_{4}^{2}y_{4}^{2})s^{2}\\+(2\,a_{1}a_{2}l_{2}^{2}l_{4}y_{1}^{2}m^{2}+2\,a_{1}^{2}l_{2}l_{4}^{2}y_{1}^{2}m^{2}-2\,a_{1}a_{2}l_{2}^{2}l_{4}y_{2}^{2}m^{2}+2\,a_{2}^{2}l_{2}^{2}l_{4}y_{3}^{2}m^{2}+2\,a_{1}a_{2}l_{2}l_{4}^{2}y_{3}^{2}m^{2}-2\,a_{1}a_{2}l_{2}l_{4}^{2}y_{4}^{2}m^{2}+2\,a_{1}l_{2}^{2}l_{4}^{2}y_{2}^{2}m+2\,a_{2}l_{2}^{2}l_{4}^{2}y_{4}^{2}m)s\\+a_{1}l_{2}^{2}l_{4}^{2}y_{1}^{2}m^{2}-a_{1}l_{2}^{2}l_{4}^{2}y_{2}^{2}m^{2}+a_{2}l_{2}^{2}l_{4}^{2}y_{3}^{2}m^{2}-a_{2}l_{2}^{2}l_{4}^{2}y_{4}^{2}m^{2}$`
`\underline{\tt o20}` : `$R$`
\end{lstlisting}

\newpage
\subsection{\texorpdfstring{$r_2(s)$ from Theorem \ref{thm:EDD2}}{r2(s) from Theorem 4.1}}
\begin{lstlisting}[language=Macaulay2output]
`\underline{\tt i21}` : clearAll
`\underline{\tt i22}` : R = QQ[a_1..a_2,y_1..y_4,m,n][s]
`\underline{\tt o22}` = `$R$`
`\underline{\tt o22}` : `$\texttt{PolynomialRing}$`
`\underline{\tt i23}` : q_1 = a_1
`\underline{\tt o23}` = `$a_{1}$`
`\underline{\tt o23}` : `${\mathbb Q}\mathopen{}[a_{1}\,{.}{.}\,a_{2},\,y_{1}\,{.}{.}\,y_{4},\,m\,{.}{.}\,n]$`
`\underline{\tt i24}` : q_2 = -a_1
`\underline{\tt o24}` = `$-a_{1}$`
`\underline{\tt o24}` : `${\mathbb Q}\mathopen{}[a_{1}\,{.}{.}\,a_{2},\,y_{1}\,{.}{.}\,y_{4},\,m\,{.}{.}\,n]$`
`\underline{\tt i25}` : q_3 = a_2
`\underline{\tt o25}` = `$a_{2}$`
`\underline{\tt o25}` : `${\mathbb Q}\mathopen{}[a_{1}\,{.}{.}\,a_{2},\,y_{1}\,{.}{.}\,y_{4},\,m\,{.}{.}\,n]$`
`\underline{\tt i26}` : q_4 = -a_2
`\underline{\tt o26}` = `$-a_{2}$`
`\underline{\tt o26}` : `${\mathbb Q}\mathopen{}[a_{1}\,{.}{.}\,a_{2},\,y_{1}\,{.}{.}\,y_{4},\,m\,{.}{.}\,n]$`
`\underline{\tt i27}` : l_1 = m*a_1
`\underline{\tt o27}` = `$a_{1}m$`
`\underline{\tt o27}` : `${\mathbb Q}\mathopen{}[a_{1}\,{.}{.}\,a_{2},\,y_{1}\,{.}{.}\,y_{4},\,m\,{.}{.}\,n]$`
`\underline{\tt i28}` : l_2 = n*a_1
`\underline{\tt o28}` = `$a_{1}n$`
`\underline{\tt o28}` : `${\mathbb Q}\mathopen{}[a_{1}\,{.}{.}\,a_{2},\,y_{1}\,{.}{.}\,y_{4},\,m\,{.}{.}\,n]$`
`\underline{\tt i29}` : l_3 = m*a_2
`\underline{\tt o29}` = `$a_{2}m$`
`\underline{\tt o29}` : `${\mathbb Q}\mathopen{}[a_{1}\,{.}{.}\,a_{2},\,y_{1}\,{.}{.}\,y_{4},\,m\,{.}{.}\,n]$`
`\underline{\tt i30}` : l_4 = n*a_2
`\underline{\tt o30}` = `$a_{2}n$`
`\underline{\tt o30}` : `${\mathbb Q}\mathopen{}[a_{1}\,{.}{.}\,a_{2},\,y_{1}\,{.}{.}\,y_{4},\,m\,{.}{.}\,n]$`
`\underline{\tt i31}` : scan(4, i -> e_(i+1) = s*q_(i+1)*y_(i+1)/( l_(i+1)-s*q_(i+1) )) 
`\underline{\tt i32}` : Rfunction = sum apply(4, i -> q_(i+1)*(y_(i+1) + e_(i+1))^2)
`\underline{\tt o32}` = `[very long expression]`
`\underline{\tt o32}` : `$\texttt{frac}\ R$`
`\underline{\tt i33}` : r2 = numerator Rfunction -- coeff's A, B, C in the paper
`\underline{\tt o33}` = `$(a_{1}y_{1}^{2}m^{2}+a_{2}y_{3}^{2}m^{2}-a_{1}y_{2}^{2}n^{2}-a_{2}y_{4}^{2}n^{2})s^{2}\\+(2\,a_{1}y_{1}^{2}m^{2}n+2\,a_{2}y_{3}^{2}m^{2}n+2\,a_{1}y_{2}^{2}m\,n^{2}+2\,a_{2}y_{4}^{2}m\,n^{2})s\\+a_{1}y_{1}^{2}m^{2}n^{2}-a_{1}y_{2}^{2}m^{2}n^{2}+a_{2}y_{3}^{2}m^{2}n^{2}-a_{2}y_{4}^{2}m^{2}n^{2}$`
`\underline{\tt o33}` : `$R$`
\end{lstlisting}

\newpage
\subsection{Theorem \texorpdfstring{\ref{thm: unique nu}}{4.3}}

{\footnotesize
\begin{lstlisting}[language=Macaulay2output]
`\underline{\tt i60}` : clearAll
`\underline{\tt i61}` : R = QQ[a_1..a_2,y_1..y_4,n][s] 
`\underline{\tt o61}` = `$R$`
`\underline{\tt o61}` : `$\texttt{PolynomialRing}$`
`\underline{\tt i62}` : q_1 = a_1
`\underline{\tt o62}` = `$a_{1}$`
`\underline{\tt o62}` : `${\mathbb Q}\mathopen{}[a_{1}\,{.}{.}\,a_{2},\,y_{1}\,{.}{.}\,y_{4},\,n]$`
`\underline{\tt i63}` : q_2 = -a_1
`\underline{\tt o63}` = `$-a_{1}$`
`\underline{\tt o63}` : `${\mathbb Q}\mathopen{}[a_{1}\,{.}{.}\,a_{2},\,y_{1}\,{.}{.}\,y_{4},\,n]$`
`\underline{\tt i64}` : q_3 = a_2
`\underline{\tt o64}` = `$a_{2}$`
`\underline{\tt o64}` : `${\mathbb Q}\mathopen{}[a_{1}\,{.}{.}\,a_{2},\,y_{1}\,{.}{.}\,y_{4},\,n]$`
`\underline{\tt i65}` : q_4 = -a_2
`\underline{\tt o65}` = `$-a_{2}$`
`\underline{\tt o65}` : `${\mathbb Q}\mathopen{}[a_{1}\,{.}{.}\,a_{2},\,y_{1}\,{.}{.}\,y_{4},\,n]$`
`\underline{\tt i66}` : l_1 = a_1 -- m has been put to 1
`\underline{\tt o66}` = `$a_{1}$`
`\underline{\tt o66}` : `${\mathbb Q}\mathopen{}[a_{1}\,{.}{.}\,a_{2},\,y_{1}\,{.}{.}\,y_{4},\,n]$`
`\underline{\tt i67}` : l_2 = n*a_1
`\underline{\tt o67}` = `$a_{1}n$`
`\underline{\tt o67}` : `${\mathbb Q}\mathopen{}[a_{1}\,{.}{.}\,a_{2},\,y_{1}\,{.}{.}\,y_{4},\,n]$`
`\underline{\tt i68}` : l_3 = a_2
`\underline{\tt o68}` = `$a_{2}$`
`\underline{\tt o68}` : `${\mathbb Q}\mathopen{}[a_{1}\,{.}{.}\,a_{2},\,y_{1}\,{.}{.}\,y_{4},\,n]$`
`\underline{\tt i69}` : l_4 = n*a_2
`\underline{\tt o69}` = `$a_{2}n$`
`\underline{\tt o69}` : `${\mathbb Q}\mathopen{}[a_{1}\,{.}{.}\,a_{2},\,y_{1}\,{.}{.}\,y_{4},\,n]$`
`\underline{\tt i70}` : scan(4, i -> e_(i+1) = s*q_(i+1)*y_(i+1)/( l_(i+1)-s*q_(i+1) )) 
`\underline{\tt i71}` : Rfunction = sum apply(4, i -> q_(i+1)*(y_(i+1) + e_(i+1))^2)
`\underline{\tt o71}` = `$\frac{(-a_{1}y_{2}^{2}n^{2}-a_{2}y_{4}^{2}n^{2}+a_{1}y_{1}^{2}+a_{2}y_{3}^{2})s^{2}+(2\,a_{1}y_{2}^{2}n^{2}+2\,a_{2}y_{4}^{2}n^{2}+2\,a_{1}y_{1}^{2}n+2\,a_{2}y_{3}^{2}n)s+a_{1}y_{1}^{2}n^{2}-a_{1}y_{2}^{2}n^{2}+a_{2}y_{3}^{2}n^{2}-a_{2}y_{4}^{2}n^{2}}{s^{4}+(2\,n-2)s^{3}+(n^{2}-4\,n+1)s^{2}+(-2\,n^{2}+2\,n)s+n^{2}}$`
`\underline{\tt o71}` : `$\texttt{frac}\ R$`
`\underline{\tt i72}` : r2 = numerator Rfunction 
`\underline{\tt o72}` = `$(-a_{1}y_{2}^{2}n^{2}-a_{2}y_{4}^{2}n^{2}+a_{1}y_{1}^{2}+a_{2}y_{3}^{2})s^{2}+(2\,a_{1}y_{2}^{2}n^{2}+2\,a_{2}y_{4}^{2}n^{2}+2\,a_{1}y_{1}^{2}n+2\,a_{2}y_{3}^{2}n)s+a_{1}y_{1}^{2}n^{2}-a_{1}y_{2}^{2}n^{2}+a_{2}y_{3}^{2}n^{2}-a_{2}y_{4}^{2}n^{2}$`
`\underline{\tt o72}` : `$R$`
`\underline{\tt i73}` : coeffs = coefficients r2
`\underline{\tt o73}` = `$((\!\begin{array}{ccc}
s^{2}&s&1
\end{array}\!),\,\begin{array}{l}\{2,\:0\}\vphantom{-a_{1}y_{2}^{2}n^{2}-a_{2}y_{4}^{2}n^{2}+a_{1}y_{1}^{2}+a_{2}y_{3}^{2}}\\\{1,\:0\}\vphantom{2\,a_{1}y_{2}^{2}n^{2}+2\,a_{2}y_{4}^{2}n^{2}+2\,a_{1}y_{1}^{2}n+2\,a_{2}y_{3}^{2}n}\\\{0,\:0\}\vphantom{a_{1}y_{1}^{2}n^{2}-a_{1}y_{2}^{2}n^{2}+a_{2}y_{3}^{2}n^{2}-a_{2}y_{4}^{2}n^{2}}\end{array}(\!\begin{array}{c}
\vphantom{\{2,\:0\}}-a_{1}y_{2}^{2}n^{2}-a_{2}y_{4}^{2}n^{2}+a_{1}y_{1}^{2}+a_{2}y_{3}^{2}\\
\vphantom{\{1,\:0\}}2\,a_{1}y_{2}^{2}n^{2}+2\,a_{2}y_{4}^{2}n^{2}+2\,a_{1}y_{1}^{2}n+2\,a_{2}y_{3}^{2}n\\
\vphantom{\{0,\:0\}}a_{1}y_{1}^{2}n^{2}-a_{1}y_{2}^{2}n^{2}+a_{2}y_{3}^{2}n^{2}-a_{2}y_{4}^{2}n^{2}
\end{array}\!))$`
`\underline{\tt o73}` : `$\texttt{Sequence}$`
`\underline{\tt i74}` : 
      S = QQ[a_1..a_2,y_1..y_4,n][sqdelta] -- We introduce the squareroot of delta as a variable. 
`\underline{\tt o74}` = `$S$`
`\underline{\tt o74}` : `$\texttt{PolynomialRing}$`
`\underline{\tt i75}` : 
      coeffs = sub( coeffs#1, S )
`\underline{\tt o75}` = `$\begin{array}{l}\{2,\:0\}\vphantom{-a_{1}y_{2}^{2}n^{2}-a_{2}y_{4}^{2}n^{2}+a_{1}y_{1}^{2}+a_{2}y_{3}^{2}}\\\{1,\:0\}\vphantom{2\,a_{1}y_{2}^{2}n^{2}+2\,a_{2}y_{4}^{2}n^{2}+2\,a_{1}y_{1}^{2}n+2\,a_{2}y_{3}^{2}n}\\\{0,\:0\}\vphantom{a_{1}y_{1}^{2}n^{2}-a_{1}y_{2}^{2}n^{2}+a_{2}y_{3}^{2}n^{2}-a_{2}y_{4}^{2}n^{2}}\end{array}(\!\begin{array}{c}
\vphantom{\{2,\:0\}}-a_{1}y_{2}^{2}n^{2}-a_{2}y_{4}^{2}n^{2}+a_{1}y_{1}^{2}+a_{2}y_{3}^{2}\\
\vphantom{\{1,\:0\}}2\,a_{1}y_{2}^{2}n^{2}+2\,a_{2}y_{4}^{2}n^{2}+2\,a_{1}y_{1}^{2}n+2\,a_{2}y_{3}^{2}n\\
\vphantom{\{0,\:0\}}a_{1}y_{1}^{2}n^{2}-a_{1}y_{2}^{2}n^{2}+a_{2}y_{3}^{2}n^{2}-a_{2}y_{4}^{2}n^{2}
\end{array}\!)$`
`\underline{\tt o75}` : `$\texttt{Matrix}$ $S^{3}\,\longleftarrow \,S^{1}$`
`\underline{\tt i76}` : A = coeffs_(0,0)
`\underline{\tt o76}` = `$-a_{1}y_{2}^{2}n^{2}-a_{2}y_{4}^{2}n^{2}+a_{1}y_{1}^{2}+a_{2}y_{3}^{2}$`
`\underline{\tt o76}` : `$S$`
`\underline{\tt i77}` : B = coeffs_(1,0)
`\underline{\tt o77}` = `$2\,a_{1}y_{2}^{2}n^{2}+2\,a_{2}y_{4}^{2}n^{2}+2\,a_{1}y_{1}^{2}n+2\,a_{2}y_{3}^{2}n$`
`\underline{\tt o77}` : `$S$`
`\underline{\tt i78}` : C = coeffs_(2,0) 
`\underline{\tt o78}` = `$a_{1}y_{1}^{2}n^{2}-a_{1}y_{2}^{2}n^{2}+a_{2}y_{3}^{2}n^{2}-a_{2}y_{4}^{2}n^{2}$`
`\underline{\tt o78}` : `$S$`
`\underline{\tt i79}` : Delta = B^2 - 4*A*C 
`\underline{\tt o79}` = `$4\,a_{1}^{2}y_{1}^{2}y_{2}^{2}n^{4}+4\,a_{1}a_{2}y_{2}^{2}y_{3}^{2}n^{4}+4\,a_{1}a_{2}y_{1}^{2}y_{4}^{2}n^{4}+4\,a_{2}^{2}y_{3}^{2}y_{4}^{2}n^{4}+8\,a_{1}^{2}y_{1}^{2}y_{2}^{2}n^{3}+8\,a_{1}a_{2}y_{2}^{2}y_{3}^{2}n^{3}+8\,a_{1}a_{2}y_{1}^{2}y_{4}^{2}n^{3}+8\,a_{2}^{2}y_{3}^{2}y_{4}^{2}n^{3}+4\,a_{1}^{2}y_{1}^{2}y_{2}^{2}n^{2}+4\,a_{1}a_{2}y_{2}^{2}y_{3}^{2}n^{2}+4\,a_{1}a_{2}y_{1}^{2}y_{4}^{2}n^{2}+4\,a_{2}^{2}y_{3}^{2}y_{4}^{2}n^{2}$`
`\underline{\tt o79}` : `$S$`
`\underline{\tt i80}` : delta = numerator ( ( Delta ) / ( 4*n^2*(n+1)^2 ) ) -- We factor out 4*n^2*(n+1)^2 from Delta to get delta.
`\underline{\tt o80}` = `$a_{1}^{2}y_{1}^{2}y_{2}^{2}+a_{1}a_{2}y_{2}^{2}y_{3}^{2}+a_{1}a_{2}y_{1}^{2}y_{4}^{2}+a_{2}^{2}y_{3}^{2}y_{4}^{2}$`
`\underline{\tt o80}` : `$S$`
`\underline{\tt i81}` : 
      splus = ( - B + 2*n*( n + 1 )*sqdelta ) / ( 2*A )
`\underline{\tt o81}` = `$\frac{(-n^{2}-n)\mathit{sqdelta}+a_{1}y_{2}^{2}n^{2}+a_{2}y_{4}^{2}n^{2}+a_{1}y_{1}^{2}n+a_{2}y_{3}^{2}n}{a_{1}y_{2}^{2}n^{2}+a_{2}y_{4}^{2}n^{2}-a_{1}y_{1}^{2}-a_{2}y_{3}^{2}}$`
`\underline{\tt o81}` : `$\texttt{frac}\ S$`
`\underline{\tt i82}` : sminus = ( - B - 2*n*( n + 1 )*sqdelta ) / ( 2*A )
`\underline{\tt o82}` = `$\frac{(n^{2}+n)\mathit{sqdelta}+a_{1}y_{2}^{2}n^{2}+a_{2}y_{4}^{2}n^{2}+a_{1}y_{1}^{2}n+a_{2}y_{3}^{2}n}{a_{1}y_{2}^{2}n^{2}+a_{2}y_{4}^{2}n^{2}-a_{1}y_{1}^{2}-a_{2}y_{3}^{2}}$`
`\underline{\tt o82}` : `$\texttt{frac}\ S$`
`\underline{\tt i83}` : scan(4, i -> l_(i+1) = sub( l_(i+1), S ))
`\underline{\tt i84}` : scan(4, i -> q_(i+1) = sub( q_(i+1), S ))
`\underline{\tt i85}` : scan(4, i -> eplus_(i+1) =  splus*q_(i+1)*y_(i+1) / ( l_(i+1) - splus*q_(i+1) ) )
`\underline{\tt i86}` : scan(4, i -> eminus_(i+1) =  sminus*q_(i+1)*y_(i+1) / ( l_(i+1) - sminus*q_(i+1) ) )
`\underline{\tt i87}` : 
      subdeltasquare = (polynom) -> ( -- This function replaces sqdelta^2 with delta
      deg = degree( sqdelta, polynom );
      coeffspoly = for i from 0 to deg list coefficient( sqdelta^i, polynom );
      subpoly = 0;
      
      for i from 0 to deg do (
      if mod( i, 2 ) == 0 then (
      subpoly = subpoly + coeffspoly_(i) * delta^( numerator( i / 2 ) );       
      ) else (
      subpoly = subpoly + coeffspoly_(i) * delta^( numerator( ( i - 1 ) / 2 ) ) * sqdelta; 
      );
      );
      subpoly
      )
`\underline{\tt o87}` = `$\texttt{subdeltasquare}$`
`\underline{\tt o87}` : `$\texttt{FunctionClosure}$`
`\underline{\tt i88}` : 
      conjugaterule = (rational) -> ( -- This function uses the conjugate rule to get rid of sqdelta in the denominator
      ratnum = numerator rational;
      ratden = denominator rational;
      
      ratnum = ratnum * (coefficient( sqdelta, ratden )*sqdelta - coefficient( sqdelta^0, ratden));
      ratden = ratden * (coefficient( sqdelta, ratden )*sqdelta - coefficient( sqdelta^0, ratden));
      subdeltasquare(ratnum) / subdeltasquare(ratden)
      )
`\underline{\tt o88}` = `$\texttt{conjugaterule}$`
`\underline{\tt o88}` : `$\texttt{FunctionClosure}$`
`\underline{\tt i89}` : 
      weightedsumplus = sum apply(4, i -> l_(i+1)*( eplus_(i+1) )^2 ) -- Weighted sum with the + critical point
`\underline{\tt o89}` = `[very long expression]`
`\underline{\tt o89}` : `$\texttt{frac}\ S$`
`\underline{\tt i90}` : weightedsumminus = sum apply(4, i -> l_(i+1)*( eminus_(i+1) )^2 ) -- Weighted sum with the - critical point
`\underline{\tt o90}` = `[very long expression]`
`\underline{\tt o90}` : `$\texttt{frac}\ S$`
`\underline{\tt i91}` : sumplus = sum apply(4, i -> ( eplus_(i+1) )^2 ) -- Non-weighted sum with the + critical point
`\underline{\tt o91}` = `[very long expression]`
`\underline{\tt o91}` : `$\texttt{frac}\ S$`
`\underline{\tt i92}` : summinus = sum apply(4, i -> ( eminus_(i+1) )^2 ) -- Non-weighted sum with the - critical point
`\underline{\tt o92}` = `[very long expression]`
`\underline{\tt o92}` : `$\texttt{frac}\ S$`
`\underline{\tt i93}` : 
      -- We replace squares of sqdelta with delta and use the conjugate rule to remove sqdelta from the denominator:
      weightedsumplus = conjugaterule ( subdeltasquare(numerator weightedsumplus) / subdeltasquare(denominator weightedsumplus) ) 
`\underline{\tt o93}` = `$\frac{-2\,n\,\mathit{sqdelta}+a_{1}y_{1}^{2}n+a_{1}y_{2}^{2}n+a_{2}y_{3}^{2}n+a_{2}y_{4}^{2}n}{n+1}$`
`\underline{\tt o93}` : `$\texttt{frac}\ S$`
`\underline{\tt i94}` : weightedsumminus = conjugaterule ( subdeltasquare(numerator weightedsumminus) / subdeltasquare(denominator weightedsumminus) )
`\underline{\tt o94}` = `$\frac{2\,n\,\mathit{sqdelta}+a_{1}y_{1}^{2}n+a_{1}y_{2}^{2}n+a_{2}y_{3}^{2}n+a_{2}y_{4}^{2}n}{n+1}$`
`\underline{\tt o94}` : `$\texttt{frac}\ S$`
`\underline{\tt i95}` : sumplus = conjugaterule ( subdeltasquare(numerator sumplus) / subdeltasquare(denominator sumplus) )
`\underline{\tt o95}` = `[very long expression]`
`\underline{\tt o95}` : `$\texttt{frac}\ S$`
`\underline{\tt i96}` : summinus = conjugaterule ( subdeltasquare(numerator summinus) / subdeltasquare(denominator summinus) )
`\underline{\tt o96}` = `[very long expression]`
`\underline{\tt o96}` : `$\texttt{frac}\ S$`
`\underline{\tt i97}` : 
      -- We get the identities from the statement by observing weighhtedsumplus and weightedsumminus directly.
      -- To examine the non-weighted sums, we first look at their denominators:
      factor denominator sumplus
`\underline{\tt o97}` = `$(n+1)^{2}(a_{1}y_{2}^{2}+a_{2}y_{4}^{2})(a_{1}y_{1}^{2}+a_{2}y_{3}^{2})$`
`\underline{\tt o97}` : `$\texttt{Expression}$ of class $\texttt{Product}$`
`\underline{\tt i98}` : factor denominator summinus
`\underline{\tt o98}` = `$(n+1)^{2}(a_{1}y_{2}^{2}+a_{2}y_{4}^{2})(a_{1}y_{1}^{2}+a_{2}y_{3}^{2})$`
`\underline{\tt o98}` : `$\texttt{Expression}$ of class $\texttt{Product}$`
`\underline{\tt i99}` : 
      -- And next their numerators:
      numsumplus = factor numerator sumplus
`\underline{\tt o99}` = `$(a_{1}y_{1}^{2}y_{2}^{2}n^{2}+a_{1}y_{2}^{2}y_{3}^{2}n^{2}+a_{2}y_{1}^{2}y_{4}^{2}n^{2}+a_{2}y_{3}^{2}y_{4}^{2}n^{2}+a_{1}y_{1}^{2}y_{2}^{2}+a_{2}y_{2}^{2}y_{3}^{2}+a_{1}y_{1}^{2}y_{4}^{2}+a_{2}y_{3}^{2}y_{4}^{2})\\ \cdot(2\,\mathit{sqdelta}-a_{1}y_{1}^{2}-a_{1}y_{2}^{2}-a_{2}y_{3}^{2}-a_{2}y_{4}^{2})(-1)$`
`\underline{\tt o99}` : `$\texttt{Expression}$ of class $\texttt{Product}$`
`\underline{\tt i100}` : numsumminus = factor numerator summinus
`\underline{\tt o100}` = `$(a_{1}y_{1}^{2}y_{2}^{2}n^{2}+a_{1}y_{2}^{2}y_{3}^{2}n^{2}+a_{2}y_{1}^{2}y_{4}^{2}n^{2}+a_{2}y_{3}^{2}y_{4}^{2}n^{2}+a_{1}y_{1}^{2}y_{2}^{2}+a_{2}y_{2}^{2}y_{3}^{2}+a_{1}y_{1}^{2}y_{4}^{2}+a_{2}y_{3}^{2}y_{4}^{2})\\ \cdot(2\,\mathit{sqdelta}+a_{1}y_{1}^{2}+a_{1}y_{2}^{2}+a_{2}y_{3}^{2}+a_{2}y_{4}^{2})$`
`\underline{\tt o100}` : `$\texttt{Expression}$ of class $\texttt{Product}$`
`\underline{\tt i101}` : 
       -- A factor of alpha^+, respectively alpha^-, comes out of the sum.
       -- Factoring these out, the numerators are the same: S*n^2 + T. 
       -- We next factor the coefficient S and T
       T = QQ[a_1..a_2,y_1..y_4][n] 
`\underline{\tt o101}` = `$T$`
`\underline{\tt o101}` : `$\texttt{PolynomialRing}$`
`\underline{\tt i102}` : 
       numsumplus = sub(numsumplus#0#0, T)
`\underline{\tt o102}` = `$(a_{1}y_{1}^{2}y_{2}^{2}+a_{1}y_{2}^{2}y_{3}^{2}+a_{2}y_{1}^{2}y_{4}^{2}+a_{2}y_{3}^{2}y_{4}^{2})n^{2}+a_{1}y_{1}^{2}y_{2}^{2}+a_{2}y_{2}^{2}y_{3}^{2}+a_{1}y_{1}^{2}y_{4}^{2}+a_{2}y_{3}^{2}y_{4}^{2}$`
`\underline{\tt o102}` : `$T$`
`\underline{\tt i103}` : S = coefficient(n^2, numsumplus)
`\underline{\tt o103}` = `$a_{1}y_{1}^{2}y_{2}^{2}+a_{1}y_{2}^{2}y_{3}^{2}+a_{2}y_{1}^{2}y_{4}^{2}+a_{2}y_{3}^{2}y_{4}^{2}$`
`\underline{\tt o103}` : `${\mathbb Q}\mathopen{}[a_{1}\,{.}{.}\,a_{2},\,y_{1}\,{.}{.}\,y_{4}]$`
`\underline{\tt i104}` : T = coefficient(n^0, numsumplus)
`\underline{\tt o104}` = `$a_{1}y_{1}^{2}y_{2}^{2}+a_{2}y_{2}^{2}y_{3}^{2}+a_{1}y_{1}^{2}y_{4}^{2}+a_{2}y_{3}^{2}y_{4}^{2}$`
`\underline{\tt o104}` : `${\mathbb Q}\mathopen{}[a_{1}\,{.}{.}\,a_{2},\,y_{1}\,{.}{.}\,y_{4}]$`
`\underline{\tt i105}` : 
       -- Here we see that S and T are the same and in the paper:
       factor S
`\underline{\tt o105}` = `$(y_{1}^{2}+y_{3}^{2})(a_{1}y_{2}^{2}+a_{2}y_{4}^{2})$`
`\underline{\tt o105}` : `$\texttt{Expression}$ of class $\texttt{Product}$`
`\underline{\tt i106}` : factor T
`\underline{\tt o106}` = `$(y_{2}^{2}+y_{4}^{2})(a_{1}y_{1}^{2}+a_{2}y_{3}^{2})$`
`\underline{\tt o106}` : `$\texttt{Expression}$ of class $\texttt{Product}$`
\end{lstlisting}
}